\documentclass[11pt]{article}

\usepackage[preprint]{jmlr2e}

\usepackage{amsmath,enumerate}
\usepackage{mdframed}
\usepackage[percent]{overpic} 
\usepackage[T1]{fontenc} 
\usepackage[singlelinecheck=false, aboveskip=5pt, belowskip=0pt]{caption}
\usepackage{xcolor}

\hypersetup{
    colorlinks,
    linkcolor={red!50!black},
    citecolor={blue!80!black},
    urlcolor={blue!80!black}
}

\usepackage[capitalise]{cleveref} 

\newcommand{\red}[1]{{\color{red} #1}}
\newcommand{\qandq}{\quad \text{and} \quad}

\newcommand{\norm}[1]{ \left\| #1 \right\| }
\newcommand{\mednorm}[1]{\| #1\| }

\newcommand{\var}{\mathbb{V}ar}
\newcommand{\Cov}{Cov}

\newcommand{\E}{\mathbb{E}}
\newcommand{\R}{\mathbb{R}}
\newcommand{\N}{\mathbb{N}}

\newtheorem{setting}[theorem]{Setting}

\usepackage{lastpage}
\jmlrheading{--}{--}{1-\pageref{LastPage}}{4/24; Revised --}{--}{-------}{Pierfrancesco Beneventano}

\ShortHeadings{On the Trajectories of SGD Without Replacement}{Pierfrancesco Beneventano}
\firstpageno{1}

\title{On the Trajectories of SGD Without Replacement}

\author{\name Pierfrancesco Beneventano \email pierb@princeton.edu \\
       \addr Department of Operations Research and Financial Engineering\\
       Princeton University,
       Princeton, NJ 08544, USA
       }

\begin{document}


\maketitle

\begin{abstract}
This article examines the implicit regularization effect of Stochastic Gradient Descent (SGD). We consider the case of SGD without replacement, the variant typically used to optimize large-scale neural networks. We analyze this algorithm in a more realistic regime than typically considered in theoretical works on SGD, as, e.g., we allow the product of the learning rate and Hessian to be $O(1)$ and we do not specify any model architecture, learning task, or loss (objective) function.

Our core theoretical result is that optimizing with SGD without replacement is locally equivalent to making an additional step on a novel regularizer.
This implies that the expected trajectories of SGD without replacement can be decoupled in (i) following SGD with replacement (in which batches are sampled i.i.d.) along the directions of high curvature, and (ii) regularizing the trace of the noise covariance along the flat ones.

As a consequence, SGD without replacement travels flat areas and may escape saddles significantly faster than SGD with replacement.
On several vision tasks, the novel regularizer penalizes a weighted trace of the Fisher Matrix, thus encouraging sparsity in the spectrum of the Hessian of the loss in line with empirical observations from prior work.
We also propose an explanation for why SGD does not train at the edge of stability (as opposed to GD).
\end{abstract}


\section{Introduction}

\subsection{The Problem and the Background}

This article examines the implicit regularization effect of Stochastic Gradient Descent (SGD) without replacement, also known as random reshuffling. This is the particular variant of SGD most commonly used in practice to optimize large-scale neural networks.

To put this into context, recall that gradient descent (GD) and its variants are a generic set of algorithms for minimizing an empirical loss $L$,
which depends on a set of parameters $\theta$ and a dataset $D$ of size $n$. Specifically, the update to the parameters after the first $t$ steps of optimization is as follows:
\[
\theta(t+1) \quad = \quad
\theta(t) \ \ - \ \ \frac{\eta }{b} 
\sum_{z \in B_{t+1}} \nabla_\theta L\big(\theta(t), z\big), \qquad  t=0,1,\ldots,
\]
where the $j$-th batch $B_j$ for all $j$ is a subset of the training dataset $D$ of size $b$ and $\eta$ is the step size (or learning rate). Full-batch gradient descent corresponds to the case when every $B_j = D$ is the entire dataset, whereas mini-batch SGD corresponds to smaller values of $b$. In practice, smaller values of $b$ are both computationally faster and often lead to better performance \cite{bottou_curiously_2009, bottou_stochastic_2012, lecun_efficient_2012, bengio_practical_2012}. The gain in speed is almost immediate since loss gradients at each step must only be computed on the current smaller batch (see \S \ref{section:training}). The nature of the improvement in model quality from smaller values of $b$, however, is less obvious.

\begin{figure}[ht!]
\begin{mdframed}
{\centering \large \textbf{Variants of SGDs on Toy Shallow Networks}\\[0.3cm]}
    \begin{minipage}{.327\textwidth} 
        \centering
        \begin{overpic}[width=\linewidth,height=1.3\linewidth]{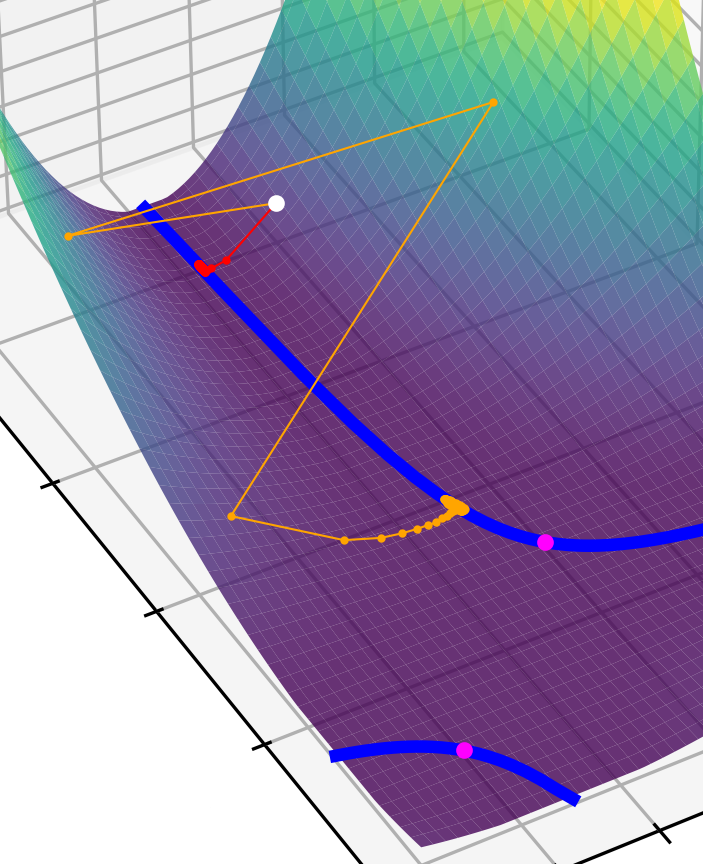}
            \put(1,90){\colorbox{white}{GD with Gaussian noise}}
        \end{overpic}
    \end{minipage}
    \hfill
    \begin{minipage}{.327\textwidth} 
        \centering
        \begin{overpic}[width=\linewidth,height=1.3\linewidth]{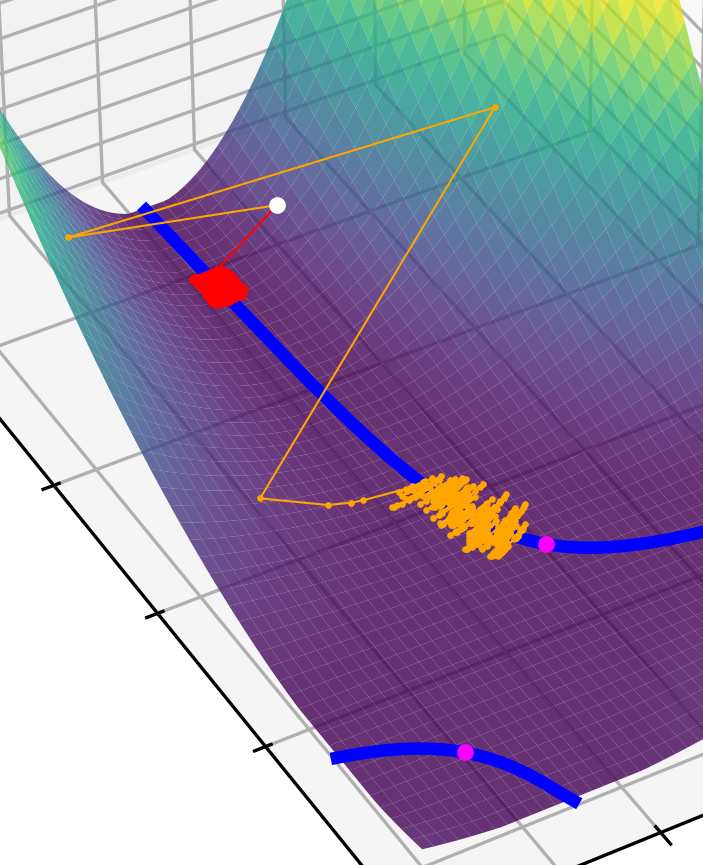}
            \put(15,90){\colorbox{white}{SGD with repl.}}
        \end{overpic}
    \end{minipage}
    \hfill
    \begin{minipage}{.327\textwidth} 
        \centering
        \begin{overpic}[width=\linewidth,height=1.3\linewidth]{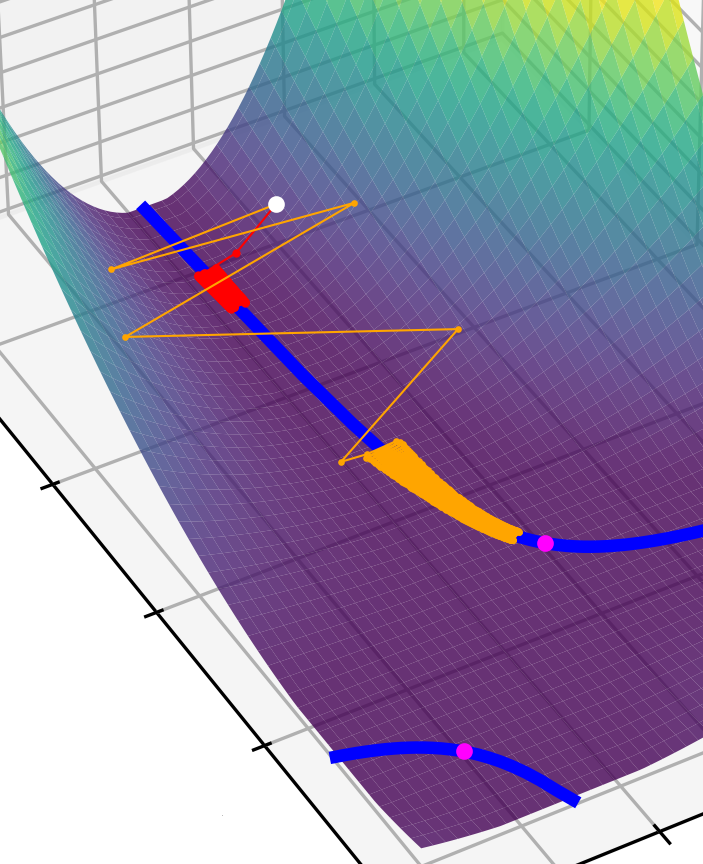}
            \put(8,90){\colorbox{white}{SGD without repl.}}
        \end{overpic}
    \end{minipage}
    \caption{
    \small
        Here are the dynamics in the setting outlined in \S \ref{section:setting_plots}, for \red{small} (red) and {\color{orange}big} (orange) learning rates. 
        The colored surface is the loss landscape, the {\color{blue}blue line} is the manifold of minima, the white bullet is the initialization, and the {\color{magenta}magenta} bullets are the lowest norm solutions.
        SGD without and with replacement both move towards the lowest norm solution, unlike noised GD.
        However,
        SGD \textit{without} replacement is converging \textit{faster}, with \textit{better accuracy}, while steadily \textit{reducing the variance}. 
        SGD with replacement is approaching slower and will oscillate around it with a higher variance and less precision.
    }
    \label{fig:intro}
\end{mdframed}
\end{figure}

The differences between full-batch GD and mini-batch SGD have been extensively studied both empirically, e.g., \cite{keskar_large-batch_2016,he_control_2019,jastrzebski_catastrophic_2021} and theoretically, e.g., \cite{kleinberg_alternative_2018,haochen_shape_2020,chen_stochastic_2023}. However, virtually all the well-known theoretical work, except \cite{gurbuzbalaban_why_2021,nguyen_unified_2021,haochen_random_2019,mishchenko_random_2020,roberts_sgd_2021,smith_origin_2021}, treats only the case of mini-batch SGD with replacement, in which the batches $B_i$ are sampled i.i.d. The goal of this article, in contrast, is to give a theoretical analysis of a significantly more realistic setting for mini-batch SGD in which we consider:
\begin{itemize}
    \item \textbf{The algorithm practitioners use:} 
    We analyze the algorithm most commonly used in practice to optimize large-scale neural networks, called SGD without replacement or sometimes random reshuffling, see \S \ref{section:training}. The key difference between SGD with and without replacement is that batches sampled without replacement are disjoint and form a random partition of the training dataset. These batches are, therefore, statistically dependent. This induces empirical differences, see, e.g., \cref{fig:intro}, and mathematical challenges, see \S \ref{section:challenges}. We also prove that it leads to qualitatively different behaviors than SGD with replacement, see \cref{fig:W_comparison}, \S \ref{section:implications}, and \S \ref{section:with_repl}.
    
    \item \textbf{In real-world optimization regimes:} Previous analyses of SGD made strong assumptions on the step size $\eta$, the number $k$ of steps of SGD that are analyzed (usually one epoch at time), or the size of the derivatives of the loss $L$. For instance \cite{smith_origin_2021} asks that 
    \begin{equation}
    \label{eq:setting_smith}
        \eta \ \cdot \ k \ \cdot \ \| \nabla L \| 
        \quad \ll \quad 1
        \qquad \qandq \qquad  
        \eta \ \cdot \ k \ \cdot \ \| \nabla^2 L \| 
        \quad \ll \quad 1.
    \end{equation}
    The first hypothesis roughly requires that the total movement in parameter space is small and corresponds to allowing a local analysis for SGD. We will keep this assumption in the present work. However, we drop the second, see \S \ref{section:error}, that is often unrealistic, as we explain in \S \ref{section:regime}.
    For instance, a well-known empirical observation about neural networks -- known as progressive sharpening
    \cite{jastrzebski_relation_2019, cohen_gradient_2021,cohen_adaptive_2022,damian_self-stabilization_2023} -- is that $\norm{\nabla^2 L}$ often grows throughout training until it is of order $\eta^{-1}$.
    Moreover, we do not make any assumptions about the learning task, model, or data set. We work on a general optimization task where the objective is a sum of functions admitting 3 derivatives.
\end{itemize}

\subsection{Informal Overview of the Results}
We provide here an overview of our results with an emphasis on new insights into the differences between SGD with and without replacement. We also offer several theoretical insights that can help explain a range of phenomena empirically observed in training large neural networks. 

We start with an informal corollary of our main result in \cref{theo:SGD_bias_eta}. 
Note that, in expectation, the trajectories of GD and SGD with replacement are the same. However, SGD without replacement is different, even on average.
We show that, relative to full-batch GD and SGD with replacement, SGD without replacement implicitly adds a regularizer. This penalizes the covariance over the dataset of the gradients, measured in a loss-dependent norm determined by the Hessian of the loss.

\begin{theorem}[Informal corollary of \cref{theo:SGD_bias_eta}]
\label{theo:intro}
In expectation over batch sampling, one epoch of SGD without replacement differs from the same number of steps of SGD with replacement or GD, by a regularizing step of size $\eta/(b-1)$. At a stationary point, the regularizer of the $i-th$ parameter is a weighted trace of the covariance of the gradients
\begin{equation*}
\begin{split}
    \mathop{trace} \left( S_{i} \ \cdot \ \Cov_{z \in D}\big( \nabla L(\theta, z) \big) \right),
\end{split} 
\end{equation*} 
where $S_i$ is a diagonal matrix which depends on the Hessian. Away from a stationary point, there is an additional term dependent on the trace of the empirical covariance of the Hessian.
\end{theorem}

Our goal in the rest of this section is to describe the nature of the regularizer from Theorem \ref{theo:intro} and, in particular, its role
in driving qualitatively different behaviors between SGD with and without replacement. As a starting point, let us consider a stationary point $\theta$ for the full-batch loss $L(\theta,D)$. The matrix $S$ depends on $\theta$ and is a function of the Hessian of the loss $\nabla^2 L$ and its variance.
Let us work on a basis of eigenvectors $\{\theta_1, \theta_2, \ldots\}$ for the Hessian corresponding to the eigenvalues $\{\lambda_1, \lambda_2, \ldots\}$ and define the effective step size $c = \eta k > 0$.
The $j$-th diagonal entry of $S_i$ approximately is
    \begin{equation}
    \label{eq:real_world_effect}
    (S_i)_{j,j}
    \quad \sim \quad \frac{c}{2}\cdot
    \begin{cases}
        \frac{1}{2}    
        \!\! \qquad \qquad \qquad \qquad \qquad \quad \text{ constant}   \qquad
        & \text{ if } c\lambda_i, c\lambda_j \text{ small}\\
        (c\lambda_i)^{-1}    
        \hspace{0.04cm}\quad \qquad \qquad \quad \qquad \text{ small}
        & \text{ if } c\lambda_i \gg 0, c \lambda_j \text{ small} \\
        (c\lambda_j)^{-1}       
        \quad \qquad \qquad \qquad \quad \text{ small}
        & \text{ if } c \lambda_i \text{ small}, c\lambda_j \gg 0 \\
        \frac{1 + (-\eta \lambda_i)^{k}}{c\lambda_i \cdot c\lambda_j} \quad \qquad \qquad \qquad \ \,
        \substack{\text{very small }\\\text{or very big}}
        & \text{ if } c\lambda_i, c\lambda_j \gg 0 
    \end{cases}
    \end{equation}
Note that when $c\lambda \ll 1$, for all $\lambda$s, we are in the setting \eqref{eq:setting_smith} of work such as  \cite{smith_origin_2021} and indeed recover their core results, with the regularizer corresponding to $\tfrac{c}{4}\E_{z \in D}[\mednorm{\nabla L(z) - \nabla L}^2]$, see \S \ref{section:bias_small} and \ref{section:dependency}. However, in our analysis, not only $c\lambda$ need not be small in absolute value, but we allow the eigenvalues $\lambda \in [-O(1/c), 2/\eta + O(1/c)]$. This setting includes a richer set of phenomena than found in prior work and contains most real-world training regimes \S \ref{section:regime}. We discuss several phenomena that arise in the following section.

\subsection{Decoupling the Dynamics}\label{section:implications}

A first key implication of \cref{theo:intro} is that in expectation
\begin{mdframed}
    The trajectory of the SGD without replacement can be decoupled as follows.
\begin{itemize}
    \item It follows SGD with replacement along the directions subject to bigger curvature.
    \item It shrinks the trace of the covariance of the gradients along the flatter directions.
\end{itemize}
\end{mdframed}

The role of $S_i$, indeed, is to approximately project or restrict the part of the dynamics relative to the regularizer on the flat directions, i.e., the kernel of the Hessian. In the case in which either big or small eigenvalues are present in the Hessian, as \cref{fig:intro}, $S_i$ is approximately zero along the directions admitting big curvature and $c/4$ along the eigendirections relative to the small ones. For a vector $v$
\[
S_i \cdot v \quad \sim \quad \frac{c}{4} \cdot \underbrace{\Pi_{flat}}_{\text{projection}} (v) \qquad \substack{\text{projects the vector } v \text{ along the directions}\\
\text{in which the Hessian is degenerate.}}
\]
This means that if two minima or stationary points are in the same connected manifold of stationary points, SGD without replacement will move towards the one admitting a lower trace of the covariance, so the flatter one. 

\cref{theo:intro} may explain the suitability of SGD without replacement for non-convex objective/loss functions with the same geometry as the ones of large scale neural networks.
Recent findings \cite{liu_toward_2020,nguyen_connected_2019,zhang_embedding_2022}, indeed, indicate that the loss-landscapes of neural networks present highly degenerate Hessians figuring high dimensional manifolds of stationary points.
Most of these stationary points generalize poorly,
SGD, however, converges to almost-global loss minima that generally generalize better \cite{gunasekar_implicit_2017,neyshabur_search_2015}.

\subsection{Shaping the Hessian - the Implicit Regularization}
The community observed that SGD shapes the spectrum of the Hessian of the loss, e.g., producing clusters of large outlying eigenvalues and sending small eigenvalues to zero in the course of training \cite{keskar_large-batch_2016,sagun_eigenvalues_2016, papyan_traces_2020,jastrzebski_catastrophic_2021}.
We find that \cref{theo:intro} is potentially enough to explain and make precise these observed phenomena.
Indeed, we find that SGD without replacement tends to decrease the trace of the covariance of the gradients over the datapoints mainly along the eigenspaces of the small and negative eigenvalues of the Hessian. These directions are considered to control (i) how the model overfits and (ii) the effective complexity and sparsity of the model.
Moreover, this can be achieved in two ways, see \cref{fig:W_without}:
\vspace{-0.15cm}
\begin{enumerate}
    \item The algorithm finds better-fitting stationary points by reducing the squared residual on every data point. Such points are often global minima, see \cref{fig:intro_2}, or
    \vspace{-0.1cm}
    \item The algorithm searches for points with smaller $\mednorm{\nabla_\theta \text{model}(\theta,x)}^2$, usually indicating flatter stationary points.
    \vspace{-0.15cm}
\end{enumerate}
We also show that in various classification tasks, the regularizer corresponds to a weighted trace of the Fisher matrix (see \S \ref{section:Hessian}).
This potentially explains and formalizes why \cite{jastrzebski_catastrophic_2021} empirically observed that a higher effective learning rate $\eta/b$ (corresponding to bigger steps on the regularizer in Theorem \ref{theo:intro}) better regularizes the trace of the Fisher matrix in vision tasks. 
This generally leads to consistent generalization improvements and makes the training less prone to overfitting and memorization \cite{jastrzebski_catastrophic_2021}.

Furthermore, in \cref{theo:reg_Hessian} we prove that, once most training points are correctly classified, SGD without replacement effectively minimizes a weighted trace of the Hessian.
\[
\mathrm{trace}\big(S \cdot \nabla^2 L(\theta, D)\big)
\]
This is the first mathematical result, to the knowledge of the author, that explains and formalizes the observation of \cite{keskar_large-batch_2016}.
In particular, this potentially explains frequent observations that SGD tends to discover minima with a sparsified Hessian, characterized by a few outlying large eigenvalues (although smaller than $\eta^{-1}$) and many smaller eigenvalues near zero, see, e.g., \cite{keskar_large-batch_2016, sagun_eigenvalues_2016, papyan_full_2019}.

\subsection{Escaping Saddles} \label{section:intro_saddles}

Theoretical works indicate that neural networks' loss landscapes have plenty of strict and high-order saddle points, see, e.g., \cite{kawaguchi_deep_2016}. Yet, surprisingly, SGD does not typically get trapped in these points in real-world applications \cite[Chapter 8]{goodfellow_deep_2016}.
Suppose $\theta$ is a strict saddle point for the loss $L(\theta, D)$ with negative eigenvalue $\lambda < 0$. We know that noise-injected GD and SGD with replacement escape a neighborhood of this saddle in respectively $O(\lambda^{-2})$ and $O(\lambda^{-3.5})$ steps under the so-called dispersive noise assumption, \cite{fang_sharp_2019,jin_nonconvex_2021}.

A key consequence of Theorem \ref{theo:intro} is that SGD without replacement escapes these saddle points \textit{faster} then these algorithms, taking only $O(\lambda^{-1})$ epochs. 
The reason is that, as we see in Theorem \ref{theo:intro}, correlations between batches in SGD without replacement leads to a non-zero effective drift in the dynamics. In contrast, algorithms such as noise-injected GD and SGD with replacement are equal, on average, to vanilla full-batch GD and escape saddles due to diffusive effects coming from the noise.

More precisely, see Proposition \ref{prop:strict_saddles}, suppose the Hessian eigenvector with eigenvalue $\lambda \leq 0$ has non-zero scalar product $u$ with $\E_{z \in D}[\nabla^2 L(z) \nabla L(z)]$ (i.e., there exists at least one data point whose regularizer's gradient has non-zero overlap with a negative eigenvector of the Hessian) and the third derivative of the loss is locally-bounded along the trajectory.
Then, the step the algorithm makes on the regularizer at every epoch makes the expected trajectory of SGD escape the saddle approximately in:
\vspace{-0.1cm}
\[
            \# \text{epochs}
            \quad \sim \quad
            \frac{\ln(\eta) + \ln(|u|)}{c\lambda}
\]
In summary, we believe we made an important step in understanding why and how SGD without replacement escapes saddles so fast in practice. The reason is that it may be simply not affected by saddles, the saddles of the loss often are not saddles for the loss plus the regularizer. This is a fundamental difference from SGD with replacement, which is unbiased and thus escapes saddles only thanks to the diffusion of the noise. 
However, other points that were not saddles for GD may be saddles for SGD without replacement.

\subsection{Escaping the Edge of Stability}

Empirical work shows that often the Hessian steadily increases along the trajectories of gradient flow and gradient descent \cite{jastrzebski_relation_2019,cohen_gradient_2021}.
In the case of GD, this progressive sharpening stops only when the highest eigenvalue $\lambda_{\max}$ of the Hessian reaches $2/\eta$ - the so-called edge of stability. Then, GD starts oscillating along the eigenvector of $\lambda_{\max}$ while still converging along other directions \cite{damian_self-stabilization_2023}.
However, \cite{cohen_gradient_2021} also observes that this is not the case of SGD without replacement, it follows the trajectory of GD leading to the Edge of Stability only up to a point where $\lambda_{\max}$ stops increasing. 
Similarly, \cite{jastrzebski_break-even_2020} observed that, both in regression and vision classification tasks, the trajectory of SGD without replacement aligns with the one of GD for a while until a \textit{breaking point} where it deviates and diverges from it.

We make an important step towards explaining and formalizing these observations in \S \ref{section:eos}.
We prove that if there are at least two big eigenvalues and the covariance of the noise is non-zero along their eigendirections, then there exists a value, denoted by $\alpha_{EoS}$, such that when $\lambda_{\max} > \eta^{-1} + \alpha_{EoS}$, the size of a step on the regularizer of \cref{theo:intro} becomes bigger than $k$ steps of GD. 
This may effectively lead the expected trajectory away from GD's one in most regimes. 
Agreeing with the empirical studies \cite{jastrzebski_break-even_2020}, this phenomenon consists of a quick phase transition. 
Moreover, this breaking point arrives earlier for bigger effective learning rates $\eta / b$ and later for smaller ones, indeed,
\[
\alpha_{EoS}
\quad = \quad 
\frac{b}{\eta} \ \cdot \ (\text{positive quantity}).
\]

\subsection{With \textit{vs} Without Replacement.} 
A subtle long-standing question is whether the fact that batches are disjoint and not sampled i.i.d.\ practically matters. 
In other words, if the induced dependency is strong enough for the trajectories SGD without replacement to get attracted to different minima than SGD with replacement or noise-injected GD. This is crucial to understand because, although SGD without replacement is faster and widely used in practice, most theoretical developments focus on GD or algorithms with independent steps.
We believe this article makes a significant step towards understanding whether these algorithms' trajectories qualitatively differ.

Our findings reveal qualitative differences in the trajectories of SGD with and without replacement through the loss landscape \S \ref{section:with_repl}. While they may not necessarily converge to different minima, \textit{how} they traverse the flat areas of the landscape is distinct, \cref{fig:intro,fig:intro_2,fig:W_comparison}.

\begin{itemize}
\vspace{-0.15cm}
    \item \textbf{Drift powered \textit{vs} Diffusion powered.} Just as in the neighborhood of flat areas (see \S \ref{section:intro_saddles} and \S \ref{section:implications}), SGD without replacement \textit{regularizes faster} than SGD with replacement \cref{fig:intro_2}. This occurs because the implicit regularization is non-zero already on the level of the average of the trajectory, i.e., SGD with replacement has the same drift than GD, SGD without replacement does not.

    \item \textbf{Smaller Oscillations:} In flatter areas, the noise/diffusion of SGD without replacement comes from the second order term in the Taylor expansion of the trajectory. The noise of SGD with replacement always comes from the first order terms. This implies that the trajectory of SGD without replacement has usually smaller variance \cref{fig:intro,fig:intro_2}. 
\end{itemize}

\begin{figure}[ht!]
\begin{mdframed}
{\centering \large \textbf{Different Behaviors of the  
SGDs: Escaping Local Minima}\\[0.3cm]}
    \begin{minipage}{.56\textwidth} 
        \centering
        \includegraphics[width=\linewidth,height=0.7\linewidth]{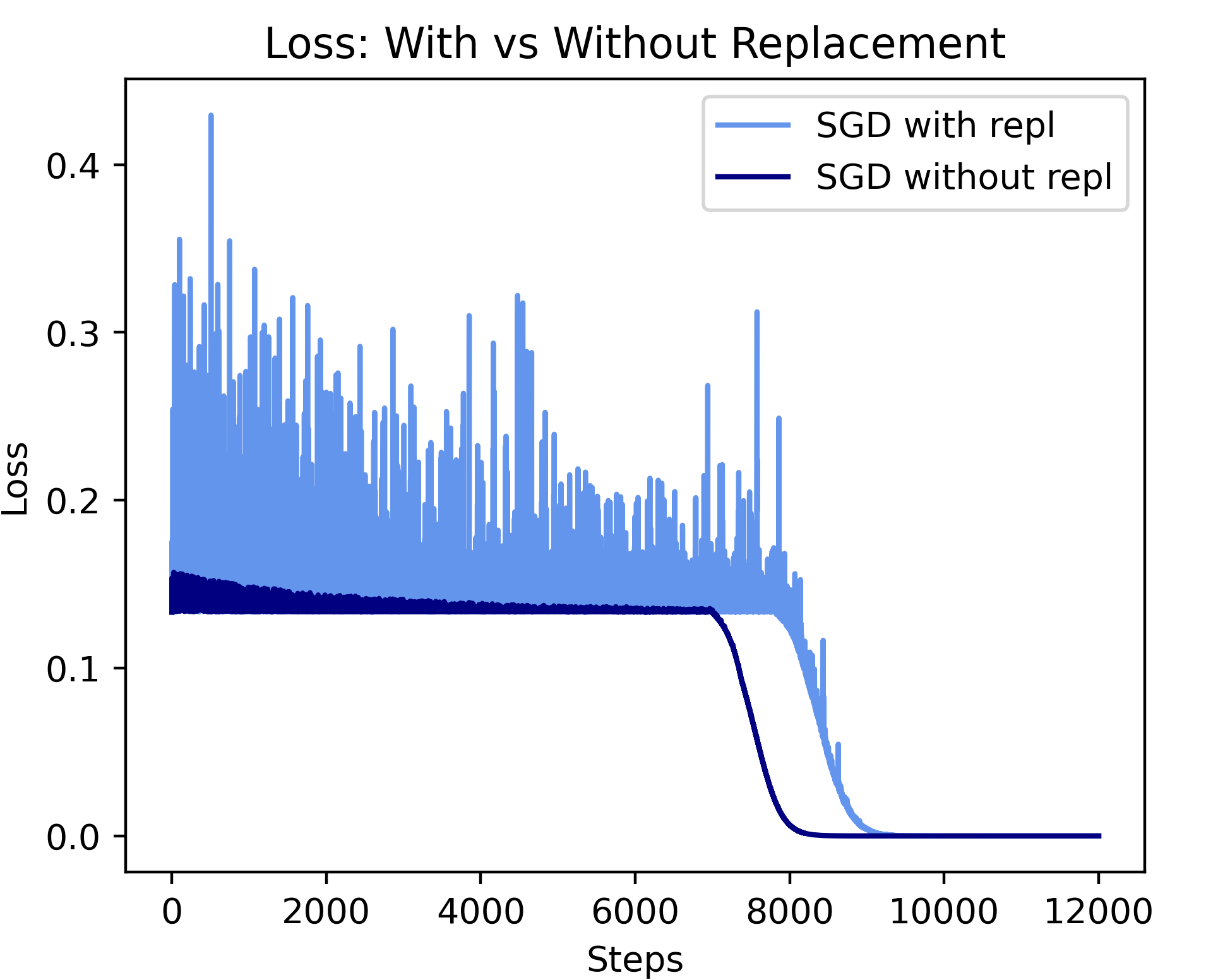}
    \end{minipage}
    \hspace{-0.015\textwidth}
    \hfill
    \begin{minipage}{.44\textwidth} 
    \caption{This is the behavior of SGDs with replacement (lighter blue) \textit{vs} SGD without replacement (darker blue) starting from a spurious local minimum of the loss of a ReLU network with the same hyperparameters.\\
    SGD without replacement escapes local minima (by traveling flat areas) in which GD converged faster than SGD with replacement and with much smaller oscillations. They both converge to a global minimum. See \cref{fig:W_comparison}.
    }
    \label{fig:intro_2}
    \end{minipage}
    \hspace{0.025\textwidth}
\end{mdframed}
\end{figure}

In conclusion, our results show that SGD without replacement implicitly regularizes by biasing the dynamics towards areas where it has lower variance. This effect results from the dependency between steps, manifesting as a drift-like phenomenon, distinct from effects attributable to the diffusion or, e.g., Fokker-Planck arguments. We demonstrate that this leads to a form of regularization, enabling the algorithm to navigate through flat areas more quickly and with fewer oscillations than expected.

\subsection{Outline of the Remainder of the Article}

We start with an overview of the problem in \S \ref{section:motivations}. Specifically, \S \ref{section:training} explores how neural networks are trained, the reasons for such approaches, and introduces SGD without replacement. 
In \S \ref{section:generalization}, we discuss our goals and where they stem from. \S \ref{section:challenges} Sheds light on the unusual mathematical challenges that we face while tackling the problem while reviewing the literature. We conclude this part with \S \ref{section:regime} in which we discuss in what regime neural networks are trained, a relevant theory must explain what happens in this scenario.
Following this introductory section, we present our main result in \S \ref{section:without_repl}. We will later discuss deeply its ingredients in \S \ref{section:nature}. 
The subsequent sections delve into its implications. In particular, we highlight the way the SGD without replacement travels the landscapes and the manifold of minima in \S \ref{section:regularization}, it regularizes the Hessian in \S \ref{section:Hessian}, it escapes saddles in \S \ref{section:saddles},
and the way it avoids the edge of stability in \S \ref{section:eos}. 
We discuss the differences between the two variants of SGD in \S \ref{section:with_repl}.
We conclude the paper with the conclusions and a discussion of the applicability and the limitations of our results \S \ref{section:comparison}.

\tableofcontents

\section{The Problem}
\label{section:motivations}

\subsection{Training Neural Network and the SGD}
\label{section:training}

\paragraph{How we train neural networks and why.}
Neural networks are commonly trained by optimizing a loss function using SGD without replacement, see, e.g., \cite{sun_optimization_2020}, and its more refined variations like Adam \cite{kingma_adam_2014}. SGD without replacement is indeed the default in widely utilized libraries such as PyTorch\footnote{\url{https://pytorch.org/docs/stable/optim.html}} and TensorFlow\footnote{\url{https://www.tensorflow.org/api_docs/python/tf/keras/optimizers/experimental/SGD}}.
Its widespread use results from early machine learning research that highlighted its practical generalization capabilities and its competitive computational complexity \cite{bottou_curiously_2009,bottou_stochastic_2012,lecun_efficient_2012}. 
Indeed, SGD without replacement is much faster, and usually better, than Gradient Descent and SGD with replacement. Precisely:
\begin{itemize}
    \item It converges with fewer steps than, e.g., SGD with replacement on many practical problems, see \cref{fig:intro} and \cite{bottou_curiously_2009,recht_parallel_2013,gurbuzbalaban_why_2021,haochen_random_2019, mishchenko_random_2020}.
    \item The step itself is much faster. It just needs direct access to the memory, not sampling as for SGD with replacement \cite{bengio_practical_2012}.
\end{itemize}

For instance, in strongly convex settings SGD without replacement has better convergence than SGD with replacement. In particular, if the number of epochs is bigger than $O(\sqrt{n})$ then the convergence rate of SGD without replacement is faster and, asymptotically, is $O(1/time^2)$, while the one of SGD with replacement is $O(1/time)$.    \cite{bottou_curiously_2009,gurbuzbalaban_why_2021,haochen_random_2019, mishchenko_random_2020}. Additionally, the oscillation at convergence is often smaller for SGD without replacement, frequently as low as  $O(\eta^2)$ compared to $O(\eta)$ for SGD with replacement, see \cref{fig:intro} and \cite{mishchenko_random_2020}.

\subsection{Implicit Regularization to Explain Generalization}
\label{section:generalization}

\paragraph{Generalization.}
Two primary objectives of machine learning theory are to understand why neural networks generalize and how they internally represent the data and functions they learn. Traditional generalization theories face various challenges in deep learning, see, e.g., \cite{bartlett_nearly-tight_2019}. 
Some of them being overparameterization \cite{zhang_understanding_2017,neyshabur_exploring_2017}, benign overfitting \cite{bartlett_benign_2020}, or double descent phenomenon \cite{belkin_reconciling_2019}.
Consequently, these challenges necessitate a re-conceptualization of generalization theory, and probably an optimization-dependent one, \cite{zhang_understanding_2017,zhang_understanding_2021}. 

\paragraph{Implicit Regularization.}
{Indeed, an intriguing observation is that, in many relevant settings, different optimization procedures converge to global minima of the training loss. However, they consistently lead to different levels of test loss, or neural networks that represent functions of different kinds of regularity, see e.g. \cite{gunasekar_implicit_2017,neyshabur_exploring_2017,blanc_implicit_2019}.
Moreover, \cite{jiang_fantastic_2019} and previous works note that empirically the features of the optimization procedure and trajectory are among those that best correlate with generalization.
This motivated the machine learning/optimization community to shift their interests.
From studying convergence \textit{rates} of algorithms in regular (such as convex) settings, to investigating the \textit{location} of convergence for non-convex landscapes with multiple stationary points; that is, the implicit effect of the optimizer.}

\subsection{Previous Work and Challenges}
\label{section:challenges}

For the reasons above, part of the research community focused on understanding the role of the optimizer and its interaction with the optimization landscape.
The primary objectives are to (i) gain theoretical insights into the role of noise, the interplay between algorithms and particular landscapes, the effect of the discretization, and (ii) develop improved algorithms with better performance.
There are various mathematical challenges in trying to perform an analysis of the trajectory of SGD, indeed, many standard assumptions are not satisfied in this setting.
The three ingredients that have a role in determining the training trajectory are the geometry of the landscape, the effect of discretization, and the effect of the noise. We quickly review the role of these ingredients and the related mathematical challenges in what follows.

\paragraph{The landscape.}
Some phenomena occurring during the training of large-scale machine learning models are more influenced by the geometry of the landscape rather than the action of discretizing the gradient flow or the noise. 
At first impact, in modern machine learning, the loss landscapes seem really not to satisfy any standard regularity condition:  (i) they are not even locally convex \cite{liu_toward_2020,safran_effects_2021}; (ii) many spurious local minimizers are present \cite{safran_spurious_2018,safran_effects_2021}, and (iii) there are also many strict and high-order saddles \cite{kawaguchi_deep_2016}, even though the used algorithms empirically tend to skip them \cite{dauphin_identifying_2014,goodfellow_qualitatively_2015}, \cite[chapter 8]{goodfellow_deep_2016}, \S \ref{section:saddles}. 
This means that, generally, it is not possible to leverage a strong characterization of the geometry of the manifold on which the trajectory lies, unlike, e.g., for convex optimization. 
This implies that the only feasible way to analyze the trajectory is \textbf{locally}. Moreover, the local analysis may have no implications on any global convergence. 
Another difficulty in dealing with the landscape is, e.g., that it is not always possible to control the effect of nonlinearity. For instance, in ReLU networks, predicting the landscape after the activation pattern changes is challenging. 
A very interesting phenomenon known as mode connectivity has also been observed: local minimizes, and stationary points in general, of a neural network's loss function, are connected by simple paths \cite{freeman_topology_2017, garipov_loss_2018, draxler_essentially_2019}, particularly in overparameterized models \cite{venturi_spurious_2020, liang_understanding_2018, nguyen_loss_2019, nguyen_connected_2019,kuditipudi_explaining_2020}.
All this means that the successful algorithms can escape saddles quickly and can travel through flat areas fast towards a better generalizing minimum. GD does not do it. SGD, instead, empirically seems particularly suitable for these landscapes (\cref{fig:intro,fig:1}), as our results show, see, e.g., \S \ref{section:regularization} and Proposition \ref{prop:flat_areas}.

\begin{figure}[ht!]
\begin{mdframed}
    \begin{minipage}{.5\textwidth}
        \centering
        \includegraphics[clip, trim=0cm 0.8cm 1cm 3.5cm, width=0.99\linewidth, height=1.3\linewidth]{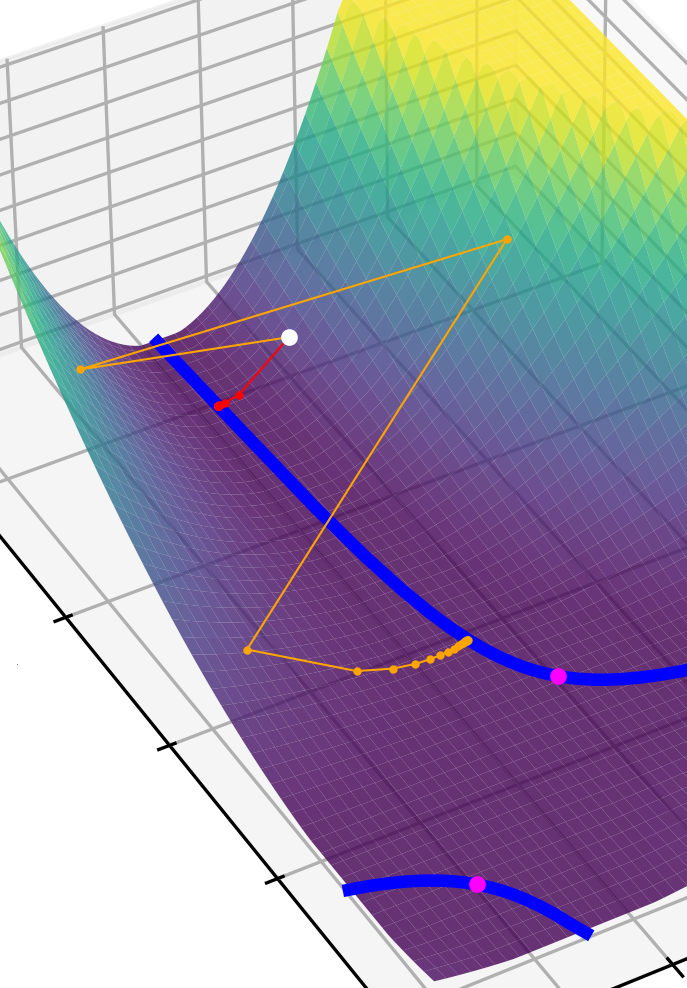}
    \end{minipage}
    \begin{minipage}{.5\textwidth}
        \centering        \includegraphics[width=0.99\linewidth,height=1.3\linewidth]{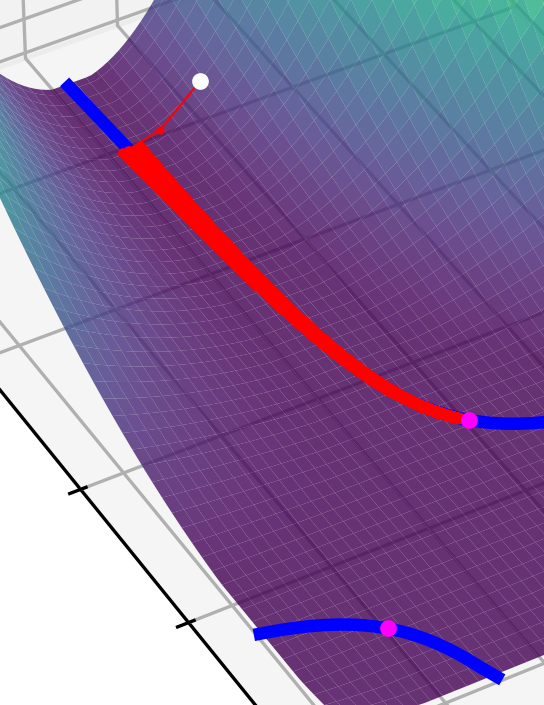}
    \end{minipage}
        \caption{
        The setting is the same as \cref{fig:intro} outlined in \S \ref{section:setting_plots}. On the left we have the trajectories of full-batch GD for \red{small} (red) and {\color{orange}big} (orange) learning rates. On the right the trajectory of SGD without replacement with small learning rate until convergence.
        On the left, we see that by increasing the step size, the algorithm identifies a solution closer to the lowest norm one. This implicit regularization effect arises from discretization, consistent with \cite{lewkowycz_large_2020, jastrzebski_catastrophic_2021}.
        However, no matter the learning rate, we see that GD stops as soon as it gets to a stationary area, while SGD without replacement navigates with an oscillatory trajectory the manifold of minima, converging to the lowest norm one.
        }
        \label{fig:1}
\end{mdframed}
\end{figure}

\paragraph{The role of the discretization.}
Some of the implicit regularization effects happening during the training are due to the discretization process.
A moderately large learning rate indeed is empirically observed to yield better generalization performance \cite{lecun_efficient_2012,lewkowycz_large_2020,jastrzebski_catastrophic_2021}. This means that the effects of discretization and noise are generally benign.
These effects include the one described in \cite{cohen_gradient_2021,cohen_adaptive_2022}, known as the edge of stability: often GD will start oscillating around the manifold of minima instead of following the gradient flow, while anyway steadily converging \cite{damian_self-stabilization_2023}. 
All this suggests that a compelling analysis must include the cases of non-vanishing yet \textbf{finite learning rate}, a \textbf{large Hessian}, or alternatively, $(learning \ rate) \cdot Hessian > 1$.
In classical numerical analysis, the effects of discretizing differential equations have been studied for long. A useful tool for that \cite[Chapter IX]{hairer_geometric_2006} is backward error analysis, introduced with the pioneering work of \cite{wilkinson_error_1960}. While in the 1990s it was used to study the stability of the discretizations, these techniques have been recently used to tackle implicit regularization of GD \cite{barrett_implicit_2021}, SGD without replacement \cite{smith_origin_2021}, SGD with momentum \cite{ghosh_implicit_2023}, Adam \cite{cattaneo_implicit_2023}. Unfortunately, this line of work usually applies to small step sizes. This article attempts to get around this limitation.

\paragraph{The effect of the noise.}
The noise of SGD has been considered by numerous authors as a potential factor for explaining generalization in neural networks. Empirically, both the size of the noise \cite{keskar_large-batch_2016, jastrzebski_three_2018} and the shape or direction of the noise indeed play crucial roles \cite{wen_interplay_2019, haochen_shape_2020, smith_origin_2021}.
That said, despite extensive research on this phenomenon, it remains an enigmatic aspect of deep learning theory and needs further investigation. 

What we do know from classical literature is that the fact of having random diffusive fluctuations implicitly biases the dynamics towards flatter areas. This is a well-known phenomenon in physics called thermophoresis, and it is mathematically explained with Fokker-Plank-like equations. There are many works in the ML community about this kind of effect, most notably \cite{chaudhari_stochastic_2018,haochen_shape_2020,orvieto_explicit_2023,chen_stochastic_2023}.
That said, these papers are about dynamics where the noise is regular and well-behaved, e.g.\ every step is independent, the noise does not vanish along the trajectories, etc. In particular, they are dealing with injections of Gaussian noise or dynamics assimilable to geometric Brownian motions.

In practice, the noise at every step is not Gaussian, does not admit lower bounded variance, etc. In particular, it vanishes in the global minima and it \textbf{has a shape} and a structure.
Moreover, in the case of the algorithms used in practice, as SGD without replacement, the noises of different steps \textbf{are not independent}, nor centered.
In particular, in SGD without replacement, the noise comes from reshuffling the dataset once every epoch, not from sampling at every step. The batches are disjoint, thus dependent on each other, and given the first $k-1$ the $k-th$ is deterministic. In particular, batches sampled are thus exchangeable but not independent random variables. We show in this paper that this dependence and shape may have an overlooked benign effect on the training.

Some of the conclusions the community got regarding SDEs trajectories, in some settings, may apply to SGD with replacement \cite{chen_stochastic_2023, li_what_2022}. 
However, SDE approximation limits are ill-posed \cite{yaida_fluctuation-dissipation_2018} and on top of that never behave as SGD apart for extremely particular cases, as very small learning rate for scale-invariant neural networks in a setting in which the variance is bigger than the gradient  \cite{li_validity_2021}. Moreover, \cite{li_validity_2021} questioned the applicability of CLT in this context, so we are not even sure in what terms we can speak of these diffusion-powered effects. On top of all the above, \cite{haochen_shape_2020} showed the differently shaped noises (Langevin dynamics, label noise, etc.) converge to different minima.
Finally, \cite{damian_label_2021,li_what_2022} show that even in the case of diffusive independent noise at every step, as label nose injection or SGD with replacement, the implicit regularization effect is due to drifts coming from the higher order terms in Taylor, not from the diffusion-powered regularization.

To the knowledge of the author, it was thus proved neither mathematically, nor empirically, that these results and a Fokker-Plank-like theory could explain the implicit regularization of SGD with replacement. Moreover, there exists no diffusion-related result applicable to SGD without replacement, nor any empirical observation that that may be the case \cite{smith_origin_2021}.

\subsection{The Real-World Regime}
\label{section:regime}
The product $c := \eta \cdot k$ between the learning rate and the number of steps in an epoch plays a crucial role in the effect of SGD.
Empirically, for instance, studies found that the size of the SGD effect scales with $c$, or analogously with $\frac{\eta}{b-1} = \frac{c}{n}$.
This is the so-called linear scaling rule in the literature \cite{goyal_accurate_2017,jastrzebski_three_2018,he_control_2019}.
Predictably, $c$ comes out in our regularizer too, see \cref{eq:real_world_effect} or \cref{theo:intro}, for instance. 
From a theoretical point of view, the reason is that we can rewrite the Taylor expansion as a sum of products of $c$ times the derivatives on the loss. So, for instance, the first terms can be rearranged as
    \[
    - \ \eta k \nabla L 
    \ + \ \eta^2 \binom{k}{2} \nabla^2 L \nabla L
    \ - \ \eta^3 \binom{k}{3} (\nabla^2 L)^2 \nabla L
    \ + \ O(\eta^3)
    \]
    In every term $\eta$ and $k$ appear at the same order, thus this can be rewritten as
    \[
    - \ c \nabla L 
    \ + \ \frac{c^2}{2} \nabla^2 L \nabla L
    \ - \ \frac{c^3}{6}  (\nabla^2 L)^2 \nabla L
    \ + \ \ldots
    \]
    plus an error of $O(1/k)$.
For this reason, most analyses require $c = \eta \cdot k$ to be much smaller than $1$, or analogously $c \nabla L, c \nabla^2 L \ll 1$. This is the case for instance of \cite{smith_origin_2021,roberts_sgd_2021}.
In practice, however, this is often not the case as $c$ or $c\nabla^2 L$ are not small, see \cref{table:1} and \cite{cohen_gradient_2021}.
The only element of the form $c \nabla^i L$ that we can be sure tends to become small during the middle and later stages of training is thus $\eta k \nabla L$, as the gradient approaches zero and the learning rate, $\eta$, gets annealed. Our analysis only requires a related less stringent assumption \S \ref{section:generality}.

\begin{table}[ht!]
\begin{tabular}{c|c|c|c|c|c}
Dataset
& Model 
& $\substack{ \text{Training} \\ \text{set size }n}$ 
& $\substack{\text{Step sizes } \eta\\ \text{range}}$
& $\substack{\text{Batch size } b\\ \text{range}}$
& $\substack{c \ = \ \eta \cdot k\\ \text{range}}$
\\[1mm]
\hline
    MNIST\footnotemark[1] & MLP & 60k & 0.1 - 0.01 & 32-64 & 190 - 9.4 \\[1mm]
    Cifar10(0)\footnotemark[2] & DenseNet-BC-190\footnotemark[4] & 50k & 0.1 - 0.001 & 64 & 78 - 0.8\\[1mm]
    Cifar100 \footnotemark[2] & ResNet-BiT\footnotemark[5] & 50k & 0.03 - 0.0003 & 4096 & 0.36 - 0.004\\[1mm]
    ImageNet\footnotemark[3] & ResNet152\footnotemark[6] & 1.2M & 0.1 - 0.001 & 256 & 470 - 4.7 \\[1mm]
    \hline
    \hline
\end{tabular}
\caption{
We list some commonly used datasets, their SotA models, and the choice of hyperparameters used in training.
In these cases, we show that the size of $c$ may range between units and hundreds. Note that in most of the cases in the table above, SGD is equipped with momentum with $\beta = 0.9$. This implies that the effective step is bigger $\eta \ (1 + 0.9 + 0.81 + \ldots) = 10 \cdot \eta
$ making the effective $c$ ten times bigger.
}
    \label{table:1}
\end{table}

\footnotetext[1]{\url{http://yann.lecun.com/exdb/mnist/}}
\footnotetext[2]{\url{https://www.cs.toronto.edu/~kriz/cifar.html}}
\footnotetext[3]{\cite{deng_imagenet_2009}}
\footnotetext[4]{\cite{huang_densely_2018}}
\footnotetext[5]{\cite{kolesnikov_big_2020}}
\footnotetext[6]{\cite{he_deep_2015}}


\section{Implicit Bias of SGD Without Replacement}
\label{section:without_repl}

We show in this section that relative to full-batch GD or SGD with replacement, SGD without replacement and shuffle once implicitly add a step on a regularizer that penalizes a weighted trace of the covariance of the loss gradients.
This is not the first such result but we go way beyond what is known in the literature. Indeed, \cite{roberts_sgd_2021} showed earlier that SGD has an implicit bias and gave some insights on the generalization benefits of it, while \cite{smith_origin_2021} showed that in a small learning rate regime SGD without replacement in expectation follows a gradient flow on a modified loss.

Technically, we analyze the optimization dynamics of SGD without replacement of one epoch, or less, at once.
\begin{itemize}
    \item \textbf{Deviation:} We first compute the deviation between trajectories Proposition \ref{theo:SGD_effect_fixed_batches}.
    \item \textbf{Moments:} Then we determine its moments with the machinery developed in the appendix \S \ref{section:exp_app}.
\end{itemize}%
In what follows, we set the notations in \S \ref{section:notations} and proceed stating our main results.

\subsection{Notations and Setting}
\label{section:notations}
We denote by $D = \{ z_i \}_{i = 1, 2, \ldots, n} \subseteq V$ the training set of size $n$. We have a parametric function  (loss) $L \colon (\vartheta, z) \mapsto L(\vartheta, z) \in \R$ that takes as input the parameters $\theta \in \Theta$ and the data $z \subseteq V$ and outputs a real number, admitting 3 weak derivatives in the parameters $\vartheta$.
For every set $B \subseteq V$ we define $L(\vartheta, B) := \tfrac{1}{|B|} \sum_{z \in B} L(\vartheta, z)$. The goal of the optimization procedure is to find a minimum $\theta^* \in \Theta$ for the objective function $L(\vartheta, D)$.

For readability, we will often omit the inputs of the function $L$. Precisely, whenever we omit the parameter $\vartheta$ we are evaluating the loss at the beginning of the epoch, we will denote the value of the parameters at the beginning of the epoch with $\theta$.
Whenever we omit the set of inputs we are evaluating the loss over the whole training set $D$. Moreover, all the derivatives we will take will be in the parameters $\vartheta$ and all the expectations will be empirical expectations over a set $B \subseteq V$.
As an example
\[
L \text{ means }L(\theta, D) \quad \qandq \quad \nabla L(B) \text{ means } 
 \nabla_\theta L(\theta, B).
\]
We denote by $\nabla^i L$ the tensor of the i-th derivative of the function $L$ in the parameters.
For every parameter vector $\theta$, we will work without loss of generality in an orthonormal basis of eigenvectors of the full batch Hessian. Precisely, we denote as $\theta_i$ the eigenvector for the $i$ biggest eigenvalue $\lambda_i$ of the Hessian $\nabla^2 L$.

Moreover, we will denote by $\theta^{SGD}(t)$ the parameters after the $t-th$ step of SGD without replacement or, analogously, of Shuffle Once\footnote{Shuffle Once is the version of mini-batch SGD where the training set is shuffled at the beginning of the training, partitioned, and at every epoch uses the same batches in the same order.} (SO) starting from $\theta^{SGD}(0) = \theta$ with learning rate $\eta > 0$.
In this section, we will take also other algorithms in exam, such as SGD with replacement, GD, or GD with noise injected. Analogously to SGD without replacement, we will denote by $\theta_i$ the parameters after the $i-th$ step of the other algorithms considered, with initialization $\theta(0) = \theta$ and the learning rate $\eta > 0$.

Let $b \in \N$ be the batch size and $k \in \N$ the number of optimization steps considered at once. Generally, $k \leq n/b$ will be and can be thought of as, the number of batches or the number of steps in an epoch. We denote by $B_i \subseteq D$ the batch used in the $i-th$ step, $i \leq k$.
For a vector or matrix we denote the components with the indices in subscripts, e.g., for $A \in \R^{d\times d}$, for $i,j \leq d$, we denote $A_{i,j}$ the $(i,j)$ component of the matrix $A$. 
Recall $c:= \eta k$ is the "effective learning rate".
We denote by $H : = \frac{c^2}{2n}\big( \E[(\nabla^2 L)^2] - (\nabla^2 L)^2 \big)$.

\subsection{A Weighted Trace of the Covariance}
Assume in this subsection that the product between the matrix $H$ and the full batch gradient $\nabla L$ is small. This is often the case, indeed, e.g., for Gaussian data in late training we have that the size of it is much smaller than the regularizer.
\[
\text{Size of }H \nabla L \quad \sim \quad \frac{c^2}{n} (\nabla^2 L)^2 \underbrace{\nabla L}_{\to \ 0}
\quad \ll \quad \frac{c}{2n} \big[I + c\nabla^2 L + \ldots \big] \cdot \nabla \underbrace{\mathrm{trace}\big( \Cov( \nabla L(z) \big)}_{\sim \ bigger }
\]
the RHS being the size of the regularizer.
In this setting, we find that SGD without replacement, in expectation over batch sampling makes an additional step along every direction on a novel regularizer.
These additional steps can be interpreted as penalizing a weighted trace of the covariance of the gradients over the single data points.
\begin{theorem}
\label{theo:SGD_bias}
In the notations of \S \ref{section:notations} assume $H\nabla L \ll 1$. Up to a multiplicative error of size $\eta^{\alpha} k \nabla L + 1/n$ with $\alpha \geq 1$, after one epoch of SGD without replacement starting from $\theta$ the parameters are different in expectation than the ones after the same number of steps of GD or SGD with replacement with the same learning rate and number of steps. This difference on the $i-th$ component $\theta_i$ of the parameters corresponds to
\[
-\frac{\eta}{b-1} \cdot \frac{d}{d\theta_i}
\mathrm{trace}\big( S_i \cdot \Cov_{z \in D}( \nabla L(z)) \big)
\]
where $S_i$ is a diagonal matrix and it is a function of the full batch Hessian $\nabla^2 L$.
\end{theorem}

We now dig into the effect and shape of $S_i$.
Let us fix from now on an orthonormal basis $\{ \theta_1, \theta_2, \ldots \}$ of eigenvectors for the Hessian $\nabla^2 L$. In the case in which $H \ll 1$ at $\theta$ stationary point, we can rewrite this bias along the eigenvector $\theta_i$ corresponding to the eigenvalue $\lambda_i$ as
\begin{equation}
    - \frac{\eta}{b-1}\ \sum_{j}
    (S_i)_{j,j} \ \frac{d}{d\theta_i} Cov_{z \in D}(\nabla L(z))_{j,j} 
\end{equation}
where
\begin{equation}
    (S_i)_{j,j} \quad = \quad
    \frac{1}{2}
    \sum_{h=0}^{k+2} (-\eta \lambda_i)^h\sum_{l = h+2}^k \binom{k}{l} (-\eta \lambda_j)^{l-h-2}.
\end{equation}
Note that the entries $(S_i)_{j,j}$ are approximately
\begin{equation}
\label{eq:R_ij}
    (S_i)_{j,j}
    \quad \sim \quad \frac{c}{2}\cdot
    \begin{cases}
        \frac{1}{2}    
        \!\! \qquad \qquad \qquad \qquad \qquad \quad \text{ constant}   \qquad
        & \text{ if } c\lambda_i, c\lambda_j \text{ small}\\
        (c\lambda_i)^{-1}    
        \hspace{0.04cm}\quad \qquad \qquad \quad \qquad \text{ small}
        & \text{ if } c\lambda_i \gg 0, c \lambda_j \text{ small} \\
        (c\lambda_j)^{-1}       
        \quad \qquad \qquad \qquad \quad \text{ small}
        & \text{ if } c \lambda_i \text{ small}, c\lambda_j \gg 0 \\
        \frac{1 + (-\eta \lambda_i)^{k}}{c\lambda_i \cdot c\lambda_j} \quad \qquad \qquad \qquad \ \,
        \substack{\text{very small }\\\text{or very big}}
        & \text{ if } c\lambda_i, c\lambda_j \gg 0 \\
        \Omega\big((c\lambda_j)^{-2}\exp(-c\lambda_j)\big)
        \qquad \substack{\text{exponen-}\\\text{tially big}}
        & \text{ if } c\lambda_j \ll 0
    \end{cases}
    \end{equation}

\subsection{Small Learning Rate or Small Hessian Regime.}
\label{section:bias_small}
Assume that the full-batch Hessian multiplied by the effective learning rate $c\nabla^2 L \ll 1$ is small but has arbitrary "variance" $H$. 
Then we can conclude that the regularizer's step coincides with 
\[
-\frac{\eta}{b-1} \nabla \E_{z \in D} [ \mednorm{\nabla L (z) - \nabla L}_{S_0}]
\]
where $S_0$ is
\[
S_0 \quad = \quad
    \frac{c}{4}
    H^{-1/2} \sqrt{\pi} \ \mathrm{erf}(\sqrt{H}) \left( 2+ 2H^{-1/2}(e^{-H}-1) - c\nabla^2 L\right)
    .
\]
When also $H \ll 1$ this becomes approximately $S_0 = \frac{c}{4}I$ and the regularizer is
\[
\frac{c}{4} \var{\mednorm{\nabla L}}.
\]
This extends prior theoretical findings by \cite{smith_origin_2021} and empirical findings by \cite{jastrzebski_catastrophic_2021}.

\begin{theorem}
\label{theo:SGD_bias_eta}
In the notations of \S \ref{section:notations}, let us assume that $\eta k \nabla L(.) \ll 1$.
Then up to a multiplicative error of size $\eta^{\alpha} k \nabla L + 1/n$ with $\alpha \geq 1$,
in expectation over batch sampling, $k$ steps of SGD without replacement and Shuffle Once\footnote{Shuffle Once is the version of mini-batch SGD where the training set is shuffled at the beginning of the training, partitioned, and at every epoch uses the same batches in the same order.} differ from the same number of steps of SGD with replacement or GD, by the following additional preconditioned steps with learning rate $\tfrac{\eta}{b-1}$ on some regularizers
\[
\quad - \quad \frac{\eta}{b-1} \ \sum_{i = 0}^{k-2} A_i
\nabla \text{Regularizer}_i
\]
where $A_i = \E_{\left[ \substack{\text{expectation over}\\ \text{batches' sampling}} \right]} \left[ \prod_{j > i}-\eta \nabla^2 L(B_j) \right]$ and 
\[
\text{Regularizer}_i \quad = \quad
\E_{z \in D}[ \mednorm{\nabla L(z) - \nabla L}_{S_i}  ]
\ + \
\nabla L^\top \mathcal{S}_i \nabla L .
\]
Where the matrices $A_i, \ S_i$, and $\mathcal{S}_i$, $i \leq k-2$, are approximately series of powers of $\nabla^2 L$ and $H$.
\end{theorem}

The proof of these results can be found in the appendix. We will give a deeper look into the nature of the regularizer in \S\ref{section:nature} and of the matrices $S_i$s in Proposition \ref{prop:S}. We deal with the error and the regimes in which \cref{theo:SGD_bias,theo:SGD_bias_eta} apply in \S \ref{section:error}.

\section{Implicit Regularization}
\label{section:regularization}

There are three notions of low complexity minima that have been discussed in literature: (i) the lowest norm minimum, i.e., the closest to the origin; (ii) the flattest minima, i.e., the one which some norm of the Hessian of the loss is minimized; and (iii) the "lowest variance" minima, i.e., the one that minimizes some function of the eigenvalues of the covariance of the gradients. In certain cases these three notions coincide, this is for instance the case of optimizing $\sum_{(x,y)}(abx-y)^2$ in $a,b \in \R$ which is the setting of \cref{fig:intro}. In some cases, as for overparameterized linear regression, the covariance of the gradients and the Hessian are the same on any solution and what it matters is just the norm of the minimum we find. In other cases these may differ. We argue here that SGD find a minimum of the third kind and what is the machinery that induces this effect.

\subsection{Decoupling Descent and Regularizing Phases}
Theorem \cref{theo:SGD_bias} indicates that using SGD effectively adds a regularization term proportional to a weighted trace of the covariance matrix across different parameter dimensions. While the Hessian of the loss is not explicitly regularized, the regularization effect intensifies in the directions relative to its smallest eigenvalues.
The role of $S_i$ in \cref{theo:SGD_bias}, indeed, is to approximately project or restrict the part of the dynamics relative to the regularizer on the flat directions, i.e., the kernel of the Hessian. In the case in which either big (but less thant $\eta^{-1}$) or small eigenvalues are present in the Hessian, as \cref{fig:intro}, $S_i$ is approximately zero along the directions admitting big curvature and $c/4$ along the eigendirections relative to the small ones, see \cref{eq:R_ij}. For a vector $v$
\[
S_i \cdot v \quad \sim \quad \frac{c}{4} \cdot \underbrace{\Pi_{flat}}_{\text{projection}} (v) \qquad \substack{\text{projects the vector } v \text{ along the directions}\\
\text{in which the Hessian is degenerate.}}
\]
This means that if two minima or stationary points are in the same connected manifold of stationary points, SGD without replacement will move towards the one admitting a lower trace of the covariance, so the flatter one. 

This means that we can understand the dynamics of SGD without replacement from an insightful and completely novel point of view.
\begin{mdframed}
    Assume SGD without replacement is traveling an approximately flat area. The dynamics along the eigendirection of the eigenvalue $\lambda$ is dominated by
    \begin{itemize}
        \item The \textbf{regularization} effect if $c \lambda$ is small or negative.
        \item The trajectory of \textbf{SGD with} replacement (and in expectation of GD) if $c \lambda$ is big but the assumptions of Proposition \ref{prop:breaking_eos} are not satisfied, e.g., if $\lambda \leq 1/\eta$.
        \item The \textbf{regularization} effect, which is very intense, in the case of Proposition \ref{prop:breaking_eos}.
    \end{itemize}
\end{mdframed}

\subsection{Suitability of SGD for Overparameterization}
Penalizing the weighted trace of the covariance in the eigendirections of the small eigenvalues of the Hessian corresponds to traveling \emph{along the flat directions} in the direction that makes the gradients less noisy. To have a mental picture of this behavior, this means that doing SGD without replacement coincides with: using GD on the directions spanned by the signal, regularizing along the perpendicular directions. In contrast, GD would live these components untouched as it happens with the perpendicular components in overparameterized linear regression.

This decoupling may explain the suitability of SGD without replacement for non-convex objective/loss functions with the same geometry as the ones of large scale neural networks.
Recent findings \cite{liu_toward_2020,nguyen_connected_2019,zhang_embedding_2022}, indeed, indicate that the loss-landscapes of neural networks present highly degenerate Hessians figuring high dimensional manifolds of stationary points.
Most of these stationary points generalize poorly,
SGD, however, converges to almost-global loss minima that generally generalize better \cite{gunasekar_implicit_2017,neyshabur_search_2015}.
We believe that the discovery that SGD without replacement travels flat manifold quickly with this machinery shed lights on why the algorithm is so successful in training large scale neural networks.

\subsection{The Phases of SGD in the Experiments}
The behavior highlighted above can be observed in \cref{fig:intro,fig:1} in the setting described in \S \ref{section:setting_plots}. In these plots we see that SGD without replacement approaches the manifold of minima as GD. However, the regularizer starts pushing in the direction in which the trace of the covariance gets smaller. Precisely, SGD without replacement, once around the manifold of minima, starts traveling in the direction of the wider (less noisy) minima that are present on that manifold.

This phenomenon can also be seen in the experiment with ReLU networks on the synthetic dataset of \cref{fig:intro_2,fig:W_without}. 
In that case we are even able to highlight which parts of the dynamics are led by the regularization and which by the gradient descent effect.

\begin{figure}[ht!]
\begin{mdframed}
    {\centering \large \textbf{Decoupling the Dynamics: the Phases}\\[0.3cm]}
    \centering
    \includegraphics[width=\linewidth, height=0.4\linewidth]{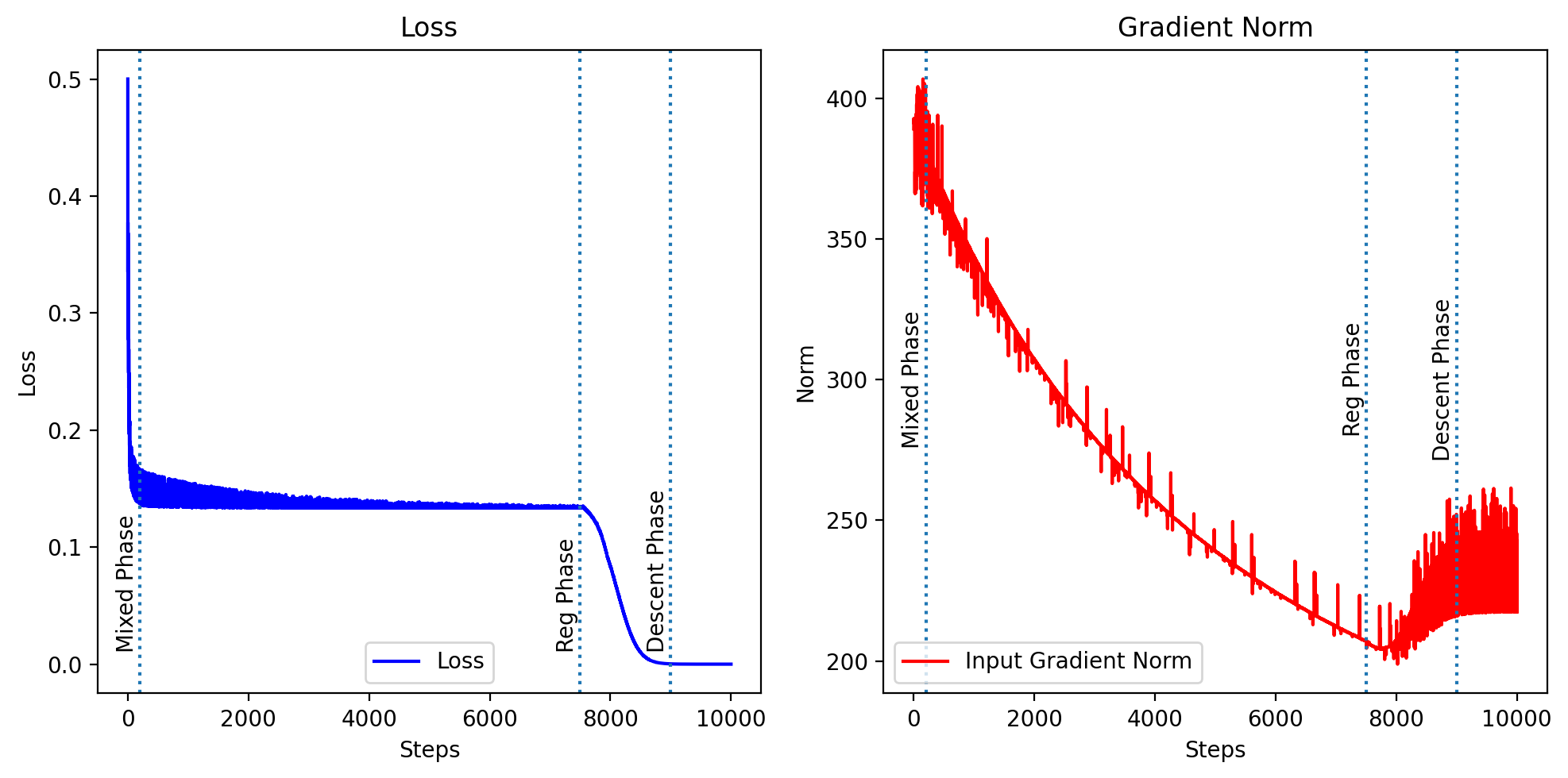}
    \caption{This is dynamics of SGDs without replacement for ReLU network on a synthetic dataset with fixed learning rate.\\
    We can observe how first we have a convergence (mixed) phase, then a regularization phase, then a convergence phase again which ends up in a global minimum. Convergence to the global minimum was impossible without the regularization phase, e.g., in the case of GD.
    }
    \label{fig:phases}
\end{mdframed}
\end{figure}

\subsection{The Speed of Regularization}
We discussed how SGD without replacement travels flat regions as neighborhoods of high-order saddles or spurious minima thanks to the steps on the regularizer. We formalize this here in the following proposition. We prove under what assumption SGD travels these flat manifolds and with which speed.
\begin{proposition}[Traveling Flat Regions.]
\label{prop:flat_areas}
Let $v$ be a vector in the kernel of the Hessian of the loss. Assume $\mednorm{v}=1$ and
    \[
    u
    \quad := \quad
    \frac{1}{n} \sum_{z \in D} \langle v, \nabla^2 L(z) \nabla L(z) \rangle
    \quad \neq \quad 0.
    \]
Let us assume also that the third derivative is bounded in a neighborhood of the trajectory.
Then, SGD without replacement travels distance $1$ in the direction of $v$, i.e., $\langle\theta^{SGD}_* - \theta, v \rangle = 1$ in a number of epoch
        \[
            \# \text{epochs}
            \quad = \quad
             \frac{2b}{\eta c^2 |u|}.
        \]
\end{proposition}

\section{Shaping the Hessian}
\label{section:Hessian}
We show here that \cref{theo:SGD_bias_eta} is potentially enough to explain several empirically observed phenomena about implicit regularization.

\subsection{Literature Review}
One of the papers that started the whole line of research on the implicit regularization of algorithms was \cite{keskar_large-batch_2016}. They empirically observed in vision tasks that a lower batch size led, generally, to a wider (or less sharp) minimum.
Precisely, they noted that to minima found by smaller batch sizes correspond smaller values of a quantity which is generally increasing in the number of large eigenvalues or in the size of the trace of the Hessian. 
Similarly, \cite{jastrzebski_catastrophic_2021}, observed that big learning rate SGD has an effect similar to penalizing
\[
\mathrm{trace} \! \left( \E \left[   \nabla L  \nabla L^\top   \right] \right)
\quad = \quad
\E \left[   \mednorm{\nabla L}^2   \right]
\]
that in the image classification task in which they work coincides with the trace of the Fisher matrix.
In related settings, the Fisher matrix has been shown to approximate the Hessian during the training; in particular, there is an overlap between the top eigenspaces of the Hessian and its eigenspaces \cite{jastrzebski_three_2018, martens_new_2020, thomas_interplay_2020}.
Furthermore, \cite{jastrzebski_catastrophic_2021} shows that, in practice, penalizing it consistently improves generalization, reduces memorization, and regularizes the trace of the final Hessian. 
There exist multiple other studies along these lines. Some recent ones include \cite{lu_benign_2023} who observe that a bigger learning rate "prevents the over-greedy convergence and serves as the engine that drives the learning of less-prominent data patterns", this aligns with the regularization phenomenon that we unveil. Analogously, \cite{geiping_stochastic_2021} showed how full-batch GD strongly regularized perform as SGD explicitly using a penalization very similar to the regularizer we find.

\subsection{Empirical Observation}
When training ReLU networks often the trajectory quickly converges to flat areas (saddles or spurious minima) and then travels slowly the flat area until it escapes, see \cref{fig:W_without} and the important discussion started by \cite{power_grokking_2022}.
Our intuition is that in the phases of quick convergence the trajectory follows gradient flow or gradient descent, while in the others the area is quite flat, \cref{theo:SGD_bias_eta} applies, and the regularization effect it unveils has visible effects.
\begin{mdframed}
    This argument may imply that the implicit regularization effect highlighted in \cref{theo:SGD_bias_eta} may be the key ingredient in explaining why and how
    SGD converges to almost-global loss minima even though spurious minima of the loss exist \cite{dauphin_identifying_2014,safran_spurious_2018,zhang_embedding_2022}. See \cref{fig:W_without} for an experiment.
\end{mdframed}

Moreover, as we can see in \cref{fig:W_without}, SGD often has the effect of shrinking to zero the already small eigenvalues of the Hessian. 
We show below in \S \ref{section:keskar} that this can also be seen as a consequence of \cref{theo:SGD_bias_eta}.
In particular, the effect unveiled may be even better than other possible regularization effects as it shrinks to zero the already small eigenvalues that may be related to overfitting, memorization, or noise. However, it leaves untouched (unlike, e.g., weight decay) the big ones, usually carrying information about the classes and the learning task \cite{papyan_traces_2020}.
\begin{mdframed}
    This argument implies that the implicit regularization effect highlighted in \cref{theo:SGD_bias_eta} may be the key ingredient in explaining why and how SGD shapes the spectrum of the Hessian of the loss, e.g., producing clusters of large outlying eigenvalues and sending small eigenvalues to zero in the course of training
    \cite{keskar_large-batch_2016,sagun_eigenvalues_2016, papyan_traces_2020,jastrzebski_catastrophic_2021}.
\end{mdframed}

\begin{figure}
\begin{mdframed}
{\centering \large \textbf{Evidences of Regularization: 
SGD Escapes Local Minima}\\[0.3cm]}
    \begin{minipage}{.5\textwidth} 
        \centering
        \includegraphics[width=\linewidth,height=0.7\linewidth]{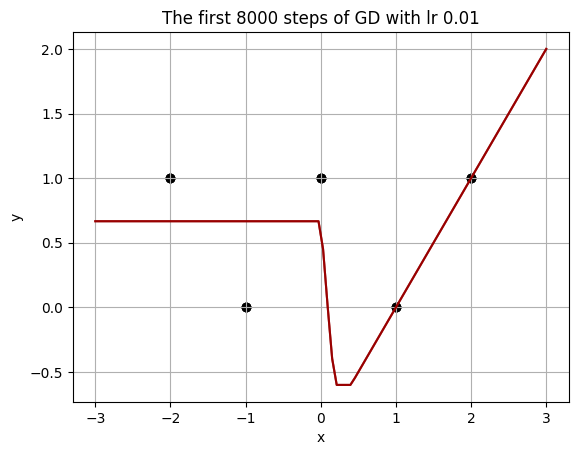}
    \end{minipage}
    \hfill
    \begin{minipage}{.5\textwidth} 
        \centering
        \includegraphics[width=\linewidth,height=0.7\linewidth]{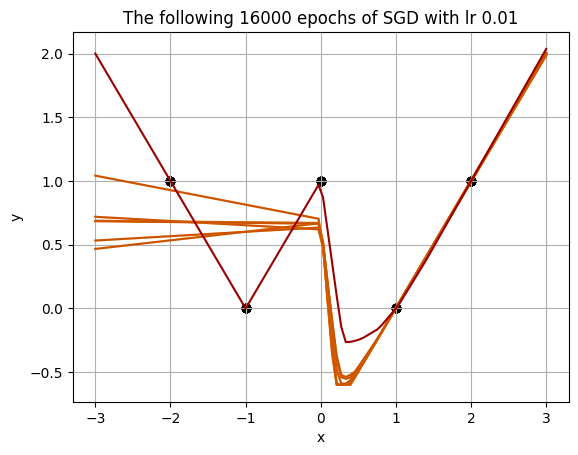}
    \end{minipage}
    \caption{
    \small
        We see here the SGD without replacement escapes local minima to which GD converged. Precisely, we can see a regularization phase while traveling a flat area followed by a grokking effect that implies convergence to a global minimum.
        \\
        We fit "W" shaped one-dimensional dataset with a shallow-ReLU network and MSE. We run GD which converges to the function represented to the right (in less than 400 steps). This point is a spurious local minimum. We then run SGD without replacement with the same learning rate.
        The orange functions above are the functions represented by the neural network at intermediate steps, the red function is the function of the neural network at convergence. 
        We can see below that SGD travels the flat area regularizing the top eigenvalues of the Hessian and the gradients of the model in both the parameters space the inputs.
    }
    \includegraphics[width = \linewidth]{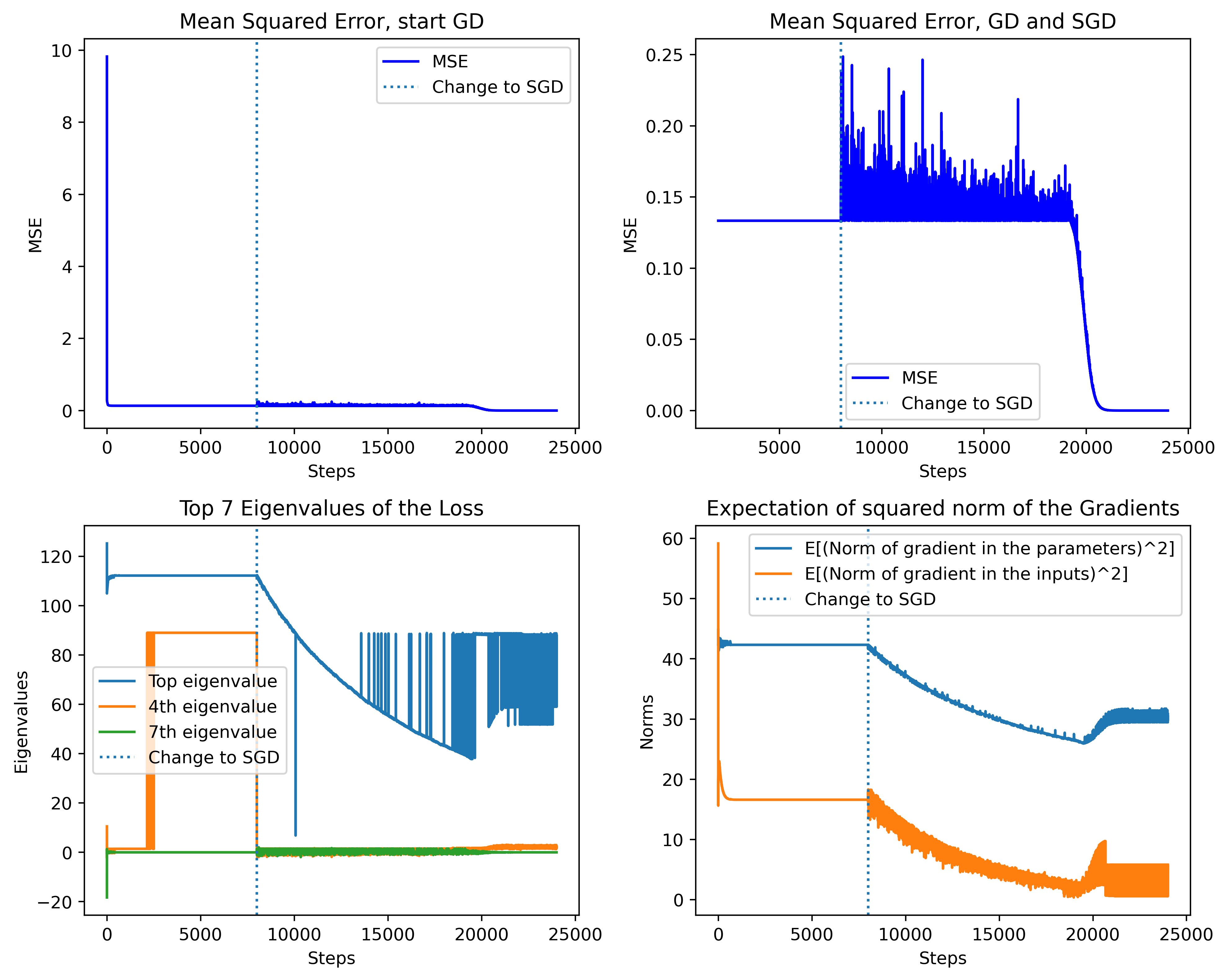}
    \label{fig:W_without}
    \vspace{-0.5cm}
\end{mdframed}
\end{figure}

\subsection{Implicit Regularization due to \cref{theo:SGD_bias_eta}}
\label{section:keskar}

We show here that \cref{theo:SGD_bias_eta} explains why the phenomena discovered \cite{jastrzebski_catastrophic_2021} and \cite{keskar_large-batch_2016} arise from the behavior of SGD without replacement in any flat area of the landscape, that being a saddle, a local minimum, or late training in general. This is the first theoretical result, to the knowledge of the author, that explains and makes more precise these important observations. That said, as a matter of facts, these phenomena may be due \textit{also} to implicit regularization effects that acts in the other phases of the training in which \cref{theo:SGD_bias_eta} is not applicable.

\paragraph{Explaining \cite{jastrzebski_catastrophic_2021}.}
Note that in any flat region of the landscape the covariance of the gradients coincides with matrix considered by \cite{jastrzebski_catastrophic_2021},
\[
\Cov_{z \in D}\big( \nabla L (z) \big) 
\quad = \quad
\E_{z \in D}[ \nabla L (z) \nabla L(z)^\top] - \underbrace{\nabla L \nabla L^\top}_{\sim \ 0}
\quad \sim \quad
\E_{z \in D}[ \nabla L (z) \nabla L(z)^\top].
\]
Moreover, \cref{theo:SGD_bias_eta} implies that SGD without replacement implicitly penalizes terms of the following form with a strength that is directly proportional to $\eta^2n/b^2$.
\[
\mathrm{trace}\big(S \cdot Cov(\nabla L) \big) 
\quad \sim \quad
\mathrm{trace}\big(S_i \cdot \E_{z \in D}[ \nabla L (z) \nabla L(z)^\top] \big).
\]
This formalizes and explains the observations of \cite{jastrzebski_catastrophic_2021}. Indeed, this second matrix corresponds to the Fisher information matrix in the case of cross entropy and they empirically observed that a higher effective learning rate $\eta/b$ (corresponding to bigger steps on the regularizer we found) better regularizes its trace.

\paragraph{Explaining \cite{keskar_large-batch_2016}.}
\cite{thomas_interplay_2020} showed that the trace of the Hessian is approximately the trace of the Fisher matrix along the learning trajectory of image classification tasks. 
The surprising part of \cite{thomas_interplay_2020} is that this is the case along most of the trajectory and not only in the final part, as shown by the following (well-known) theorem. This means that generally in the first part of the training SGD already correctly learns the vast majority of the labels and the major part of the learning is just regularization by SGD and other elements of the training.
\begin{proposition}[Regularizing the trace of Hessian]
\label{theo:reg_Hessian}
    Assume all the training data are correctly classified and we use cross-entropy loss. Then
    \[
    Fisher \quad = \quad \E_{z \in D} [ \nabla L \nabla L^\top] \quad = \quad \nabla^2 L \quad = \quad Hessian.
    \]
    In particular, the local minima of $\mathrm{trace}(S \cdot \text{Fisher})$ coincides with the local minima of $\mathrm{trace}(S \cdot \text{Hessian})$ for all $S$.
\end{proposition}
This is only one of the possible results that link the covariance of the gradients to the Hessian. The general idea is that those two matrices are both linked to the change in the gradients. Indeed, if we have big curvature the gradients will change more in a neighborhood of $\theta$.
This argument and \cref{theo:reg_Hessian} are instrumental to show that once most training points are correctly classified, SGD without replacement effectively minimizes a weighted trace of the Hessian. Indeed the following two regularizers correspond
\[
\mathrm{trace}\big(S \, \cdot \, 
\underbrace{\E\big[ 
\nabla L(z)\nabla L(z)^\top
\big]}_{ \substack{
\text{Empirical}\\
\text{Fisher Matrix}
}}\big)
\quad = \quad
\mathrm{trace}\big(S \, \cdot \, \nabla^2 L(\theta, D)\big)
\]
This is the first mathematical result, to the knowledge of the author, that explains and formalizes the observation of \cite{keskar_large-batch_2016}.
\cref{theo:SGD_bias_eta} thus explains frequent observations that SGD tends to discover minima with a sparsified Hessian, characterized by a few outlying large eigenvalues (although smaller than $\eta^{-1}$) and many smaller eigenvalues near zero, see, e.g., \cite{keskar_large-batch_2016, sagun_eigenvalues_2016, papyan_full_2019}. 
This implies that SGD without replacement has an implicit regularization effect which may be extremely benign for some learning tasks.

\subsection{Connecting Variance, Global Minima, and Flatter Models}
In the example in \cref{fig:W_without}, SGD travels the flat area regularizing the top eigenvalues of the Hessian and the gradients of the model in both the parameters space the inputs.
\cref{theo:SGD_bias_eta} implies that SGD penalizes the trace of the covariance of the gradients. We draw here a connection between these notions.
Note, indeed, that
\[
\mathrm{trace}\big( 
\E\big[ 
\nabla L(z)\nabla L(z)^\top
\big]\big)
\quad = \quad
\E_{z \in D}
\big[
    \underbrace{\mednorm{\nabla L(z)}^2}_{\substack{\text{Gradients'}\\ \text{Norm}}}
\big].
\]
When the loss function is of the form $L = \ell \circ f$ and the output dimension is one, we have that
\[ 
\E_{z \in D}\left[
\mednorm{\nabla L}^2
\right]
\quad = \quad
\E_{z \in D}
\left[
    \nabla_\theta f^\top \nabla \ell^\top \nabla \ell \nabla_\theta f
\right]
\quad = \quad
\underbrace{\E_{z \in D}
\left[
    \sum_i \mednorm{[\nabla \ell]_i}^2 \mednorm{[\nabla_\theta f]_i}^2
\right]}_{\text{on every output component}}
\]
denoting by $\nabla \ell := \frac{d \ell(z)}{dz}|_{z = f(\theta,x)}$ the residuals. As an example, if $L(\theta, z=(x,y)) = \frac{1}{2}\mednorm{f(\theta, x) - y}^2$ we have $\nabla \ell = f(\theta,x)-y$. This means that at every step, SGD without replacement penalizes what essentially is the sum over the output components $i$ of products of $\mednorm{[\nabla f]_i}^2$ and $\mednorm{[\nabla \ell]_i}^2$, i.e., the residuals and norm of the gradient of the model on the components of the output. When the dimension of the outputs is one, this is exactly $\mednorm{\nabla \ell}^2 \mednorm{\nabla_\theta f}^2$. This corresponds to lowering the size of $\nabla \ell$ and/or the size of $\nabla_\theta f$. Thus if we lower the variance we lower this product and viceversa.
So, with different weights given by the PD matrices $S_i$, the regularizer pushes towards areas where either $\mednorm{\nabla \ell}^2$ or $\mednorm{\nabla_\theta f}^2$ are smaller.
Moreover, note that the Hessian of the loss can be rewritten as
\begin{equation}
\label{eq:hessian}
\nabla^2 L \quad = \quad  \nabla^2 \ell \ \nabla f \ \nabla f^\top \ + \ \nabla \ell \ \nabla^2 f.
\end{equation}
This means that if either of the two terms $\nabla f$ or $\nabla \ell$ is smaller, the corresponding part of the Hessian is smaller.
We can thus conclude that the regularizers effect corresponds to

\begin{mdframed}
    The regularizer pushes towards locations in parameter space where either:
\begin{enumerate}[1)]
\vspace{-0.1cm}
    \item The squared residuals $\mednorm{\nabla \ell}^2$ are smaller. \\ I.e., it escapes towards a better-fitting stationary point, as a global minimum, or
\vspace{-0.1cm}
    \item The gradients of the model $\mednorm{\nabla f(\theta,x)^2}$ are smaller. \\
    This, in turn, may correspond to a stationary point with smaller Hessian, a wider minimum, and to a smaller norm of the gradient in the inputs $x$ of the function $f$. Thus potentially finding a better generalizing minimum by making the function represented by the neural network less oscillating and less prone to overfitting and memorization.
\end{enumerate}
\end{mdframed}
This effect was not observed in previous works, to the knowledge of the author, and is the opposite of what we expect from full-batch GD, since \cite[Appendix C]{cohen_gradient_2021} observed that along the trajectory of GD, $\mednorm{\nabla f(\theta,x)^2}$ keeps steadily increasing. Nonetheless, it agrees with empirical observations.

\section{Saddles}
\label{section:saddles}

We believe we make an important step in understanding why and how SGD without replacement escapes saddles so fast in practice. The reason is that the regularizer is simply not affected by most saddles. This is a fundamental difference between SGD with replacement and noise GD, which escape saddles only thanks to the diffusion of the noise.

\subsection{Literature Review}

\paragraph{Saddles are there.} Many theoretical works deal with the presence of saddles in the loss landscape of neural networks. For instance \cite{baldi_neural_1989} proved that in the landscapes of shallow linear networks, all the stationary points that are not global minima are saddle. Later, \cite{kawaguchi_deep_2016} proved the same for deeper networks, under more general assumptions on the data, showing also the presence of higher-order saddles. Later, plenty of work, as \cite{petzka_non-attracting_2020}, characterize large families of saddles (and local minima) in the loss landscape, for example in terms of stationary points of embedded smaller neural networks.

\paragraph{Escaping saddles with noise.}
Many influential works on the optimization side observed how SGD often empirically escapes them \cite{dauphin_identifying_2014,goodfellow_deep_2016}, even though the time required by GD to escape them may often be exponential \cite{du_gradient_2017}. 
An important number of influential theoretical work was produced on trying to explain why and how fast variants of SGD escape (at least the strict) saddles and developing new algorithms that provably escape faster.
For instance, \cite{lee_gradient_2016} proved that almost surely GD escapes saddles asymptotically, and \cite{ge_escaping_2015,jin_how_2017,jin_nonconvex_2021} proved that injecting Gaussian noise in the gradients makes GD escaping in $O(\lambda^{-2})$ time.

\paragraph{A diffusion-powered escape.} The conclusion of many works is that a noised version of GD can escape saddles if the noise is \textit{dispersive}, a concept essentially analogous to having the covariance of the noise positive in an escaping direction. This is for instance the case of Gaussian noise injection. In that case, with a high probability, the noise will eventually shoot the trajectory away from the saddles. Away enough that the gradient descent part of the algorithm will have a considerable size \cite{jin_nonconvex_2021}. 
Following this idea, under the dispersive noise assumption, \cite{ge_escaping_2015}, \cite{daneshmand_escaping_2018}, \cite{jin_nonconvex_2021}, and later \cite{fang_sharp_2019} proved that SGD with replacement escapes saddles in $O(\lambda^{-3.5})$ time, for proper choices of hyperparameters.
The phenomenon described by these works is diffusive in nature, the reason why SGD with replacement is escaping is about the variance term of the i.i.d.\ noise of each step. We are not aware of work on SGD without replacement escaping saddles. We, however, conjecture that it is possible to obtain a result similar to the ones above.

\subsection{SGD Without Replacement Escapes Faster}
We show here that SGD without replacement escapes saddles as well. However, our result is very different in nature from the ones produced in the past. We discover a \textit{drift}-powered escaping effect, not a diffusive one:
\begin{mdframed}
    SGD without replacement, unlike the other algorithms, escapes saddles simply because the implicit step on the regularizer does not vanish there. The regularizer biases the trajectory in escaping directions, if any is spanned by the gradients.
\end{mdframed}

The great news of this section is thus not that SGD may escape saddles, that was known already. The novelty is the way and the speed in which SGD escapes saddles. It is the nature of the effect that makes SGD without replacement escape saddles. SGD without replacement is thus empowered by two weapons against the saddles issue: (i) we believe that in case of dispersive noise, it escapes with a similar \textit{dispersive} machinery than \cite{jin_nonconvex_2021,fang_sharp_2019}, although we are not aware of works in this direction; and (ii) the regularizer induces a \textit{drift}-like effect biasing the trajectory towards escaping directions. The interaction and coexistence of these two effects make SGD without replacement escape faster. Moreover, SGD without replacement escapes even when initialized exactly at saddle points.

\begin{proposition}[Escaping strict saddles.]
\label{prop:strict_saddles}
Let $\theta$ be a higher-order saddle for the loss $L$. Let $v$ be an eigenvector of the negative eigenvalue $\lambda < 0$ of the Hessian of the loss. Assume
    \[
    u
    \quad := \quad
    \frac{1}{n} \sum_{z \in D} \langle v, \nabla^2 L(z) \nabla L(z) \rangle
    \quad \neq \quad 0.
    \]
Let us assume also that the third derivative is bounded in a neighborhood of the trajectory.
Then, SGD without replacement escapes saddle, i.e., the loss is at least $O(1)$ smaller, after
        \[
            \# \text{epochs}
            \quad > \quad
            2 \ \frac{\ln(\eta) + \ln(|u|) + 2\ln(c) - \ln(b) }{c\lambda}.
        \]
\end{proposition}
Analogously, if escaping directions for high-order saddles are spanned by the updates, we escape a higher-order saddle in the same amount of time. 
This is very surprising as it says that no matter the order of the saddle SGD without replacement travels the region with the same speed. Precisely, if a direction of escape is aligned with the direction of movement of Proposition \ref{prop:flat_areas} the saddles gets escaped no matter the order.
\begin{proposition}[Escaping high-order saddles.]
\label{prop:saddles}
Let $\theta$ be a higher-order saddle for the loss $L$. Let $v$ an escaping direction in the kernel of the Hessian of the loss. Assume
    \[
    u
    \quad := \quad
    \frac{1}{n} \sum_{z \in D} \langle v, \nabla^2 L(z) \nabla L(z) \rangle
    \quad \neq \quad 0.
    \]
Let us assume also that the third derivative is bounded in a neighborhood of the trajectory.
Then, SGD without replacement escapes saddle, i.e., the loss is $O(1)$ smaller, after
        \[
            \# \text{epochs}
            \quad > \quad
             \frac{2b}{\eta c^2 |u|}
        \]
independently of the order of the saddle.
\end{proposition}

\subsection{Where is the Catch?}

\paragraph{Conflicting results.}
The findings of this section appear to conflict with previous work demonstrating the difficulty of escaping saddles.
A lot of work in the past, indeed, focused on the difficulty of escaping saddles with gradient-based algorithms, highlighting inherent challenges and inefficiencies of gradient-based methods in navigating the landscape of non-convex optimization.
For instance, it has been shown that it is NP-hard to find a fourth-order local minimum \cite{anandkumar_efficient_2016}. 
Moreover, GD has been shown to potentially take exponential time to escape from saddle points \cite{du_gradient_2017}. This slowdown occurs even with natural random initialization schemes and non-pathological functions.

\paragraph{Why SGD does it so fast.}
SGD without replacement, however, may escape saddle points very quickly as we showed above. The reason why this makes sense is that it is biased. The bias makes it behave as if the loss was not anymore $L$ but $L$ plus a penalization $P$. Thus in a way SGD without replacement implicitly sees the landscape differently and what are saddles for GD on $L$ may not be saddles at all for SGD without replacement. It is important to notice that Proposition \ref{prop:saddles} could already be proved starting from the main result of \cite{smith_origin_2021} in the setting in which that result applies.

\paragraph{The limitation.}
The saddles that we skip are not saddles from the point of view of the algorithm. This, in turn, comes with new challenges. Indeed, while some saddles are not saddles for SGD, some points that were not saddles for the loss may be seen as saddles by SGD, precisely, those points in which the push of the regularizer is exactly the opposite of the push of gradient descent. For those, all the previous negative results in theory apply.

\section{At the Edge of Stability}
\label{section:eos}

We propose here an explanation for why SGD does not train at the edge of stability. The additional drift implied by \cref{theo:SGD_bias_eta} becomes bigger in size than the GD step when the Hessian presents two eigenvalues of size $O(\eta^{-1})$. This means that SGD without replacement deviates from the trajectory of GD in a phase transition and eventually it changes "valley" or it stabilizes in areas where the GD step and the step on the regularizer cancel out.

\subsection{Empirical Observations}

\paragraph{The Hessian grows.} 
Another interesting phenomenon is the edge of stability, as known from \cite{jastrzebski_relation_2019,jastrzebski_break-even_2020,cohen_gradient_2021}.
Precisely, empirical work shows that along the trajectories of gradient descent and Adam the highest eigenvalue of the Hessian of the loss often steadily increases \cite{jastrzebski_break-even_2020,cohen_gradient_2021,cohen_adaptive_2022}.
This process, known as progressive sharpening, stops only when it reaches $2/\eta$ for GD and $38/\eta$ for Adam, as those are the instability thresholds for the algorithms. At this points usually the trajectories enter a different, oscillatory, regime called Edge of Stability.

\paragraph{SGD induces a smaller Hessian.}
However, \cite{cohen_gradient_2021} observed that this is not the case for SGD.
\cite{damian_self-stabilization_2023} claims that the reason why SGD with added label noise does not train at the edge of stability is that the implicit regularizer due to label noise, which penalizes the trace of a function of the Hessian \cite{damian_label_2021}, explodes close to the boundary $\lambda_{\max}  = 2/\eta$ working as a log-barrier. They also conjecture that something similar may happen to SGD alone.
Analogously, \cite{jastrzebski_break-even_2020} observed that in both regression and vision classification tasks, the trajectory of SGD aligns with the one of GD for a while until a \textit{breaking point} where it diverges from it and goes in areas of the parameter space where the trace of the Empirical Fisher Matrix is substantially lower. This breaking point is observed to arrive earlier for bigger effective learning rates $\eta / b$ and later for smaller ones.

\subsection{When SGD Deviates}

In what follows, we formalize this observation by finding that
the divergence after the \textit{breaking-point} observed by \cite{jastrzebski_break-even_2020} is due to the bias of SGD without replacement that we unveiled in \cref{theo:SGD_bias_eta}.

\paragraph{Breaking-point: when and why.}
At the beginning of training usually, the size of gradients decreases steadily while often the Hessian grows. This means that, soon after the beginning, the training often enters a regime in which we can apply \cref{theo:SGD_bias_eta}. In the case in which there are at least 2 eigenvalues $\lambda_1, \lambda_2 > \eta^{-1}$, the regularizer's step from \cref{theo:SGD_bias} along $\theta_1$ eigenvector of $\lambda_1$, see \S \ref{section:reg_new}, takes the following form up to an exponentially small term.
\begin{equation}
\begin{split}
\label{eq:EoS_step}
    \frac{\eta}{b-1} \
    \left[\frac{c}{2}\frac{1 + (-\eta \lambda_1)^{k}}{(c\lambda_1)(c\lambda_2)}
    \right] \ \frac{d}{d\theta_i}
    \Cov_{z \in D} (\nabla L (z) )[2,2] 
    \quad + \quad
    \substack{\text{Other terms of}\\\text{the regularizer.}}
\end{split}
\end{equation}
Or analogously, by exchanging the indexes, we obtain the update on $\theta_2$. 
This means that 

\begin{proposition}[Breaking point]
\label{prop:breaking_eos}
Assume we are in the hypothesis of \cref{theo:SGD_bias_eta}. Assume there exist $\lambda_1, \lambda_2$ eigenvalues of the Hessian $\nabla^2 L$ with eigenvectors $\theta_1, \theta_2$ such that $\lambda_2 > \eta^{-1}$, $u := \frac{d}{d\theta_1}\Cov_{z \in D}(\nabla L(z))[2,2] > 0$, and $\lambda_1 \geq \eta^{-1} + \alpha_{EoS}$ with
\[
\alpha_{EoS} \quad := \quad
\frac{1}{c} \left| \ln \left( \frac{c^2 \lambda_1 \lambda_2 \mednorm{\nabla L}}{u} \right) \right|
\quad > \quad
\frac{1}{c} \left| \ln \left( \frac{c^2 \mednorm{\nabla L}}{\eta^2 u} \right) \right|
.
\]
Then the additional step on the regularizer described by \cref{theo:SGD_bias_eta} is bigger than $k$ steps of GD.
\end{proposition}
The proof of this proposition is immediate after imposing that the quantity in \cref{eq:EoS_step} is bigger than $k$ steps of GD starting from $\theta$.

\paragraph{Dependence on $\eta/b$.}
Moreover, agreeing with the empirical observations by \cite{jastrzebski_break-even_2020}, we see that $\alpha_{EoS}(\eta,k)$ is a monotonic decreasing function of $\eta/b$, precisely it goes like its inverse.
\[
\alpha_{EoS} 
\quad := \quad
\frac{1}{c} \cdot \big( \text{positive}_\eta \big)
\quad = \quad 
\frac{b}{\eta} \cdot \frac{1}{n} \cdot \big( \text{positive}_\eta \big)
\quad = \quad 
\frac{b}{\eta} \cdot \big( \text{positive}_\eta \big).
\]

\paragraph{Phase transition.}
This phenomenon closely agrees with what was observed empirically, indeed this coincides with a phase transition, not with a slow continuous process, as empirically observed by \cite{jastrzebski_break-even_2020}. Indeed, with $\epsilon :=  \frac{|\ln(\eta)|}{c}> 0$ we have
\begin{equation}
\substack{\text{Size of the step}\\\text{on the regularizer}}
\quad = \quad
\begin{cases}
    \Theta(\eta \cdot \text{GD step}) \qquad &\text{if }
    \lambda_1 = \eta^{-1} + \alpha_{EoS} - \epsilon \\
    \Theta(\text{GD step}) \qquad &\text{if }
    \lambda_1 = \eta^{-1} + \alpha_{EoS}\\
    \Theta(\eta^{-1} \cdot \text{GD step}) \qquad &\text{if }
    \lambda_1 = \eta^{-1} + \alpha_{EoS} + \epsilon \\
\end{cases}
\end{equation}
Note that $\epsilon$ is usually $O(1)$ so much smaller than $\eta^{-1}$. For instance, see \S \ref{section:regime}, second line of the table, for cifar10 with $\eta = 0.01$ and $b=64$ we have that it is smaller than the gain in sharpness due to an epoch.
\[
\epsilon \quad = \quad \frac{|\ln(\eta)|}{c} \quad = \quad 0.6.
\]

\subsection{Limitations}
In some cases, however, the regularizer we unveiled may not explain why SGD does not train at the Edge of Stability. Based on the assumptions of Proposition \ref{prop:breaking_eos} this occurs when either: (i) \cref{theo:SGD_bias_eta} is not applicable, (ii) $u = 0$ in Proposition \ref{prop:breaking_eos} as in linear regression, (iii) there exists only one eigenvalue bigger than $\eta^{-1}$, (iv) in scenarios where $\alpha_{EoS} > \eta^{-1}$, and (v) when the regularizer directs the trajectory into areas of the parameter space that continue to experience progressive sharpening.
In these cases SGD may anyway not train at the Edge of Stability but because of the effect of the diffusion instead of the effect of the bias, see \S \ref{section:nature}.
If the optimizer has independent steps with a certain variance, indeed, we conjecture that we can apply a version of the results of \cite{damian_label_2021,li_what_2022} to figure out what the implicit regularization due to the higher-order terms of the Taylor is. 
However, this will be the focus of further studies.

\section{The Nature of the Effect}
\label{section:nature}

We list and discuss here the nature of the effect. We discuss in the following subsection the two ingredients of the effect: dependency, and discretization. Later we discuss the fact that the effect is a drift, it is not a diffusion-powered effect.

\subsection{Discretization}
\label{section:discretization}
The regularization term that we find \cref{theo:SGD_bias_eta} is the expectation of the quantity in Proposition \ref{theo:SGD_effect_fixed_batches}. Proposition \ref{theo:SGD_effect_fixed_batches} is the result of computing the effect of the \textit{discretization}, it is unrelated to the optimization problem or to the gradient flow itself.

\begin{center}
\emph{
Assume the batches $B_1, B_2, \ldots, B_k \subseteq D$ are fixed. How far are the trajectories of mini-batch and full-batch GD after $k$ steps of learning rate $\eta$?
}
\end{center}

We expand $k$ steps of SGD and GD with learning rate $\eta$ at the beginning of the epoch, and we keep the terms in which $\nabla L(\cdot)$ appears to the power 1. We write the difference between this quantity in the case of SGD and the case of GD and we obtain the following
\begin{proposition}
Set $B_1, B_2, \ldots, B_k$ fixed batches of the dataset $D$. The trajectory of mini-batch SGD deviates from the trajectory of the same number of steps $k$ of full-batch GD with the same learning rate $\eta$ of
\label{theo:SGD_effect_fixed_batches}
\begin{equation*}
\begin{split}
    &- \eta \sum_{i=1}^k \nabla L(B_i)
    \quad + \quad
    \eta^2
    \sum_{1 \leq i < j \leq k} 
    \nabla^2 L(B_j) 
    \Bigg[
    \prod_{h=i+1}^{j-1} [I-\eta \nabla^2 L(B_{h})]
    \Bigg]  
    \ \nabla L (B_i)
    \\[0.2cm]&\quad-\quad \substack{\text{same terms but full-batch:}\\ B_j \curvearrowleft D \text{ for all }j\leq k}
\end{split}
\end{equation*}
up to first order in $\eta \nabla L(B_i)$, $i = 1, 2, \ldots, k$.
\end{proposition}

\subsection{Dependency}
\label{section:dependency}
The deviation between trajectories is a sum of products of functions evaluated on different batches. Thus the expectation of the difference between the trajectories of a mini-batch SGD and GD is zero when the steps are:
\begin{itemize}
    \item \textbf{Centered:} every step is a random variable centered in the GD step. E.g., an unbiased estimator of it as for SGD with replacement or some centered Gaussian noise gets injected in the gradient.
    \item \textbf{Independent:} the steps are random variables independent of each other.
\end{itemize}
Indeed in that case, the expectations of those products are products of the expectations, in particular full-batch derivatives.
This is not the case of SGD without replacement. Indeed here the batches are dependent as they are disjoint. The regularizer we obtain taking the expectation originates from this dependence.
For this reason, \cref{theo:SGD_bias} applies when comparing SGD without replacement and SO and any algorithm with centered independent steps, such as GD, SGD with replacement, their label-noise-injected versions, etc. In particular, this regularization effect of SGD without replacement arises \textbf{solely} from the absence of the second assumption above: the independent steps assumption.
Similarly, our analysis can be reworked to unveil the effect of every algorithm that does not satisfy any of the two assumptions above.
Moreover, we conjecture that other analysis present in the literature that deal with centered independent noise can be applied on top of ours, to the de-biased trajectory. However, this will be the argument of future work.

\subsection{A Change of Drift}

When analyzing random dynamics we usually have two factors that contribute to the dynamic: the \textit{drift} and the \textit{diffusion} terms.
In particular, when the dynamics are discrete we can rewrite every step as
\[
X_{i+1} \ = \ X_i + \underbrace{\E[\text{step}]}_{\text{drift}} \ + \ \underbrace{\left( \text{step} - \E[\text{step}] \right)}_{\text{diffusion}}.
\]
These two parts impact the trajectory in different ways, with different speeds.
In particular, when we make use of tools such as concentration inequalities and CLT we are usually dealing with the effect of the \textit{diffusion} of the process. When we make use of LLN we generally deal with the effect of drifts.

The effect of SGD without replacement that we unveil is a drift-like effect, it is not an effect of the diffusion.
This is a story about the absence of stochasticity that we thought was there, not about the effect of the noise.
In particular, the effect we unveiled is not about a discrete version of the Fokker-Planck theory or another kind of effect of the diffusion part.
What we are observing in \cref{theo:SGD_bias} is that there exists a bias that points in a certain direction, moreover, we reduce the size of the diffusion of a related amount. 
This observation is in line with several existent works, e.g., \cite{smith_origin_2021,damian_label_2021,li_what_2022}. Indeed, they show that even in the case of diffusive independent noise at every step, as label nose injection or SGD with replacement, the implicit regularization effect may be due to drifts coming from the higher order terms in Taylor, not from diffusion-powered regularization.

\section{Comparison between SGDs}
\label{section:with_repl}

\subsection{The Question}
A subtle long-standing question is whether the fact that batches are disjoint and not sampled i.i.d.\ does matter in practice. In other words, if the induced dependency is strong enough for the trajectories of SGD without replacement and SGD with replacement or noised GD to get attracted to different minima. This is important to know because while SGD without replacement is much faster and widely used in practice, see \S \ref{section:training}, most of the theory developed is for GD or algorithms where the steps are independent, see \S \ref{section:challenges}. The question thus is
\begin{center}
    \emph{Do the trajectories of SGD without and with replacement or the minima they get attracted to qualitatively differ?}
\end{center}
\cref{theo:SGD_bias_eta} implies that the trajectories do significantly differ. SGD without replacement's regularization effect is, indeed, powered by the drift we found in \cref{theo:SGD_bias_eta}, while SGD with replacement's one by the diffusion term or a bias arising from different elements. 
As a consequence, SGD without replacement (i) escapes many loss saddles and travels flat areas much faster, as already discussed in \S \ref{section:saddles}, (ii) converges to certain minima, as already shown in \S \ref{section:Hessian}, (iii) and oscillates with a smaller variance than SGD with replacement.     
That said, the minima to which they converge can be the same, or as good, although they get there for different reasons, with different paths, and different speeds.

\begin{figure}[ht!]
\begin{mdframed}
{\centering \large \textbf{Different Behaviors of the  
SGDs: Escaping Local Minima}\\[0.3cm]}
    \begin{minipage}{.5\textwidth} 
        \centering
        \includegraphics[width=\linewidth,height=0.7\linewidth]{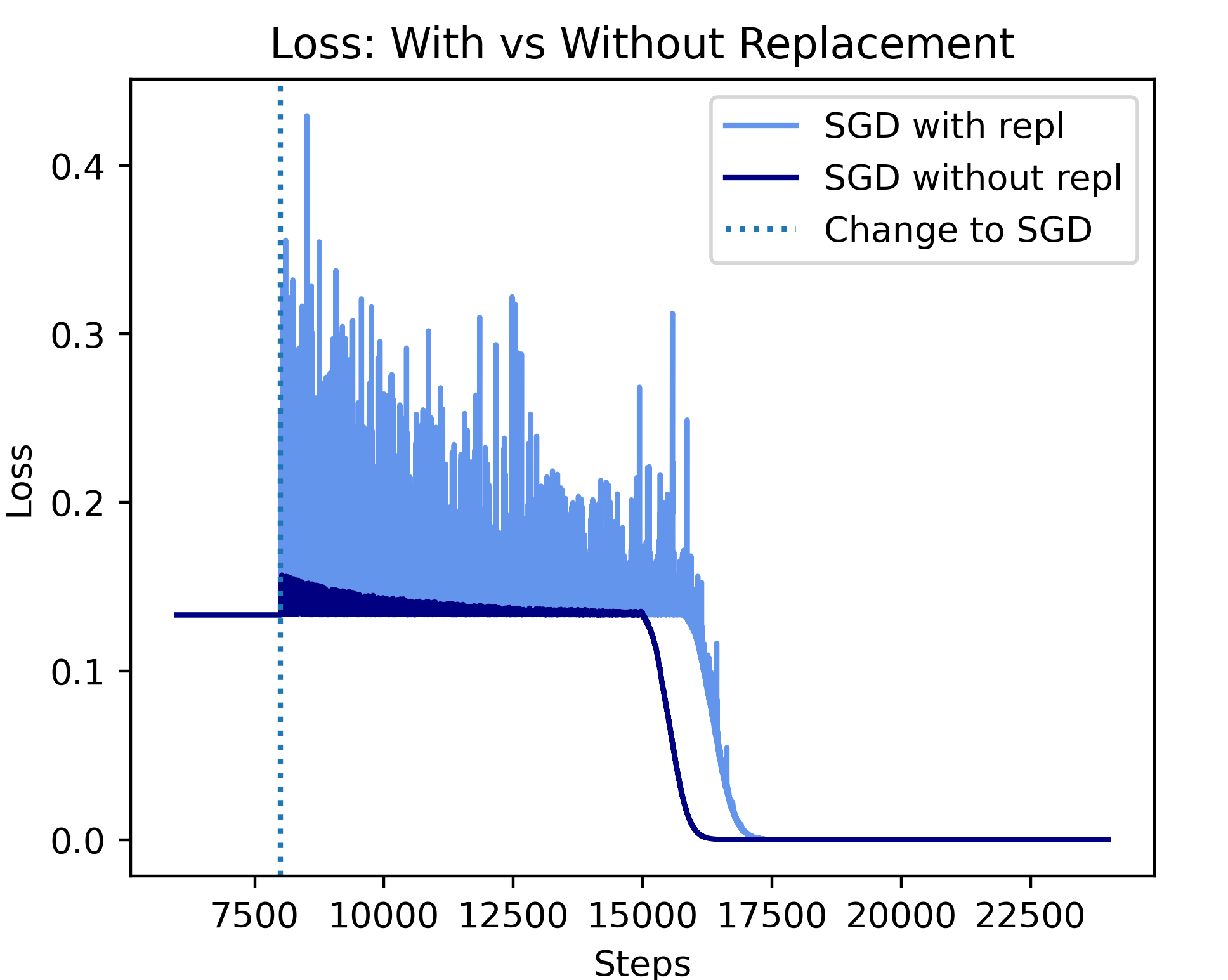}
    \end{minipage}
    \hfill
    \begin{minipage}{.5\textwidth} 
        \centering
        \includegraphics[width=\linewidth,height=0.7\linewidth]{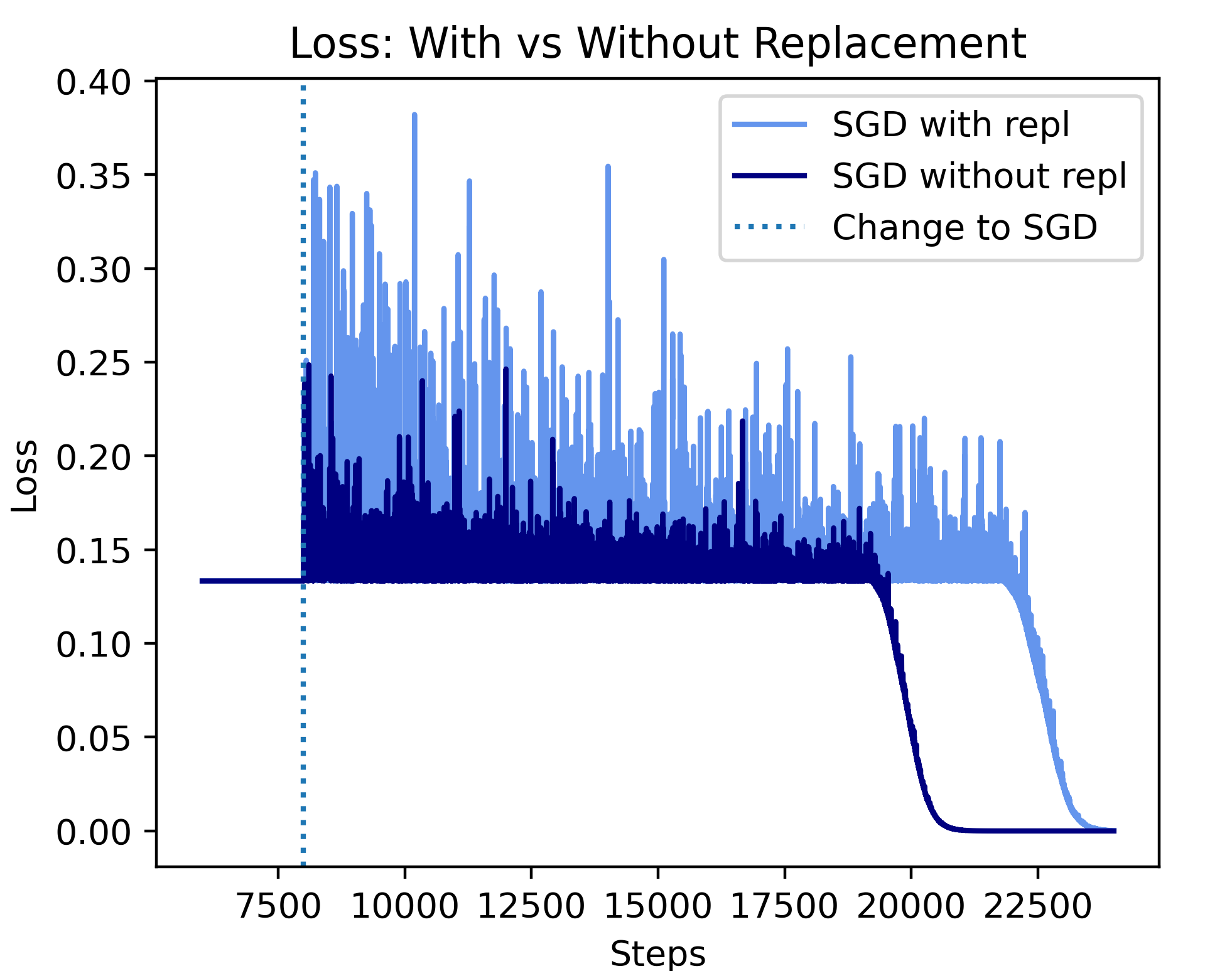}
    \end{minipage}
    \caption{
    SGD without replacement escapes local minima (by traveling flat areas) in which GD converged faster than SGD with replacement and with much smaller oscillations. They both converge to a global minimum. See \cref{fig:W_comparison}.
    \\
    This is the plot of the training loss (MSE) of a ReLU network on a synthetic dataset. We first run full batch GD until convergence, we check that it converged to a spurious local minimum, and we start from there both the SGDs with the same hyperparameters. Precisely, on the left with batch size 1 and on the right with batch size 2.
    }
    \label{fig:W_comparison}
\end{mdframed}
\end{figure}

\subsection{Shape and Intensity of the Noise}
From a certain perspective, SGD without replacement is less noisy than SGD with replacement. Performing an epoch of SGD without replacement consists of reshuffling the whole dataset once and then partitioning it. The randomness of SGD with replacement comes from the sampling of a new independent batch at every step. To get a sense for why SGD without replacement has a lower amount of noise, one can think that the last batch is deterministic, given the previous ones, unlike the case of SGD with replacement.
The way in which the smaller stochasticity shows up in the trajectory is clear when analyzing it. Let us expand in Taylor's series at initialization the $k$ steps of an epoch of SGD:
\[
\theta(k)^{SGD} \ - \ \theta
\quad = \quad - \ \underbrace{\eta
\sum_{t=0}^k \nabla L(\theta, B_t)}_{\substack{
\text{Noisy for SGD with r.}\\
\text{but deterministic for}\\ \text{SGD without r.}}} 
\ + \
\underbrace{\frac{\eta^2}{2}\sum_{s < t} \nabla^2 L(\theta, B_s)\nabla L(\theta, B_t)}_{\substack{
\text{Centered noise for SGD with r.}\\
\text{but Bias + Noise for SGD without}}} 
\ + \ \ldots.
\]
In the case of SGD with replacement the expectation of these terms is exactly the terms of the Taylor expansion of gradient descent.
Indeed, both the first and second terms are centered on the gradient descent counterparts $\eta k \nabla L + \frac{\eta^2 k(k-1)}{2} \nabla^2 L \nabla L$. However, both the terms above have their variance and induce oscillations in the trajectory.

The nature of the effect of SGD without replacement is very different. The first term deterministically sums up to a step of gradient descent of size $\eta k = c$. There is no noise coming from it. The trajectory of SGD without replacement is about the second term. In this case, $ \nabla^2 L(\theta, B_s)$ and $\nabla L(\theta, B_t)$ are not independent when $s \neq t$. This implies that (i) this second term is not centered, inducing a bias in the trajectory, and (ii) the size of the noise is related to the standard deviation only of this term.

More generally, all the terms considered in Proposition \ref{theo:SGD_effect_fixed_batches} bring a bigger and bigger bias in the case without replacement but are simply centered in the case with replacement.
The first biased term in the expansion of SGD with replacement is part of what goes into the error in our analysis, as previously highlighted in \cite[Proposition 5]{damian_label_2021}.
This means that the effect of both the algorithms comes from the effect of discreitization, see \S \ref{section:nature}. However, the interplay between the noise in the process of sampling batches and this discretization effect is very different for the two algorithms.

The fact that a bias is present and that the noise comes from the second term only is visible in experiments. For instance in \cref{fig:intro_2} and \cref{fig:W_comparison}, where we trained a Shallow ReLU network on synthetic data. Here the oscillations of SGD without replacement at a spurious local minimum are consistently smaller. Nonetheless, SGD without replacement travels the manifold of minima faster, due to the bias, escaping it quicker.
The same can be noticed also in \cref{fig:intro} where SGD converges faster and with noticeably smaller oscillations that SGD with replacement.

\section{Conclusions, Broad Applicability, and Limitations}
\label{section:comparison}
We discuss how our work compares to other studies, to what extent the following question has been answered, and its the limitations, before concluding.

\begin{center}
\emph{Can we characterize the minima to which SGD converges? How?
}
\end{center}

We argue that from a certain point of view, we employ the minimal and most natural set of assumptions for such an analysis.
We discuss the broad applicability of this approach in \S \ref{section:generality}. 
We discuss the limitations of our findings in \S \ref{section:shortcomings}.

\subsection{Generality and Optimality of our Approach}
\label{section:generality}

As previously discussed in \S \ref{section:challenges}, there are several reasons for conducting a local analysis. 
This is primarily due to the inapplicability of "global" methods such as stochastic approximation by Robbins and Monro. Also, there is a lack of understanding of the geometry of the manifolds. 
Essentially, conducting a local analysis means expanding in the Taylor series or a similar approximation technique. 
In this context, as explained in \S \ref{section:regime}, a Taylor expansion over the $k$ steps of an epoch can be represented as a sum of product of terms of the form:
\[
\text{coefficient} \ \cdot \ \eta \ \cdot \ k \ \cdot \ \text{function}\big( \nabla^i L(\cdot, \cdot) \big)
\]
for some $i \in \N$, some real coefficient, and a function that could be, e.g., an average over batches, and so on.
For such Taylor expansions to converge and represent our trajectory, it is thus \textit{necessary} that some of the terms above are small.
As argued in \S \ref{section:regime}, $\eta k \nabla L$ appears in every term and it is the only such term that vanishes throughout the training. 
This happens because $\eta$ is annealed and $\nabla L$ converges to 0 approaching a stationary point.
Moreover, no matter the size of the other derivatives, if $\eta k \nabla L \ll 1$ the series generally converges.
Consequently, for the Taylor series to converge, or analogously for a local analysis to be performed, requiring the assumption $\eta \nabla L \ll 1$ is the most natural of the possible \textit{necessary} conditions.
An analysis based solely on this assumption is thus highly versatile, arguably the broadest in scope among local analyses.

\begin{mdframed}
    Whenever one can use a local argument, one can expand in Taylor, in particular, generally $\eta^{>1} \cdot k \cdot \nabla L \ll 1$. Thus our analysis characterizes the dynamics.
\end{mdframed}

Our approach differs from past community efforts that relied on stricter assumptions $\eta k \ll 1$ \cite{smith_origin_2021,roberts_sgd_2021}, or conditions involving other higher order derivatives. The only works that could relax these assumptions usually assume independence between single steps, e.g.\ \cite{damian_label_2021}. We demonstrated that the inherent effect of SGD without replacement is due to the dependence. This means that that such results, while being still instrumental in addressing the effect of other algorithms, do not apply to our setting.

\subsection{Limitations: Broader Scope implies Weaker Results}
\label{section:shortcomings}
In the context of small learning rate \cite{smith_origin_2021,roberts_sgd_2021} or independent steps \cite{damian_label_2021,li_what_2022} the results are mathematically elegant: we have a penalized loss on which GF or GD mirrors the trajectory of the SGD. Those works deliver insightful heuristics and theorems about where the trajectory leads and the generalization improvements.
In our context, however, as discussed in \S \ref{section:motivations}, with $\eta \not \! \ll 1$ and dependent steps, we cannot explain convergence using those results and techniques.
This limitation hampers our ability to establish the beautiful convergence results that mark the earlier studies. We will investigate this in future works, however, the absence of such strong convergence results would not be entirely surprising. It would align with the understanding that in many contexts there is not such a thing as a potential that the algorithms are minimizing. This has been shown for instance for gradient flow for ReLU networks by \cite{vardi_implicit_2021}. It has also been discovered by \cite{damian_label_2021} and later made rigorous by \cite{li_what_2022} for the case of SGD, in which it is possible that SGD cycles.
We can thus only hope to better characterize the trajectories. This is exactly what we manage to do with our technique.

Our \cref{theo:SGD_bias_eta} characterizes the path-dependent regularizer of the algorithm in analysis.
In general, we were unable to describe our results, unlike previous work, as 
"SGD without replacement works as a specific (different) optimization algorithm on a modified loss".
This, however, may not be a shortcoming of our analysis. Rather, it could underscore the nature of the effect, which proves more challenging to describe mathematically than many earlier hypotheses.
That said, heuristically the meaning of \cref{theo:SGD_bias_eta} is similar to working with a penalized loss: If the regularizer is not minimized it guides the trajectory towards areas where it decreases.
Returning to the main question of this section, we could not precisely define the minima to which SGD converges, in contrast to previous studies that have achieved such a goal in different contexts. However, the task is likely unattainable in full generality.
Identifying such specific cases where this is possible is beyond the scope of our current work.

\subsection{Conclusions}

We showed that SGD without replacement implicitly regularizes by biasing the dynamics towards areas with lower variance. This is due to the dependence of the steps and manifests as a drift-like effect, not attributable to diffusion or a Fokker-Planck-like argument. We demonstrate that this leads to a form of regularization, enabling the algorithm to navigate through flat areas more quickly and with fewer oscillations than expected.




\acks{I want to thank Prof.\ Boris Hanin for his crucial support and help.
A special thanks to Alex Damian, Prof.\ Jason D.\ Lee, Samy Jelassi, and Ahmed Khaled for invaluable discussions that were key to my understanding of the topic and the development of this project.
I thank Prof.\ Bartolomeo Stellato, Prof.\ Jianqing Fan, Giulia Crippa, Arseniy Andreyev, Hezekiah Grayer II, Ivan Di Liberti, Valeria Ambrosio, Alexander Krauel, Tommaso Portaluri, Simone Brutti, and Camilla Beneventano for their valuable advice and comments on various parts of this manuscript.
Special thanks to Prof.\ Misha Belkin, as the inspiration for this project was sparked during one of his mini-courses and our conversations.
I also warmly thank the participants and organizers of the "Statistical Physics and Machine Learning Back Together Again" workshop for the meaningful dialogues that have profoundly shaped this paper.}

\appendix









\vskip 0.2in
\bibliography{My_Library}

\clearpage

\addtocontents{toc}{\protect\setcounter{tocdepth}{-1}}

\section{The Effect of Mini-batching}
\label{section:mini-batching}
The goal of this section is to understand the role of the following two properties of SGD without replacement:
\begin{itemize}
    \item The fact that the mini-batches add up to the dataset.
    \item The fact that once the mini batches are sampled we do not care about their order.
\end{itemize}
The first property above is shared by the SGD without replacement and Shuffle Once, and the second by SGD with and without replacement. However, once we take the expectation over the possible outcome of the initial shuffling, the same holds for Shuffle Once.
We argue that it is the fact that we do not care about the order that allows us to rewrite our quantity as a derivative of a penalization. This properties are not shared, e.g., by Adam.

\subsection{The Implicit Bias of the Order}
\label{section:order}

We derived a formula for the deviation in the trajectory in Proposition \ref{theo:SGD_effect_fixed_batches} in \S \ref{section:discretization}, which depends on the order in which we observe the batches. 
We examine here the expectation and standard deviation of this SGD discretization with respect to the GD trajectory due to the randomly uniform order of the batches.
\begin{center}
\emph{
Assume we sampled already the batches $B_1, B_2, \ldots, B_k$ (fixed). How far do we go from the GD/GF trajectory, \textit{on average over the possible orders of the batches?}
}
\end{center}
To do it, we consider the following quantity
\begin{equation}
\label{eq:potential}
    \frac{\eta^2}{2}
    \nabla_\theta 
    \sum_{1 \leq i < j \leq k} 
    [\nabla^\red{1} L(B_j) ]^\top
    \Bigg[
    \prod_{h=i+1}^{j-1} [I-\eta \nabla^2 L(B_{h})]
    \Bigg]  
    \ \nabla L (B_i)
\end{equation}
The derivative of this quantity in $\theta$ is
\begin{equation}
\begin{split}
\label{eq:deriv_1}
    & \quad
    \frac{\eta^2}{2}
    \sum_{1 \leq i < j \leq k} 
    \nabla^2 L(B_j) 
    \ \Bigg[
    \prod_{h=i+1}^{j-1} [I-\eta \nabla^2 L(B_{h})]
    \Bigg]  
    \ \nabla L (B_i)
    \qquad \quad \ \ \! \leftarrow \text{Proposition 
    \ref{theo:SGD_effect_fixed_batches}!}
    \\&+
    \frac{\eta^2}{2}
    \sum_{1 \leq i < j \leq k} 
    \nabla^2 L(B_i) 
    \ \Bigg[
    \prod_{h=i+1}^{j-1} [I-\eta \nabla^2 L(B_{h})]
    \Bigg] ^\top
    \ \nabla L (B_j)
    \qquad \ \ \! \leftarrow \text{flipped }i,j
    \\&+
    \frac{\eta^2}{2}
    \sum_{1 \leq i < j \leq k} 
    \nabla L(B_j) ^\top \
    \Bigg[ \nabla_\theta 
    \prod_{h=i+1}^{j-1} [I-\eta \nabla^2 L(B_{h})]
    \Bigg] ^\top
    \ \nabla L (B_i).
    \ \ \leftarrow \text{error term!}
\end{split}
\end{equation}
In the first line we have half of the term in Proposition \ref{theo:SGD_effect_fixed_batches}, and in the second line we have the same quantity but with the indices switched. The third line is a part of the error as constitute a term of second order in $\eta \nabla L$.
This is the key idea of the proof of the following theorem other than of the main theorem, \cref{theo:SGD_bias_eta}.

\begin{proposition}
\label{theo:SGD_effect_exp_fixed_batches}
With an additive error of size smaller than $O ( \text{that} ) \cdot (\eta k\nabla L + 1/k)$, we have
\[
\theta_k^{SGD}
\quad \overset{\substack{\text{expectation over}\\ \text{order of batches}}}{=} \quad 
\theta_k^{GD} 
\quad - \quad
\eta \nabla_\theta Regularizer
\]
where where the regularizer is 
\begin{equation*}
\begin{split}
    \eta
    \E_{order} \left[
    \sum_{1 \leq i < j \leq k} 
    \nabla L(B_j) ^\top
    \Bigg[
    \prod_{h=i+1}^{j-1} [I-\eta \nabla^2 L(B_{h})]
    \Bigg]  
    \ \nabla L (B_i)
    \right]
    \quad-\quad \substack{\text{same but full-batch:}\\ B_j \curvearrowleft D \text{ for all }j\leq epoch}.
\end{split}
\end{equation*}

\end{proposition}

We can thus conclude that:
(i) The reason why we can rewrite this deviation as the derivative of a potential is that we do not care about the order in which we see the batches. And (ii) this happens for every SGD technique once we sample the batches, irrespective of how we sampled them.

\subsection{Observations on the Regularizer}

\paragraph{Linear scaling rule.}
A first important observation is that the step on the regularizer is exactly of size $\eta/b$. This aligns with what has been widely empirically observed. Starting with \cite{goyal_accurate_2017,jastrzebski_three_2018,he_control_2019}, indeed, and continuing with, e.g., \cite{jastrzebski_catastrophic_2021}, the community noticed that any regularization effect attributable to SGD empirically scales with that quantity $\eta / b$.

\paragraph{The effective step size $c = \eta k$.}
Every time a covariance of two things appears in the formula for the regularizer the multiplying constant is $\frac{c^2}{2n}$.
Indeed, there are $k$ terms of size $\eta$, so there are $\eta^2\binom{k}{2}$ covariances, each multiplied by $-1/n$ coming from Corollary \ref{cor:chain_cond_exp_var}.
E.g., to the zeroth order in $c\lambda$ we have
\[
\frac{1}{n} \E \left[\tfrac{1}{2}c\nabla L(\cdot) ^\top c\nabla L(.)\right]
\quad = \quad
\frac{c^2}{2n} \E [\mednorm{\nabla L(\cdot)}^2]
\quad = \quad
\frac{\eta}{b-1} \cdot \frac{c}{2} \cdot \E [\mednorm{\nabla L(\cdot)}^2].
\]
This appears also when dealing with the Hessian terms, indeed $H$ is defined as $\frac{c^2}{2n}$ times $\E[\nabla^2 L(\cdot)^2]-(\nabla^2 L )^2$, see \cref{prop:S}.

\paragraph{Sanity checks.}
The $\frac{c^2}{2}$ that multiplies everything in \cref{prop:S} is in reality 
\[
\frac{c^2}{2} \quad = \quad \frac{\eta^2 k(k-1)}{2}.
\]
So it cancels out if we are doing full-batch, as $k=1$.

\paragraph{Dependency $\implies$ Covariance.}
The SGD trajectory is different from GD's only when the Hessians and the gradient for different batches are significantly different. Accordingly, the elements steering the difference between the two trajectories come up in the regularizer as:
\[
\E[\big(\eta\nabla^2L(\cdot)\big)^2]-(\eta\nabla^2L)^2    \qquad and \qquad
\nabla_\theta \var( \mednorm{\eta\nabla_\theta L(\cdot)} )
\]
This is natural since the dependence implies covariance terms to appear in the expectation of the product.

\section{Dealing with Expectations}
\label{section:exp_app}

To prove our main result, we need to be able to compute all the expectations that we encounter. Those are taken over the SGD sampling procedure, i.e., sampling without replacement:

\begin{setting}[Sampling without replacement]
\label{setting:without_replacement}
Assume $z_1, z_2, \ldots, z_k$ are sampled in one of the following two \textbf{equivalent} ways. 
\begin{enumerate}
    \item \label{item:sampling_sets} $\{z_1, z_2, \ldots, z_k\}$ is a set of $k$ different elements that has been uniformly sampled from the family of all the possible different subsets (without multiple copies) of cardinality $k$ of a set $D$ with cardinality $n$.
    \item \label{item:sampling_without_rep} We sampled $z_1$ uniformly from $D$, $z_2$ uniformly from $D \backslash \{z_1\}$, $\ldots$, and $z_k$ uniformly from $D \backslash \{z_1, z_2, \ldots, z_{k-1}\}$.
\end{enumerate}
All the expectations in what follows will be over this sampling procedure. Indeed, if we see a batch or a data point we will not see it until the end of the epoch.

We will assume we have a sequence of functions that take in input elements of $D$ and output a tensor such that the shapes of the tensors match and we can make a product. Precisely, for all $i \in \{1, 2, \ldots, k\}$ we have $f_i \colon D \to \R^{n_i \times n_{i+1}}$ where $n_1, n_2, \ldots, n_{k+1} \in \N$. 
Precisely, in the following these $f$s will take a training data point $z \in D$ in input and they will output a derivative of the loss in the parameters evaluated at that data point, e.g. $f_1(z_1) = I - \nabla_\theta^2 L (\theta, z_1)$.
\end{setting}

\subsection{The Main Lemma}
\begin{lemma}[How the expectation looks like]
\label{lemma:chain_cond_exp}
Assume we are in Setting \ref{setting:without_replacement} above.
Then, when $S$ is sampled as in \cref{item:sampling_sets} of Setting \ref{setting:without_replacement}, the expectation $\E\left[ f_1(z_1) \cdot f_2(z_2) \cdot \ldots \cdot f_k(z_k) \right]$, is equal to
\begin{equation}
\label{eq:average_sampling_sets}
\begin{split}
    \left[\prod_{i=0}^{k-1} (n-i)\right]^{-1} \sum \prod_{i=1}^k f_i(z_i)
\end{split}
\end{equation}
where, thus, the sum is taken over all the possible $k$-uple $(z_1, z_2, \ldots, z_k)$ of $k$ different elements of $D$. By averaging over the sampling technique in \cref{item:sampling_without_rep}, we obtain quantities like the following in the case of commutation of the $f_i$s.
\begin{equation}
\label{eq:average_sampling_without_rep}
\begin{split}
    \left[\prod_{i=0}^{k-1} \frac{n}{n-i}\right]
    \sum_{j=1}^k  \sum_{\sigma \in \Pi}
    \prod_{ 
    \substack{\text{all the possible} \\
    0=i_0 < i_1 < \ldots <i_j = k}} 
    \!\!\!\!\!
    \left[\frac{-1}{n}\right]^{i_j-i_{j-1}-1}\E_{z \in D}\left[ \prod_{h = i_{j-1}+1}^{i_j} f_{\sigma(h)}(z) \right].
\end{split}
\end{equation}
where $\Pi$ is the set of all the possible permutations of $\{1,2, \ldots, k\}$.
That is the sum of all the possible products of expectations of products of functions. Each of the expectations of products of $l$ functions is multiplied by $\left[\frac{-1}{n}\right]^{l-1}$.
Another way to rewrite it is $\left[\prod_{i=0}^k \frac{n}{n-i}\right]$ multiplied by
\begin{equation}
\label{eq:cond_exp_usable}
\begin{split}
    &
    \prod_{i = 1}^k \E[f_i] - n^{-1} \sum_{i \neq j} \E[f_i f_j] \prod_{m\neq j,i} \E[f_m] 
    +
    n^{-2} \sum_{i \neq j \neq h \neq l} \E[f_i f_j]\E[f_h f_l] \prod_{m\neq j,i,h,l} \E[f_m] 
    + \ldots
\end{split}
\end{equation}
\end{lemma}

\begin{proof}[Proof of Lemma \ref{lemma:chain_cond_exp}.]
Let us now prove what above, regarding the first part in \cref{eq:average_sampling_sets}, it follows directly from the definition of average.
We prove the other thing by \textbf{induction}. Note that for $k=1,2$ we have the result in \cref{eq:average_sampling_without_rep} already (base case).
Indeed, note that
\begin{equation}
    \E_{D \backslash \{z\}} [f]
    \quad = \quad 
    \frac{1}{n-1} \sum_{x \in D \backslash \{z\}} f(x)
    \quad = \quad 
    \frac{n}{n-1} \left[ \E_D [f] - \frac{f(z)}{n}\right]
\end{equation}
and by plugging this in
\begin{equation}
\begin{split}
    \E_{z_1 \in D}[f_1(z_1) \cdot \E_{z_2 \in D \backslash \{z_1\}} [f_2(z_2)]]
    \quad &= \quad 
    \frac{n}{n-1}\E_D [f_2] \E[f_1]  - \frac{1}{n-1} \E_{D} \left[f_1 f_2\right].
\end{split}
\end{equation}
Let us see the step $k=3$ to better explain what is happening
\begin{equation}
\begin{split}
    &\E_{z_1 \in D}[f_1(z_1) \cdot \E_{z_2 \neq z_3 \in D \backslash \{z_1\}} [f_2(z_2)f_3(z_3)]]
    \\&\quad = \quad
    \E_{z_1 \in D} \left[ f_1(z_1) \cdot \left( \frac{n-1}{n-2} \E_{D\backslash \{z_1\}} [f_2] \E_{D\backslash \{z_1\}}[f_3]  - \frac{1}{n-2} \E_{D\backslash \{z_1\}} \left[f_2 f_3\right]\right)\right]
    \\&\quad = \quad
    \E_{z_1 \in D} \Bigg[ f_1(z_1) \cdot \bigg( \frac{n-1}{n-2} \frac{n}{n-1} \left[ \E_D [f_2] - \frac{f_2(z_1)}{n}\right] \frac{n}{n-1} \left[ \E_D [f_3] - \frac{f_3(z_1)}{n}\right]   
    \\&\qquad
    - \frac{1}{n-2} \left[ \frac{n}{n-1}\E_{D} \left[f_2 f_3\right] - \frac{f_2(z_1)f_3(z_1)}{n-1} \right] \bigg) \Bigg]
    .
\end{split}
\end{equation}
Next by noting that
\begin{equation}
\begin{split}
    (n-1) \ \frac{f_2(z_1)}{n-1}\frac{f_3(z_1)}{n-1}
    \quad = \quad 
    \frac{f_2(z_1)f_3(z_1)}{n-1}
\end{split}
\end{equation}
we conclude that
\begin{equation}
\begin{split}
    &\E_{z_1 \neq z_2 \neq z_3 \in D}[f_1(z_1) f_2(z_2) f_3(z_3)]
    \quad= \quad
    \frac{n^3}{n(n-1)(n-2)}\E[f_1] \E[f_2] \E[f_3]  
    \\&\quad
    - \frac{n^2}{n(n-1)(n-2)} \left(\E\left[f_1 f_2\right]\E[f_3] + 
    \E[f_1]\E\left[f_2 f_3\right] + 
    \E\left[f_1 \cdot \E[f_2] \cdot f_3 \right] \right)
    \\&\quad
    + 2
    \frac{n}{n(n-1)(n-2)}
    \E[f_1f_2f_3]
    .
\end{split}
\end{equation}
More generally we will have all the possible products showing up with the multiplier $(-n)^{-order+1}$.
Indeed or what regards the \textbf{inductive step}, note that by tower rule
\begin{equation}
\begin{split}
\label{proof:chain_1}
    \E_D
    \left[   \, f_1(z_1) \cdot f_2(z_2) \cdot \ldots \cdot f_k(z_k) \right]
    \quad = \quad
    \E_{z_1 \in D}\left[ f_1(z_1) \ \cdot \
    \E_{D\backslash\{z_1\}} \left[ f_2(z_2) \cdot \ldots \cdot f_k(z_k) \right]\right]
\end{split}
\end{equation}
So the expectation for $z_1, z_2,\ldots, z_k \in D$ without replacement it is the same as the expectation over $z_1 \in D$ of the expectation for $z_2, z_3,\ldots, z_k \in D\backslash\{z_1\}$ without replacement.
We will also prove that 
Now, applying the inductive hypothesis for $k \curvearrowleft k-1$ and $D \curvearrowleft D \backslash \{z\}$ on $\E_{D\backslash\{z_1\}} \left[ f_2(z_2) \cdot \ldots \cdot f_k(z_k) \right]$ we obtain that
\begin{itemize}
    \item The multiplier is $\prod_{i=0}^{k-1} \frac{n}{n-i}$.
    \item That we have $n^k \prod_{i=1}^k \E[f_i] - n^{k-1} \sum_{i<j} \E[f_if_j] \prod_{h\neq i,j} \E[f_h] + \ldots$ and in general appears every possible combination of expectations of products of the functions, such that when we are taking the expectations of $l$ functions together, we have $(-1/n)^{l-1}$ multiplying the head coefficient.
\end{itemize}

\begin{remark}
From now on we perform the computations in the case of commutations for readability. They do not change at all in case of no commutation.
\end{remark}

Let us assume it works for $D \backslash\{z_1\}$ and $k-1$. Then the expectation in that setting is the following and satisfies the points above.
\begin{equation}
\begin{split}
\label{proof:chain_2}
    \left[\prod_{i=1}^k \frac{n-1}{n-i}\right]
    \left[
    \prod_{i=2}^k \E_{D\backslash\{z_1\}}[f_i] - (n-1)^{-1} \sum_{i \neq j} \E_{D\backslash\{z_1\}}[f_i f_j] \prod_{m \neq i,j} \E_{D\backslash\{z_1\}}[f_m] + \ \ldots \
    \right]
\end{split}
\end{equation}
Thus plugging this into \cref{proof:chain_1} we obtain that $\E_D\left[   \, f_1(z_1) \cdot f_2(z_2) \cdot \ldots \cdot f_k(z_k) \right]$ is equal to
\begin{equation}
\begin{split}
    \E_D\Bigg[ f_1(z_1) \cdot 
    \Bigg[
    (n-1)^{k-1} \prod_{i=2}^k \E_{D\backslash\{z_1\}}[f_i] 
    - 
    (n-1)^{k-2} \sum_{i \neq j} \E_{D\backslash\{z_1\}}[f_i f_j] \prod_{m \neq i,j} \E_{D\backslash\{z_1\}}[f_m]
    \Bigg]
    \Bigg]
\end{split}
\end{equation}
multiplied by $\left[\prod_{i=1}^k \frac{1}{n-i}\right]$.
Next, note that
\begin{equation}
    \E_{D \backslash \{z\}} [f]
    \quad = \quad 
    \frac{1}{n-1} \sum_{x \in D \backslash \{z\}} f(x)
    \quad = \quad 
    \frac{n}{n-1} \left[ \E_D [f] - \frac{f(z)}{n}\right].
\end{equation}
Applying this above, we obtain
\begin{equation}
\begin{split}
\label{proof:chain_3}
    &\left[\prod_{i=1}^k \frac{1}{n-i}\right]
    \E_D\Bigg[ f_1(z_1) \cdot 
    \Bigg[
    (n-1)^{k-1} \prod_{i=2}^{k} \frac{n}{n-1}\Bigg[\E[f_i]-\frac{f_i(z_1)}{n}\Bigg]
    \\& \quad - 
    (n-1)^{k-2} \sum_{2 \leq i \leq j \leq k} 
    \frac{n}{n-1}\left[
    \E[f_i f_j] - \frac{f_i(z_1)f_j(z_1)}{n}
    \right]
    \prod_{n \neq i,j,1} \frac{n}{n-1}\left[\E[f_i]-\frac{f_i(z_1)}{n}\right]
    \Bigg]\Bigg]
    \\&=
    \left[\prod_{i=0}^k \frac{1}{n-i}\right]
    \E_D\Bigg[ f_1(z_1) \cdot 
    \Bigg[
    n^{k-1} \prod_{i=2}^{k} \Bigg[\E[f_i]-\frac{f_i(z_1)}{n}\Bigg]
    \\& \qquad \qquad \quad \qquad \qquad - n^{k-2} \sum_{2 \leq i \leq j \leq k} 
    \left[
    \E[f_i f_j] - \frac{f_i(z_1)f_j(z_1)}{n}
    \right]
    \prod_{h \neq i,j,1} \left[\E[f_h]-\frac{f_h(z_1)}{n}\right]
    \Bigg]\Bigg]
\end{split}
\end{equation}
where the average is now taken over $D$.
Note that for all $i,j$ it holds
\begin{equation}
\begin{split}
    n^{k-1} \ \frac{f_i(z_1)}{n}\frac{f_j(z_1)}{n} \prod_{h \neq i,j, 1}^k \E[f_h]
    \quad = \quad 
    n^{k-2} \ \frac{f_i(z_1)f_j(z_1)}{n}
    \prod_{h \neq i,j, 1}^{k} \E[f_h]
\end{split}
\end{equation}
and the same holds for all the other expectations of products of more functions.
In general, the sign and power of $n$ multiplying all the expectations of products of $l$ terms is $n^{-l+1}$.
Precisely, we obtain
\begin{equation}
\begin{split}
    &\E\left[ f_1(z_1) \cdot f_2(z_2) \cdot \ldots \cdot f_k(z_k) \right]
    \\& \qquad =
    \left[\prod_{i=0}^k \frac{1}{n-i}\right]
    \E_D\Bigg[ f_1(z_1) \cdot 
    n^{k}\Bigg[ \prod_{j > 1} \E[f_j] - \sum_{i=2}^k \frac{f_i(z_1)}{n} \prod_{j\neq 1,i} \E[f_j]
    \\&\qquad \qquad \qquad - 
    n^{k-1} \sum_{2 \leq i < j} \Bigg[
    \E[f_i f_j]
    \prod_{n \neq i,j,1} \left[\E[f_i]-\frac{f_i(z_1)}{n}\right]\Bigg] \ + \ \ldots
    \Bigg]
\end{split}
\end{equation}
That is equal to the usual $\left[\prod_{i=0}^k \frac{1}{n-i}\right]$ multiplied by
\begin{equation}
\begin{split}
    &
    n^{k} \prod_{i = 1}^k \E[f_i] - n^{k-1} \sum_{i \neq j} \E[f_i f_j] \prod_{m\neq j,i} \E[f_m] 
    +
    n^{k-2} \sum_{i \neq j \neq h \neq l} \E[f_i f_j]\E[f_h f_l] \prod_{m\neq j,i,h,l} \E[f_m] 
    + \ldots
\end{split}
\end{equation}
This concludes the proof Lemma \ref{lemma:chain_cond_exp}.
\end{proof}

\begin{remark}
\label{remark:size_of_rest}
When we are taking $z_i=z_j$ for all the possible $\binom{k}{2}$ couples (i,j), and multiplying by all picked alone. Then $\tfrac{1}{2} \cdot \binom{k}{2} \cdot \binom{k-2}{2}$ for the case of two couples and all the rest alone, etc.
When we are taking $m$ terms together, we can choose them in $\binom{k}{m}$ so the size of the products in which those appear is much smaller precisely is about the size of the others but we divide by $n^{m-1}$ instead of $m$ and those are not $k(k-1)\cdot \ldots \cdot (k-2m+1)/2^m$ but at the denominator there is $m!$ so they are much less and much smaller. We will thus throw them in the error later.
\end{remark}

\subsection{Corollary for big $n$ and $k$}

Our setting is special: we can cancel out some terms because $n,k \gg 1$. 
Assume then that we have 3 or more functions $f_1, f_2, f_3, \ldots$.
The terms that multiply the expectation of the product of 3 or more of them, e.g.\ $\E_{z \in D} f(z)g(z)h(z)$ in the big sum Lemma \ref{lemma:chain_cond_exp} can be thrown in the error part, indeed
\begin{equation}
\begin{split}
&n^{-2}\E_{z \in D} \big[ f_1(z)f_2(z)f_3(z) \big] 
\prod_{i=4}^k \E_{z \in D} \big[ f_i(z) \big]
\\ & \quad = \quad 
O(n^{-1}) \ \cdot \
n^{-1}\E_{z \in D} \big[ f_1(z)f_2(z) \big] 
\prod_{i=3}^k \E_{z \in D} \big[ f_i(z) \big]
\end{split}
\end{equation}
if $\E_{z \in D} \big[ f_1(z)f_2(z)f_3(z) \big] = O (\E_{z \in D} \big[ f_1(z)f_2(z)\big] \E_{z \in D} \big[ f_3(z) \big] )$.
Or analogously 
\begin{equation}
\begin{split}
&\binom{k}{3} n^{-2}\E_{z \in D} \big[ f_1(z)f_2(z)f_3(z) \big] 
\prod_{i=4}^k \E_{z \in D} \big[ f_i(z) \big]
\\& \quad = \quad 
O(k^{-1}) \ \cdot \
\binom{k}{4} n^{-2}\E_{z \in D} \big[ f_1(z)f_2(z) \big] \E_{z \in D} \big[ f_3(z)f_4(z) \big] 
\prod_{i=5}^k \E_{z \in D} \big[ f_i(z) \big]
\end{split}
\end{equation}
if $\E_{z \in D} \big[ f_1(z)f_2(z)f_3(z) \big] = O (\E_{z \in D} \big[ f_1(z)f_2(z) \big] \E_{z \in D} \big[ f_3(z)f_4(z) \big]  )$.

We can thus conclude with a corollary of Lemma \ref{lemma:chain_cond_exp} by using this and induction that


\subsection{Corollary about Covariances}
Secondly, note that we can rewrite this in terms of covariances, not of expectations of products, indeed, for every term we have the analogous
\begin{equation}
\begin{split}
&
n^{-1}\E_{z \in D} \big[ f_1(z)f_2(z) \big] 
\prod_{i=3}^k \E_{z \in D} \big[ f_i(z) \big] 
\ - \ n^{-1} \prod_{i=1}^k \E_{z \in D} \big[ f_i(z) \big] 
\\& \quad = \quad
n^{-1}\Cov_{z \in D} \big( f_1(z), f_2(z) \big) 
\prod_{i=3}^k \E_{z \in D} \big[ f_i(z) \big] 
\end{split}
\end{equation}
and analogously for a higher number of expectations of products of terms.
This and induction prove that
\begin{corollary}[Corollary of Lemma \ref{lemma:chain_cond_exp}.]
\label{cor:chain_cond_exp_var}
Assume the different averages over the dataset are of the same order, and $n,k \gg 1$. Then
Assume we are in Setting \ref{setting:without_replacement} above.
Then, when $S$ is sampled as in \cref{item:sampling_sets} of Setting \ref{setting:without_replacement}, the expectation $\E\left[ f_1(z_1) \cdot f_2(z_2) \cdot \ldots \cdot f_k(z_k) \right]$, up to a multiplicative error of size $O(1/n)$, is equal to
\begin{equation}
\label{eq:cond_exp_var_usable_approximate}
\begin{split}
    &
    \prod_{i = 1}^k \E[f_i] \ - \ (n-1)^{-1} \sum_{i \neq j} \Cov(f_i, f_j) \prod_{m\neq j,i} \E[f_m] 
    \\& \quad
    + \
    ((n-1)(n-2))^{-1} \sum_{i \neq j \neq h \neq l} \Cov(f_i, f_j) \Cov(f_h, f_l) 
    \prod_{m\neq j,i,h,l} \E[f_m] 
    + \ \ldots
\end{split}
\end{equation}
where only the terms in which the expectations of products of up to 2 elements appear are taken.
\end{corollary}

\begin{proof}[Proof of Corollary \ref{cor:chain_cond_exp_var}.]
To conclude the proof just note that the constant 
\[ \prod_{i=0}^k \frac{n}{n-i} \]
coincides with
\[
\sum_{i=0}^{k/2}
\left[\prod_{j=0}^{i} \binom{k-2j}{2}\right] \big(-n^{-1}\big)^{i}
\]
and contains the following terms
\[ 
1 \ + \ \underbrace{\binom{k}{2} (-n^{-1})}_{\substack{\text{}\\\text{covariances}}} 
\ + \ \underbrace{\binom{k}{2}\binom{k-2}{2} (-n^{-1})^2}_{\substack{\text{the double}\\\text{covariances}}}  \ + \ \ldots
\]
and for all $i$ there are exactly $\prod_{j=0}^{i} \binom{k-2j}{2}$ terms with $i$ couples taken together in the expectation, and those terms are multiplied by $ \big(-n^{-1}\big)^{i}$.
\end{proof}

\section{Proof of Proposition \ref{theo:SGD_effect_fixed_batches}}
\label{proof:theor_SGD_effect_fixed}
In the following section we prove Proposition \ref{theo:SGD_effect_fixed_batches}. This is a first step towards the proof of \cref{theo:SGD_bias_eta}.

\subsection{Mini batch SGD vs Big-Learning-Rate GD.}
We expand in Taylor at initialization the difference between the trajectories of SGD and a GD with \textit{different} learning rate. Precisely we compare here $\theta_k^{SGD}$, so the parameters after $k$ steps of SGD with learning rate $\eta$ and batches $B_1,B_2, \ldots, B_k$ and $\theta_1^{GD, k\eta}$, so the parameters after one step of GD with learning rate $k\eta$ on the dataset $D_k := \cup_{i=1}^k B_i$. 
Precisely, note that
\begin{equation}
\label{eq:difference}
\begin{split}
    \theta_k^{SGD, \eta} - \theta_1^{GD, k\eta}
    \ &=\quad
    \theta_{k-1}^{SGD, \eta} 
    - \
    \eta \nabla L\big(\theta_{k-1}^{SGD, \eta} , B_k\big)
    \ - \
    \theta_1^{GD, (k-1)\eta}
    + \
    \eta \nabla L\big(\theta, B_k\big).
\end{split}
\end{equation}
So defining this deviation $\Delta$ as below, we have that
\begin{equation}
\label{eq:difference_1}
\begin{split}
    \Delta_k^{SGD} 
    \quad:&=\quad
    \theta_k^{SGD, \eta} - \theta_1^{GD, k\eta}
    \\&= \quad
    \Delta_{k-1}^{SGD} 
    \quad - \quad \eta \nabla L\big(\theta_{k-1}^{SGD, \eta}, B_k\big)
    \quad + \quad \eta \nabla L\big(\theta, B_k\big).
\end{split}
\end{equation}
Next, expanding in Taylor centered in $\theta$ we have
\begin{equation}
\label{eq:difference_2}
\begin{split}
    \Delta_k^{SGD} 
    \quad &= \quad
    \Delta_{k-1}^{SGD} 
    \quad - \quad 
    \eta \nabla^2 L\big( B_k\big)
    \left[\theta_{k-1}^{SGD, \eta} -\theta \right] 
    \quad + \quad 
    \substack{\text{higher order}\\\text{Taylor rest}}
    \\&= \quad
    \Delta_{k-1}^{SGD} 
    \quad-\quad 
    \eta \nabla^2 L\big( B_k\big)
    \left[\theta_1^{GD, (k-1)\eta} - \theta + \Delta_{k-1}^{SGD} \right]
    \quad+\quad
    \substack{\text{higher order}\\\text{Taylor rest}}.
\end{split}
\end{equation}
That is equal to
\begin{equation}
\label{eq:iteration}
    \left[I - \eta \nabla^2 L\big( B_k\big)\right] \Delta_{k-1}^{SGD}
    \quad + \quad 
    \underbrace{\eta^2 \nabla^2 L\big( B_k\big)
    \left[\sum_{i=1}^{k-1} \nabla L(B_i) \right]}_{ = \ \theta_1^{GD, \eta (k-1)} - \theta}
    \quad + \quad
    \substack{\text{higher order}\\\text{Taylor rest}}.
\end{equation}
%

We now consider the lowest order terms, let us define thus $\alpha_0^{SGD} = 0$, and for all $k \in \N$ denote $\alpha_k^{SGD} = \left[I - \eta \nabla^2 L\big( B_k\big)\right]  \alpha_{k-1}^{SGD} + \eta^2 \nabla^2 L\big( B_k\big) \sum_{i=1}^{k-1} \nabla L(B_i)$. Thus we obtain 
\begin{equation}
\label{eq:alpha}
    \alpha_k^{SGD}
    \quad = \quad 
    \eta^2 \sum_{i=2}^k  \left[ \left[\prod_{j=i+1}^k [I - \eta \nabla^2 L(B_j) ] \right] \nabla^2 L(B_i) \left[ \sum_{j=1}^{i-1} \nabla L(B_j) \right] \right].
\end{equation}
Next note that defining analogously $\alpha_k^{GD}$ in the case of $B_i = D$, for all $i$, we have the difference between the GD with small and big learning rates.

\subsection{Reorganizing the terms, 1} 
To conclude this part, 
the quantity in \cref{eq:alpha} can be rewritten as 
\begin{equation}
\begin{split}
    &
    \eta^2 \sum_{i=2}^k  \left[ \left[\prod_{j=i+1}^k [I - \eta \nabla^2 L(B_j) ] \right] \nabla^2 L(B_i) \left[ \sum_{j=1}^{i-1} \nabla L(B_j) \right] \right]
    \\& = \quad
    \eta^2
    \sum_{1 \leq i < j \leq k} 
    \Bigg[ \sum_{l < k-j}
    \sum_{\substack{\text{all the possible}\\j<h_1<h_2<\ldots <h_l\leq k}} 
    \prod_{m=1}^l (-\eta \nabla^2 L(B_{h_m})) \Bigg]  
    \ \nabla^2 L(B_j) 
    \ \nabla L (B_i).
\end{split}
\end{equation}
In particular, for all the batches $B_i$, in the sum above we have $\nabla L(B_i)$ multiplied by any possible product of $k-i-1$ terms that are for all $j$ all the $I$ or $\nabla^2 L(B_j)$. To be precise, all but the product of all identities.
This can be rewritten by grouping starting from $j$ being the biggest non-identity instead of the smallest one, as
\begin{equation}
\begin{split}
    &\eta^2
    \sum_{1 \leq i < j \leq k} 
    \nabla^2 L(B_j) 
    \Bigg[ \sum_{l < j-i}
    \sum_{\substack{\text{all the possible}\\i<h_1<h_2<\ldots <h_l<j}} 
    \prod_{m=1}^l (-\eta \nabla^2 L(B_{h_m})) \Bigg]  
    \ \nabla L (B_i)
    \\& = \quad
    \eta^2
    \sum_{1 \leq i < j \leq k} 
    \nabla^2 L(B_j) 
    \Bigg[
    \prod_{h=i+1}^{j-1} [I-\eta \nabla^2 L(B_{h})]
    \Bigg]  
    \ \nabla L (B_i).
\end{split}
\end{equation}

\subsection{Reorganizing the Terms, 2}
Thus the quantity above can be rewritten in 3 ways:
The first we have seen is \cref{eq:alpha}
\begin{equation}
    \eta^2 \sum_{i=2}^k  \left[ \left[\prod_{j=i+1}^k [I - \eta \nabla^2 L(B_j) ] \right] \nabla^2 L(B_i) \left[ \sum_{j=1}^{i-1} \nabla L(B_j) \right] \right].
\end{equation}
The second is 
\begin{equation}
\begin{split}
\label{eq:reg}
    \eta^2
    \sum_{1 \leq i < j \leq k} 
    \nabla^2 L(B_j) 
    \Bigg[
    \prod_{h=i+1}^{j-1} [I-\eta \nabla^2 L(B_{h})]
    \Bigg]  
    \ \nabla L (B_i).
\end{split}
\end{equation}
Finally, let us expand it to products of Hessians instead of $I- \eta \nabla^2 L$. Note that we will work on expectation later, any expectation we will work with does not change with re-ordering, thus every moment of the quantity above will correspond to moments of 
\begin{equation}
\label{eq:prod_hess}
\begin{split}
    - \sum_{i = 2}^{k}
    \binom{k}{i}
    \left[\prod_{j=2}^{i}
    [-\eta \nabla^2 L(B_j)]\right]      
    \ \eta \nabla L (B_1).
\end{split}
\end{equation}
All this concludes the proof of Proposition \ref{theo:SGD_effect_fixed_batches}.

\section{Proof of \cref{theo:SGD_bias_eta} and Proposition \ref{theo:SGD_effect_exp_fixed_batches}}

\label{proof:theo_SGD}

We prove here \ref{theo:SGD_effect_exp_fixed_batches} and \cref{theo:SGD_bias_eta}.
Note that if $D =  \{B_1, B_2, \ldots, B_k\}$, by taking $n=k$ and calling $B_i$ every $z_i$ in the notations of Setting \ref{setting:without_replacement}, we obtain the setting of the already sampled batches. From this point of view Proposition \ref{theo:SGD_effect_exp_fixed_batches} is a particular case of \cref{theo:SGD_bias}. We will prove just \cref{theo:SGD_bias}
without loss of generality.

\subsection{(Step 1) The Summands and the Potential}

Note that every summand of Proposition \ref{theo:SGD_effect_fixed_batches} is
\begin{equation}
\label{eq:def_g_i} 
    \nabla^2 L(B_j) 
    \Bigg[
    \prod_{h=i+1}^{j-1} [I-\eta \nabla^2 L(B_{h})]
    \Bigg]  
    \ \nabla L (B_i).
\end{equation}
Moreover, we have one for every $1 \leq i < j \leq k$.
Note that expanding all the sums (remembering that $\nabla_\theta^k L(B) = \frac{1}{|B|}\sum_{z \in B} \nabla_\theta^k L(\theta, z)$) we have that
\begin{equation}
\begin{split}
    &
    \E\left(
    \nabla^2 L(B_j) 
    \Bigg[
    \prod_{h=i+1}^{j-1} [I-\eta \nabla^2 L(B_{h})]
    \Bigg]  
    \ \nabla L (B_i)
    \right)
    \\&=\quad
    \E_{\left[\substack{z_i, z_{i+1}, \ldots, z_{j}\\ \text{all different}}\right]}
    \left(
    \nabla_\theta^2 L(z_j)
    \left[\prod_{h=i+1}^{j-1} [I - \eta \nabla_\theta^2 L(z_h)]\right]
    \nabla_\theta L(z_{i}) \right)
    \\&=\quad
    \E_{\left[\substack{z_i, z_{i+1}, \ldots, z_{j}\\ \text{all different}}\right]}
    \left(
    \nabla_\theta^2 L(z_\red{i})
    \left[\prod_{h=i+1}^{j-1} [I - \eta \nabla_\theta^2 L(z_h)]\right]
    \nabla_\theta L(z_\red{j}) \right).
\end{split}
\end{equation}
Where the second equality comes from the fact that we are averaging over the orders, so the order does not matter. Note that in the line 2 and 3 above we have the expectation of the elements in line 1 and 2 of \cref{eq:deriv_1}.


\paragraph{The Potential.}
As in \cref{eq:potential}, let us define the function $V_{i,j} \colon \R^d \to \R$ as follows 
\begin{equation}
\begin{split}
\label{eq:def_V_i}
    V_{i,j}(\theta) 
    \quad := \quad
    \left[\nabla_\theta L(B_j)\right]^\top
    \left[\prod_{h=i+1}^{j-1} [I - \eta \nabla_\theta^2 L(B_h)]\right]
    \nabla_\theta L(B_{i})
    .
\end{split}
\end{equation}
Note that, analogously as in the paragraph above, $\E[V_{i,j}(\theta)]$ is the expectation of
\begin{equation}
\begin{split}
\label{eq:elements_of_V}
    [\nabla_\theta L(z_j)]^\top 
    \left[\prod_{h=i+1}^{j-1} [I - \eta \nabla_\theta^2 L(z_h)]\right]
    \nabla_\theta L(z_{i})
\end{split}
\end{equation}
where $z_i, z_{i+1}, \ldots, z_{j} \in D$ are all different.
Note that the expectation is over a uniform distribution over a finite set, so we can exchange differentiation in $\theta$ and expectation over $x \in D$. 
The derivative in $\theta$ of $\E[V_{i,j}(\theta)]$ as already hinted is
\begin{equation}
\begin{split}
    &\quad
    \E_{\left[\substack{z_i, z_{i+1}, \ldots, z_{j}\\ \text{all different}}\right]}
    \left(
    \nabla_\theta^2 L(z_j)
    \left[\prod_{h=i+1}^{j-1} [I - \eta \nabla_\theta^2 L(z_h)]\right]
    \nabla_\theta L(z_{i}) \right)
    \\&+
    \E_{\left[\substack{z_i, z_{i+1}, \ldots, z_{j}\\ \text{all different}}\right]}
    \left(
    \nabla_\theta^2 L(z_i)
    \left[\prod_{h=i+1}^{j-1} [I - \eta \nabla_\theta^2 L(z_h)]\right]
    \nabla_\theta L(z_{j}) \right)
    \\&-
    \E_{\left[\substack{z_i, z_{i+1}, \ldots, z_{j}\\ \text{all different}}\right]}
    \left(
    \sum_{h = i+1}^{j-1}
    [\nabla_\theta L(z_j)]^\top 
    \left[\prod_{l=h+1}^{j-1} [I - \eta \nabla_\theta^2 L(z_l)]\right]
    \eta \nabla^3_\theta L(z_h) 
    \left[\prod_{l=i+1}^{h-1} [I - \eta \nabla_\theta^2 L(z_l)]\right]
    \nabla_\theta L(z_{i})
    \right)
\end{split}
\end{equation}
Here the first two lines are simply the same expectation as seen above, and the second is part of the terms in the sum in \cref{eq:error}.
Since for the expectation the only thing that matters is how many elements we are multiplying, we can write the expectation of the term in Proposition \ref{theo:SGD_effect_fixed_batches} as
\begin{equation}
\begin{split}
\label{eq:expect_delta}
    \frac{\eta^2}{2} \nabla_\theta \left(
    \sum_{i<j}
    \E_{\left[\substack{\text{over the}\\\text{sampling}\\\text{of batches}} \right]}
    V_{i,j}(\theta)
    \right)
    \quad = \quad
    \frac{\eta^2}{2} \nabla_\theta \left(
    \sum_{1 < i \leq k}
    (k-i+1)
    \E_{\left[\substack{\text{over the}\\\text{sampling}\\\text{of batches}} \right]}
    V_{1,i}(\theta)
    \right)
\end{split}
\end{equation}
minus the term where we exchange every $\nabla L(B_i)$ with $\nabla L$ and $\nabla^2 L(B_i)$ with $\nabla^2 L$.
This that if the mini-batches add up to the dataset, every epoch of SGD corresponds to the same number of steps of GD plus an additional step of the regularizer given by this term, precisely by observing that
\begin{equation}
\begin{split}
&\frac{\eta^2}{2} 
\sum_{1 \leq i < j \leq k}
    \nabla L^\top
    [I-\eta \nabla^2 L]^{j-i-1}
    \nabla L (B_i)
\\& \quad = \quad
\frac{\eta^2}{2} \nabla L^\top [\nabla^2 L]^{2\dag} \big( \exp(-\eta k \nabla^2 L) + \eta k \nabla^2 L - I \big) \nabla L
\end{split}
\end{equation}
we conclude that
\begin{equation}
\begin{split}
\label{eq:reg_formula}
    Reg \quad &= \quad \frac{1}{2}
    \sum_{1 \leq i < j \leq k}
    \eta \nabla L(B_j) ^\top
    \Bigg[
    \prod_{h=i+1}^{j-1} [I-\eta \nabla^2 L(B_{h})]
    \Bigg]  
    \ \eta \nabla L (B_i)
    \\& \quad - \quad 
    \nabla L^\top [\nabla^2 L]^{2\dag} \big( [I - \eta \nabla^2 L]^k + \eta k \nabla^2 L - I \big) \nabla L.
\end{split}
\end{equation}

\subsection{(Step 2) Taking the Expectation of the Regularizer}
We compute now the expectations of the step on the regularizer in \cref{eq:reg_formula}.
Recall from \cref{eq:average_sampling_sets,eq:average_sampling_without_rep,eq:cond_exp_var_usable_approximate} and from Corollary \ref{cor:chain_cond_exp_var} that we can write the expectation as the non-commutative version of (we copy here the commutative version for clarity)
\begin{equation*}
\begin{split}
     &
    \prod_{i = 1}^k \E[f_i] \ - \ (n-1)^{-1} \sum_{i \neq j} \Cov(f_i, f_j) \prod_{m\neq j,i} \E[f_m] 
    \\& \quad
    + \
    ((n-1)(n-2))^{-1} \sum_{i \neq j \neq h \neq l} \Cov(f_i, f_j) \Cov(f_h, f_l) 
    \prod_{m\neq j,i,h,l} \E[f_m] 
    + \ \ldots
\end{split}
\end{equation*}
First, note that the part we are removing due to GD is exactly
\[
\sum \prod \E[f_i] 
\quad = \quad 
\underbrace{ 
\nabla L^\top [\nabla^2 L]^{2\dag} \big( \exp(-\eta k \nabla^2 L) + \eta k \nabla^2 L - I \big) \nabla L
}_{\text{GD part}}
\]
because the expectation of every derivative is the full-batch derivative.
Let's analyze the step on the regularizer in the form \cref{eq:prod_hess}. Precisely, we have to compute the following expectation
\begin{equation}
    - \ \E_{\left[\substack{B_1, B_2, \ldots, B_z\\ \text{all disjoint}} \right]} \left[
    \sum_{i = 2}^{k}
    \binom{k}{i}
    \left[\prod_{j=2}^{i}
    [-\eta \nabla^2 L(B_j)]\right]      
    \ \eta \nabla L (B_1)
    \right].
\end{equation}

Where $f_1 = \nabla L$ and $f_i = - \eta \nabla^2 L$ for all $i > 1$. We split this sum into 2 parts.

\begin{enumerate}[(i)]
    \item The terms in which $f_1$ and $f_i$ are taken together in the covariances. Here, every summand of the quantity above can be rewritten as
    \begin{equation}
    \label{eq:reg_1}
    \begin{split}
    &
    \E_{B_1 \subseteq D} \Bigg[ \E_{\left[\substack{B_2, \ldots, B_{j-1}, B_{j+1}, \ldots, B_i\\ \text{all disjoing}
    \\ \text{and disjoint from }B_1} \right]} \Bigg[
    \sum_{j = 2}^{i}
    \left[\prod_{h=j+1}^{i}
    [-\eta \nabla^2 L(B_j)]\right] 
    \eta \nabla^2 L (B_1)
    \\& \qquad \quad \cdot
    \left[\prod_{h=2}^{j-1}
    [-\eta \nabla^2 L(B_j)]\right] 
    \ \eta \nabla L (B_1)
    \Bigg] \Bigg]
    \quad - \quad \ldots.
    \end{split}
    \end{equation}
    We will better analyze this term in what follows.
    
    \item The terms in which $f_1$, appear alone (so we have $\E_{z \in D} [f_1(z)]$ appearing in the formula are terms in which figures the covariance of at least 2 Hessians. These terms disappear around a stationary point, as $\E_{z \in D} [f_1(z)] = \eta \nabla L \sim 0$ there.
    \begin{equation}
    \label{eq:reg_2}
    \begin{split}
    &
    - \ \E_{\left[\substack{z_2, \ldots, z_i\\ \text{all different}} \right]} \left[
    \sum_{i = 2}^{k}
    \binom{k}{i}
    \left[\prod_{j=2}^{i}
    [-\eta \nabla^2 L(B_j)]\right]    \right]  
    \ \cdot \ \eta \nabla L.
    \end{split}
    \end{equation}

\end{enumerate}

\subsection{(Step 3) Summing up.}
The term in item (i) above can be rewritten as follows
\begin{equation}
\begin{split}
    &
    \E_{z_1 \in D} \Bigg[ \E_{\left[\substack{z_2, \ldots, z_{j-1}, z_{j+1}, \ldots, z_i\\ \text{all different in }D}\right]} \Bigg[
    \sum_{i = 2}^{k}
    \binom{k}{i}
    \sum_{j = 2}^{i}
    \left[\prod_{h=j+1}^{i}
    [-\eta \nabla^2 L(B_j)]\right] 
    \eta \nabla^2 L (B_1)
    \\& \qquad \quad \cdot
    \left[\prod_{h=2}^{j-1}
    [-\eta \nabla^2 L(B_j)]\right] 
    \ \eta \nabla L (B_1)
    \Bigg] \Bigg]
    \quad - \quad \ldots.
\end{split}
\end{equation}
We can rewrite the quantity above as follows
\begin{equation}
\begin{split}
    &
    \sum_{i = 2}^{k}
    \binom{k}{i}
    \sum_{j = 2}^{i}
    \E_{\left[\substack{z_{j+1}, \ldots, z_i\\ \text{all different in }D}\right]}
    \left[\prod_{h=j+1}^{i}
    [-\eta \nabla^2 L(B_j)]\right] 
    \E_{z_1 \in D} \Bigg[ \eta \nabla^2 L (B_1)
    \\& \qquad \quad \cdot
    \E_{\left[\substack{z_2, \ldots, z_{j-1}\\ \text{all different in }D}\right]}
    \left[\prod_{h=2}^{j-1}
    [-\eta \nabla^2 L(B_j)]\right] 
    \ \eta \nabla L (B_1)
    \Bigg]
    \quad - \quad \ldots.
\end{split}
\end{equation}
By observing that 
\[
\Cov_{z \in D}(\nabla^2 L(z), \nabla L(z)) 
\quad = \quad
\frac{1}{2} \nabla \E_{z \in D}[ \mednorm{\nabla L(z) - L}^2]
\]
We can reorganize the terms as follows
\begin{equation}
\begin{split}
    Reg \ step \ = \ - \frac{c}{n-k} \nabla_\theta 
    \sum_{i = 0}^{k-2}
    \E_{\left[\substack{z_{1}, \ldots, z_i\\ \text{all different in }D}\right]}
    \left[\prod_{j=1}^{i}
    [-\eta \nabla^2 L(B_j)]\right] 
    \var_{z \in D} (\mednorm{
    \nabla L (z) - \nabla L}^2_{S_i} )
\end{split}
\end{equation}
where $S_i$ is defined as
\[
S_{i} \quad := \quad \frac{\eta^2}{2c}
\sum_{j=0}^{k-i-2} \binom{k}{i+2} \E_{\left[\substack{z_{1}, \ldots, z_j\\ \text{all different in }D}\right]}
    \left[\prod_{j=1}^{k-i-2}
    [-\eta \nabla^2 L(B_j)]\right] 
\]
Precisely, the part $\prod_{i} \E_{z \in D}[f_i(z)]$ in the notations of Corollary \ref{cor:chain_cond_exp_var} of this quantity is
\[
S_i \quad = \quad \frac{\eta^2}{2c}
[-\eta\nabla^2 L]^{- i - 2 }
 \left( [I - \eta \nabla^2 L]^{k} - \sum_{j=0}^{i+1} \binom{k}{j} (-\eta \nabla^2 L )^j \right).
\]
This means that around stationary points and in the case in which $\E[ (\eta \nabla^2 L)^2] - \E[ \eta \nabla^2 L]^2 \ll n$ (usually as they are $O(1)$ )
We have
\begin{equation}
\label{eq:reg_small_var}
    Reg \ step \quad = \quad 
    - \frac{\eta}{b-1} \ \nabla_\theta 
    \sum_{i = 0}^{k-2} (-\eta \nabla^2 L )^{i} 
    \var_{z \in D} (\mednorm{
    \nabla L (z) - \nabla L}^2_{S_i} ).
\end{equation}

\subsection{(Step 4) The Size of the Step Given the Hessian}
\label{section:proof_S}

Note that in the case in which $\E[ (\eta \nabla^2 L)^2] - \E[ \eta \nabla^2 L]^2 \ll n$, that happens often as usually as they are $O_n(1)$
\begin{equation}
\label{eq:S_i}
\begin{split}
    S_i 
    \quad &= \quad
    \frac{\eta^2}{2c}
    [-\eta \nabla^2 L]^{- i - 2 }
    \left( [I - \eta \nabla^2 L]^{k} - \sum_{j=0}^{i+1} \binom{k}{j} (-\eta \nabla^2L)^j \right)
    \\&=\quad 
    \frac{\eta^2}{2c}\sum_{j=i+2}^{k} \binom{k}{j} (-\eta \nabla^2L)^{j-i-2}
\end{split}
\end{equation}
This matrix is a function of the Hessian and on the span of  eigenvectors of the eigenvalue $\lambda$ of the Hessian, $S_i$ takes the following values approximately
\begin{equation}
S_i \quad = \quad
\frac{\eta^2}{2c}
\begin{cases}
    \binom{k}{i+2} - \binom{k}{i+3}\eta\lambda \qquad
    & \text{ when }|c\lambda| \ll 1 \\
    \binom{k}{i+1}(\eta\lambda)^{-1}
    - \binom{k}{i}(\eta\lambda)^{-2} + \ldots
    & \text{ when }c\lambda \gg 0 \\
    (-\eta \lambda)^{-i-2} \exp(-c \lambda)
    & \text{ when } c\lambda \ll 0 
\end{cases}
\end{equation}
This concludes the proof of \cref{theo:SGD_bias_eta}.

\section{A Deeper Look into $S_0$ of \cref{theo:SGD_bias_eta}}
Our goal in the rest of this section is to describe the nature of this regularizer and in particular the role of the matrices $S_i$, $i \leq k-2$, in driving qualitatively different behaviors between SGD with and without replacement.

Recall from \S \ref{section:notations} that $c:= \eta k$ and $H : = \frac{c^2}{2n}\big( \E[(\nabla^2 L)^2] - (\nabla^2 L)^2 \big)$.

\begin{proposition}[What $S_i$ is]
\label{prop:S}
In the setting of \cref{theo:SGD_bias_eta} we have that $S_i$ approximately share eigenspaces with $\nabla^2 L$. More precisely,
\[
S_i \ = \
\frac{\eta^2}{2c}\sum_{j=i+2}^k \binom{k}{j}[-\eta \nabla^2 L]^{j-i-2}
\ = \
\frac{\eta^2}{2c}[-\eta \nabla^2 L]^{-i-2} \left( 
[I-\eta \nabla^2 L]^k - \sum_{j=0}^{i-1} \binom{k}{j}[-\eta \nabla^2 L]^{j}
\right).
\]
Thus, on the eigenspace of the eigenvector $\lambda$ of the Hessian, we have that
\begin{equation}
    S_0 
    \quad \sim \quad
    \begin{cases}
    \frac{c}{4}
    H^{-1/2} \sqrt{\pi} \ \mathrm{erf}(\sqrt{H}) \left( 2+ 2H^{-1/2}(e^{-H}-1) - c\lambda \right)
    & \text{if }c\lambda \text{ small}\\
    (2\lambda)^{-1}
    \qquad &\text{if }c\lambda >> 0 \\
    \tfrac{1}{2\lambda^2}\exp(-c\lambda) \left(e^{-H} -4H/c\lambda \right)
    \qquad & \text{if }c\lambda << 0
    \end{cases}
\end{equation}
In particular, when $c\lambda$ is small we have up to $O(c^2\lambda^2)$ that
\begin{equation}
    S_0 
    \quad \sim \quad
    \frac{c}{4}\cdot
    \begin{cases}
    1 - \frac{1}{3}c\lambda - \frac{1}{6}H + \frac{1}{10}c\lambda H + \frac{1}{30}H^2 - \frac{1}{42}c\lambda H^2
    & \text{if } H \text{ small} \\
    H^{-1} \exp(-H) \left( 2H^{-1} - c\lambda \right)
    & \text{if } H \text{ big}
    \end{cases}
\end{equation}
and for $H \ll 1$ we have up to $O(H)$ approximately
\begin{equation}
    cS_i 
    \quad \sim \quad \frac{1}{2}
    \begin{cases}
    \binom{k}{i+2} - \binom{k}{i+3} \eta \lambda
    & \text{if }c\lambda \text{ small}\\
    \lambda^{-1}
    \qquad &\text{if }c\lambda >> 0 \\
    \lambda^{-2}\exp(-c\lambda)
    \qquad & \text{if }c\lambda << 0
    \end{cases}
\end{equation}
In addition, $\mathcal{S}_i$ in \cref{theo:SGD_bias_eta} is approximately the following
\[
\mathcal{S}_i \ = \ [I-\eta \nabla^2L]^{2\dag}H S_i.
\]
\end{proposition}

Moreover, the non-diagonal part of $S_i$ only depends on $H$ so in particular, if $\mednorm{H} \ll 1$ then $S_i$ is essentially diagonal and precisely, if $\lambda_1, \lambda_2, \ldots, \lambda_p$ are the eigenvalues of $\nabla^2 L = diag(\lambda_1, \lambda_2, \ldots, \lambda_p)$ we have that
\[
S_i \quad = \quad
diag \big( f_i(\lambda_1), f_i(\lambda_2), \ldots, f_i(\lambda_p) \big)
\quad + \ O(H)
\]
with 
\[
f_i(\lambda)
\quad = \quad 
\tfrac{1}{2} \lambda^{-2}
\left( 
[1 - \eta \lambda]^k - \sum_{j=0}^{i+1} \binom{k}{j} (- \eta \lambda)^j
\right)
\quad \cdot \quad \big(1 + O(1/k) \big),
\]
or analogously
\[
f_i(\lambda)
\quad = \quad 
\tfrac{1}{2} 
\sum_{j=i+2}^{k} \binom{k}{j} (- \eta \lambda)^{j-2}
\quad \cdot \quad \big(1 + O(1/k) \big),
\]

\paragraph{$S_0$ and the small Hassian regime.}
Note that $S_0$ \cref{theo:SGD_bias_eta} is
\[
\sum_{i=2}^{k}(k-i+1)\E_{\left[\substack{z_{2}, \ldots, z_{i-1} \\ \text{all different}}\right]}  
    \left[
    \prod_{h=2}^{i-1} [I - \eta \nabla_\theta^2 L(z_h)]
    \right] 
\]
Again using the 
\cref{eq:average_sampling_sets,eq:average_sampling_without_rep,eq:cond_exp_usable} we can write the expectation as the non-commutative version of (we write here the commutative version for clarity) $\prod_{j=1}^{i-1}\left[1 + \frac{j}{n-j}\right]$ that multiplies
\begin{equation*}
\begin{split}
    &
    [I-\eta \nabla^2 L]^{i-2} - n^{-1} \binom{i-2}{2} \big(\E[(I - \eta \nabla^2 L)^2] - \E[I - \eta \nabla^2 L]^2\big)  [I-\eta \nabla^2 L]^{i-4}
    +
    \ldots
\end{split}
\end{equation*}
So calling $X:= \E[(\nabla^2 L)^2] - \E[\nabla^2 L]^2 $, and $Y:= I-\eta \nabla^2 L$ we have
\begin{equation}
\begin{split}
    &
    Y^{i-2} 
    - \eta^2 n^{-1} \binom{i-2}{2} X  Y^{i-4}
    + \eta^4 n^{-2} \tfrac{1}{2} \binom{i-2}{2}\binom{i-4}{2} X^2  Y^{i-6}
    +
    \ldots
\end{split}
\end{equation}

Precisely, it is
\begin{equation}
\begin{split}
    &
    Y^{i-2} 
    - \eta^2 n^{-1} \bigg( X Y^{i-4} + Y X Y^{i-3} + 
    Y^2 X Y^{i-2} + \ldots
    +
    \eta^2 \tilde X Y^{i-3} + \ldots
\end{split}
\end{equation}
 and so on.
We can anyways count how many terms of every kind we have, precisely
in the commutative case
\begin{equation}
\label{eq:S_series}
\begin{split}
    &
    S_0 \quad = \quad \frac{\eta}{2k} \sum_{i=2}^{k} (k-i+1) \Bigg(
    Y^{i-2} 
    - \eta^2 n^{-1} \binom{i-2}{2} X  Y^{i-4}
    +
    \ldots \Bigg)
\end{split}
\end{equation}
\paragraph{When $H \ll 1$.} Then we can consider
And $\tilde Y := \sum_{i=2}^{k} (k-i+1) Y^{i-2}$ is
\[
\sum_{i=2}^{k} (k-i+1) Y^{i-2} 
\quad  = \quad 
\frac{-I + k - k Y + Y^k}{(Y-I)^2}
\quad  = \quad 
\frac{[I -  \eta \nabla^2 L ]^k + \eta k \nabla^2 L - I}{\eta^2 [\nabla^2 L]^2}
\]
That on the eigenspace of $\lambda$ eigenvalue of the Hessian, for $\eta k = c$ this is about
\begin{equation}
\label{eq:Y_tilde}
\frac{ \exp(-c\lambda) + c\lambda - I}{\eta^2 \lambda^2}
\end{equation}
And if we expand in Taylor the exponential, we obtain, agreeing with \cite{smith_origin_2021}, that
\[
\frac{I - c\lambda + \tfrac{c^2\lambda^2}{2} - \tfrac{c^3\lambda^3}{6} + \frac{c^4\lambda^4
}{24} + c \lambda - I}{\lambda^2}
\quad = \quad
\frac{c^2
}{2}
\ - \
\frac{c^3\lambda}{6}
\ + \
\frac{c^4\lambda^2}{24}
\]
In particular,
\begin{equation}
\eta^2 \tilde Y \quad \sim \quad
\begin{cases}
    \exp(-c\lambda)
    \qquad 
    & \text{if }c\lambda << 0\\
    \frac{c^2}{2} \ - \ \frac{c^3\lambda}{6}  
    & \text{if }c\lambda \sim 0\\
    c\lambda^{-1}
    &\text{if }c\lambda >> 0 \\
\end{cases}
\end{equation}

\paragraph{When $c\lambda << 1$, whatever $X/n$.}
Now regarding the terms with also $X$, we have that in the case of commutation the sum is
\[
\eta^4 X \frac{1}{n} \sum_{i=4}^k (k-i+1) \binom{i-2}{2} Y^{i-4}
\]
That happens to be, if $\eta \lambda << 1$
\[
\frac{\eta^4}{n}\binom{k}{4} X 
\]
and otherwise
\begin{equation}
\begin{split}
    &
    \frac{1}{n} X\lambda^{-4}
    \left(
    -3 + c \lambda + \tfrac{1}{2} (1 - \eta \lambda)^{k-2} [c^2 \lambda^2 - 5 \eta c \lambda^2 + 4 c \lambda + 6 \eta^2 \lambda^2 - 12 \eta \lambda + 6]
    \right)
\end{split}
\end{equation}
This, when $\eta \lambda << 1$ is
\begin{equation}
\label{eq:S_X}
\begin{split}
    &
    \frac{1}{n}\lambda^{-4}X
    \left(
    -3 + c\lambda + \tfrac{1}{2} \exp(-c\lambda) [(c\lambda  + 2)^2 + 2 ]
    \right)
    \\&=\quad
    \frac{1}{n}X \left(
    \frac{c^4}{24} \ - \ \frac{c^5\lambda}{40} \ + \ 
    \frac{c^6\lambda^2}{120}
    \ + \
    O\big(c^{2i} \lambda^{i-2} / i! \big) \right)
\end{split}
\end{equation}
We can here compute the size of this term as for $\tilde Y$. precisely the size is
\begin{equation}
\frac{1}{n}\left(\E[(\nabla^2L)^2]-\E[\nabla^2 L]^2\right) \quad \times \quad
\begin{cases}
    > \exp(-c\lambda)
    \qquad 
    & \text{if }\lambda << 0\\
    \frac{c^4}{4!} \ - \ \frac{3 c^5\lambda}{5!} \ + \ \frac{5 c^6\lambda^2}{6!}
    & \text{if }\lambda \sim 0\\
    < c\lambda^{-3}
    &\text{if }\lambda >> 0 \\
\end{cases}
\end{equation}
More generally, 
the term that multiply the $X^i$s are:
\[
    \eta^{4}\sum_{i=2}^k (k-i+1)(I-\eta \lambda)^{i-2} 
    \quad = \quad 
    \frac{c^2}{2}
    \ - \
    \frac{c^3\lambda}{6}
    \ + \ O(\lambda^2)
\]
that is equal to
\[
\lambda^{-2}
    \left(
    \exp(-c\lambda) + c\lambda - 1
    \right)
\]
for $i=0$. Recall $H : = \frac{c^2}{2n}X$, then the sum above is
\[
    \eta^{4}\sum_{i=4}^k (k-i+1)\frac{(i-2)!}{(i-4)! \cdot 2}  (I-\eta \lambda)^{i-4} \frac{X}{n}
    \quad = \quad 
    \frac{c^2}{12}H
    \ - \
    \frac{c^3\lambda}{20}H
    \ + \ O(\lambda^2H)
\]
that is equal to
\[
\lambda^{-4} H
    \left(
    - 3 + c \lambda + \tfrac{1}{2} \exp(-c\lambda) [c^2 \lambda^2 + 4 c \lambda + 6]
    \right)
\]
for $i = 1$, this
\[
    \eta^{6}\sum_{i=6}^k (k-i+1)\frac{(i-2)!}{(i-6)! \cdot 8}  (I-\eta \lambda)^{i-6} \frac{X^2}{n^2}
    \quad = \quad 
    \frac{c^2}{60}H^2 
    \ - \
    \frac{c^3\lambda}{84}H^2
    \ + \ O(\lambda^2H^2)
\]
that is equal to
\[
\lambda^{-6} H^2 \left(-15 + 3c\lambda + \exp(-c\lambda) (15 + 12 c \lambda + 4.5 c^2 \lambda^2 + c^3 \lambda^3 + c^4\lambda^4/8 
) \right)
\]
for $i = 2$, this
\[
    \eta^{8}\sum_{i=8}^k (k-i+1) \frac{(i-2)!}{(i-8)! \cdot 6 \cdot 8} (I-\eta \lambda)^{i-8} \frac{X^3}{n^3}
    \quad = \quad
    \frac{c^2}{336}H^3
    \ - \ \frac{c^3\lambda}{432} H^3
    \ + \ O(\lambda^2H^3)
\]
that is equal to
\[
\lambda^{-8} H^3
\left(-105 + 15c\lambda + \exp(-c\lambda) \big(105 + 90c \lambda + \tfrac{75}{2} c^2 \lambda^2 + \ldots + c^5\lambda^5/4 + c^6\lambda^6/48 \big) 
) \right)
\]
for $i = 3$, and so on.

Thus we have that for $\lambda << 1$ or $\eta << 1$, the zeroth order in $\lambda$ of $S_0$ is:
\[
\frac{c^2}{2} \left( 1 - \frac{1}{6}H + \frac{1}{30}H^2 - \frac{1}{168}H^3 + \ldots \right).
\]
Here, at the denominators we have $(n+1)! \cdot (2n+1)$ that is OEIS A175925. So we have
\[
\frac{c^2}{2} \sum_{m = 0}^{k/2 - 1} \frac{1}{(m+1)! (2m +1)} (-H)^m
\quad \sim \quad 
 \frac{c^{2}}{2}H^{-1} \big( \sqrt{\pi H} \ \mathrm{erf} (\sqrt{H}) + e^{-H} -1 \big)
\]
That for $H \to \infty$ goes to 0 as 
\[
\frac{c^{2}}{2} \left(H^{-2} + O(H^{-3}) \right) \ e^{-H}.
\]

The first order in $\lambda$ instead is $-c \lambda$ that multiplies
\[
\frac{c^2}{2} \left(\frac{1}{3} - \frac{1}{10}H + \frac{1}{42}H^2 - \frac{1}{216}H^3 + \ldots 
\right)
\]
precisely $-c \lambda$ that multiplies terms that at the denominator have $n! \cdot (2n+1)$, that is OEIS A007680,  precisely
\[
-c \lambda \frac{c^2}{2} \sum_{m = 0}^{k/2 - 1} \frac{1}{m! (2m +1)} (-H)^m
\quad \sim \quad 
 -\frac{c^3\lambda}{4} \sqrt{\pi} H^{-1/2} \ \mathrm{erf}\!\left( \sqrt{H} \right)
\]
That for $H \to \infty$ goes to 0 as 
\[
-\frac{c^{3}\lambda}{4} \left(H^{-1} + O(H^{-2}) \right) \ e^{-H}.
\]
Thus we can conclude that for $c\lambda << 1$, by recalling $H : = \frac{c^2}{2n}X$ we have 
\begin{equation}
S 
\quad = \quad \frac{c}{4}\cdot
\begin{cases}
1 - \frac{1}{3}c\lambda - \frac{1}{6}H + \frac{1}{10}c\lambda H + \frac{1}{30}H^2 - \frac{1}{42}c\lambda H^2
& \text{if } H \text{ small} \\
H^{-1} \exp(-H) \left( 2H^{-1} - c\lambda \right)
& \text{if } H \text{ big}
\end{cases}
\end{equation}
and in general
\begin{equation}
S 
\quad = \quad \frac{c}{4}
H^{-1/2} \sqrt{\pi} \ \mathrm{erf}(\sqrt{H}) \left( 2+ 2H^{-1/2}(e^{-H}-1) - c\lambda \right)
\end{equation}
with an error of $O(c^2\lambda^2)$.

\paragraph{If $c\lambda \neq 0$.} We sum all the elements in \cref{eq:S_series} multiplied by $X$ alone, then all the ones by $X^2$ alone, then all the ones by all the other possible expectations of products, and so on. All these sums are sums of powers of $\lambda$, precisely all of them, as \cref{eq:Y_tilde,eq:S_i} are a negative power of $\lambda$ multiplied by some polynomial terms in $c\lambda$, usually a term that is $O(1)$ and one that is $O(c\lambda)$, and another one that is $O(\max\{c^i\lambda^i, 1\}) \exp(-c\lambda)$. Precisely, all these terms for $i \geq 1$ are $[-X/n]^i$ multiplied by
\begin{equation}
    \lambda^{-2i-2} \left(-a_i + b_i \cdot c\lambda + \exp(-c\lambda) \left(a_i + (a_i-b_i) c \lambda + O(c^2\lambda^2 + c^{2i-2}\lambda^{2i-2}) + \frac{c^{2i}\lambda^{2i}}{2^{i}i!} \right)\right)
\end{equation}
where 
\[
a_i = (2i+1)!! \qandq b_i = (2i-1)!!.
\]

\paragraph{If $\lambda << 0$.}
This and the equation in \cref{eq:Y_tilde} imply  $\lambda << 0$ we have the exponential in $\lambda$ that governs the terms. Moreover, since $\lambda^2 >> 1$, we have that the biggest terms are the ones divided by $\lambda^2$ to a smaller exponent. Recall that $H:= c^2X/2n$, and noting that generally $X/n = O(1/n)$. So in this case $2H/c^2\lambda^2 = X/\lambda^2n = O(\lambda^2 n)^{-1}  \ll 1$. We obtain that all the terms that sum the exponential have size smaller than the last one. So in particular the size of $S_0$ in that direction is about
\begin{equation}
\lambda^{-2} \exp(-c\lambda) \left( 1+  \sum_{i=1} a_i\left[\frac{-H}{c^2\lambda^2}\right]^i + c \lambda (a_i-b_i)\left[\frac{-H}{c^2\lambda^2}\right]^i + \ldots + \frac{(-H)^i}{i!}\right)
\end{equation}
that is 
\begin{equation}
\lambda^{-2} \exp(-c\lambda) \left(e^{-H} - 6(c\lambda)^{-2}H
- 4(c\lambda)^{-1} H + O(H^2/\lambda) \right)
\end{equation}
and can be rewritten
\begin{equation}
\lambda^{-2} \exp(-c\lambda) \left(e^{-H} -  4(c\lambda)^{-1} H + O(H^2/c\lambda) \right).
\end{equation}

\paragraph{If $\lambda >> 0$.} Analogously in this case, the term that governs the quantity is
\[
\lambda^{-2i-2} (- a_i + c \lambda b_i) [-X/n]^i
\]
as the exponential eats all the other terms. Thus we have the following series
\begin{equation}
\lambda^{-2} \sum_{i=0}^{k/2-1} -a_i\left[\frac{-H}{c^2\lambda^2}\right]^i + c \lambda b_i\left[\frac{-H}{c^2\lambda^2}\right]^i
\end{equation}
And unlike the previous one, the biggest terms in the case in which $\lambda >> 0$ remain 
\[
\frac{c}{\lambda} + O(\lambda^{-2})
\]
where the error is always negative and so in particular it is $< \frac{c}{\lambda}$. This concludes the proof.

\section{The Whole Regularizer: \cref{theo:SGD_bias}}
\label{section:reg_new}
Assume $H \ll1$. 
We proved that every epoch SGD without replacement, in expectation is GD plus one additional step on a regularizer in \cref{eq:reg_small_var}. Let us now rewrite everything in an orthonormal basis of eigenvectors for the Hessian. We have that the update on $\theta_i$ eigenvector of the $i$-th eigenvalue $\lambda_i$ of the Hessian $\nabla^2 L$ is
\[
    - \frac{\eta}{b-1} 
    \sum_{j = 0}^{k-2} (-\eta \lambda_i )^{j} 
    \ \nabla_{\theta_i} \var_{z \in D} (\mednorm{
    \nabla L (z) - \nabla L}^2_{S_j} ).
\]
Or analogously up to terms in the approximation error
\begin{equation}
\begin{split}
    - \frac{\eta}{b-1}
    \sum_{j = 0}^{k-2} (-\eta \lambda_i)^{j} 
    \mathrm{trace}\big( S_j \cdot \frac{d}{d\theta_i}\Cov_{z \in D} (\nabla L (z) ) \big).
\end{split}
\end{equation}
Next note that we proved that $S_j$ is approximately simultaneously diagonalizable with $\nabla^2 L$, so in this basis it is diagonal and on the $i-th$ element of the diagonal it is 
\[
\frac{\eta^2}{2c} \sum_{l = j+2}^k \binom{k}{l} (-\eta \lambda_i)^{l-j-2}
\]
that approximately corresponds to  
\[
S_j[i,i] = \frac{\eta^2}{2c} 
\begin{cases}
    \binom{k}{j+2} - \binom{k}{j+3} \eta \lambda_i
    & \text{if }c\lambda_i \ll 1
    \\
    \frac{[1-\eta \lambda_i]^k}{-\eta\lambda_i^{j+2}} + \binom{k}{j+1} (\eta \lambda_i)^{-1}
    & \text{if }c\lambda_i \gg 1
\end{cases}.
\]
By expanding the trace as a sum over the eigenvalues $\lambda_1, \lambda_2, \ldots, \lambda_l, \ldots$ of $\nabla^2 L$, along the direction $\theta_i$ we can rewrite our steps as
\begin{equation}
\begin{split}
\label{eq:whole_reg}
    &-\frac{\eta}{b-1} \sum_{l}
    \sum_{j = 0}^{k-2}
    (-\eta \lambda_i )^{j} 
    S_j[l,l] \cdot \frac{d}{d\theta_i} \Cov_{z \in D} (\nabla L (z) )_{l,l}
    \\&= \quad 
    -\frac{\eta}{b-1} \frac{\eta^2}{2c} \sum_{l}
    \sum_{j = 0}^{k-2}
    (-\eta \lambda_i )^{j} 
    \sum_{h = j+2}^{k} \binom{k}{h}
    (-\eta \lambda_l )^{h-j-2} 
    \cdot \frac{d}{d\theta_i} \Cov_{z \in D} (\nabla L (z) )_{l,l}
    .
\end{split}
\end{equation}
This concludes the proof of \cref{theo:SGD_bias} by calling $R_{i,j}$ all the terms that multiply the derivative of the covariance. Now we compute the size of the elements $R_{i,j}$.

Then when $l \neq i$ and $c\lambda_l \ll 1$, is:
\begin{equation}
\begin{split}
    &
    -\frac{\eta}{b-1} \frac{\eta^2}{2c} \sum_{l}
    \sum_{j = 0}^{k-2} \binom{k}{j+2}
    (-\eta \lambda_i )^{j} 
    \frac{d}{d\theta_i} \Cov_{z \in D} (\nabla L (z) )_{l,l}.
\end{split}
\end{equation}
Thus the sum of these terms for $l\neq i$ coincides with
\begin{equation}
\begin{split}
    &
    -\frac{\eta}{b-1}
    \frac{c}{2}[c\nabla^2L]^{2\dag} \big(
    [I-\eta \nabla^2 L]^k + c \nabla^2 L - I
    \big)
    \frac{d}{d\theta_i}
    trace(\Cov_{z \in D} (\nabla L (z) )).
\end{split}
\end{equation}
While when $l\neq i$, $c\lambda_l \gg 1$ we approximately obtain
\begin{equation}
\begin{split}
    &
    -\frac{\eta}{b-1} \frac{\eta^2}{2c} \sum_{l}
    \sum_{j = 0}^{k-2} \binom{k}{j+1}
    (-\eta \lambda_i )^{j} 
    (\eta \lambda_l)^{-1} \frac{d}{d\theta_i} \Cov_{z \in D} (\nabla L (z) )_{l,l}.
\end{split}
\end{equation}
Thus the sum of these terms for $l\neq i$ coincides with
\begin{equation}
\begin{split}
    -\frac{\eta}{b-1} 
    c[c\nabla^2L]^{\dag} \big(
    [I-\eta \nabla^2 L]^k - I - [-\eta \nabla^2 L]^k
    \big)
    \mathrm{trace} \left([2c\nabla^2 L]^\dag \frac{d}{d\theta_i}  \Cov_{z \in D} (\nabla L (z) )\right).
\end{split}
\end{equation}
Instead, the term $l = i$ is the following
\begin{equation}
\begin{split}
    &
    -\frac{\eta}{b-1} \frac{\eta^2}{2c}  
    \sum_{j = 0}^{k-2}
    (-\eta \lambda_i )^{j} 
    \sum_{h=j+2}^k \binom{k}{h} (-\eta \lambda_i)^{h-j-2}
    \frac{d}{d\theta_i} \Cov_{z \in D} (\nabla L (z) )_{i,i}
    \\&=
    -\frac{\eta}{b-1}  \frac{\eta^2}{2c} 
    \sum_{j=2}^k (k-j+1)\binom{k}{j} (-\eta \lambda_i)^{j-2}
    \frac{d}{d\theta_i} \Cov_{z \in D} (\nabla L (z) )_{i,i}
    \\&=
    -\frac{\eta}{b-1}
    \left[\frac{d}{d\theta_i} \Cov_{z \in D} (\nabla L (z) )_{i,i}\right] \cdot
    \begin{cases}
        \tfrac{c}{4} & \text{if }c \lambda_i \ll 1\\
        \left(\frac{k}{2c\lambda_i^2} + \frac{1}{2\lambda_i}\right) [1-\eta \lambda_i]^k & \text{if }c \lambda_i = O(\eta^{-1})
    \end{cases}
\end{split}
\end{equation}

We can thus conclude that the step on the regularizer along the direction $\theta_i$ is $-\eta/(b-1)$ times the derivative along $\theta_i$ of the quantity
\begin{equation}
\label{eq:whole_regularizer}
\sum_{j=0}^{k-2} (-\eta \lambda_i)^j \underbrace{\frac{\eta^2}{2c}\sum_{\substack{\lambda_l \text{ eigen-}\\ \text{value of }\nabla^2 L}} \sum_{h=j+2}^k \binom{k}{h}(-\eta \lambda)^{h-j-2} \cdot \Cov_{z \in D}(\nabla L(z))_{l,l}}_{\text{trace of }S_j \cdot \Cov_{z \in D}(\nabla L(z))}.
\end{equation}
This can be rewritten as
\begin{equation}
\begin{split}
    Reg_i
    \quad &= \quad
    c(c\lambda_i)^{-2} \big(
    [I-\eta \lambda_i]^k + c \lambda_i - I
    \big)
    \sum_{\substack{\lambda_l \neq \lambda_i\text{ small}\\ \text{eigenvalue} \\ \text{of }\nabla^2 L}} \!\!\!\!
    \Cov_{z \in D} (\nabla L (z) )_{l,l}
    \\[0.2cm] & \quad + \
    c(c\lambda_i)^{-1} \big(
    [I-\eta \lambda_i]^k - I - [-\eta \lambda_i]^k
    \big)
    \sum_{\substack{\lambda_l \neq \lambda_i \text{ big}\\ \text{eigenvalue} \\ \text{of }\nabla^2 L}} \!\!\!\!
    (c\lambda_l)^{-1} \Cov_{z \in D} (\nabla L (z) )_{l,l}
    \\[0.2cm] & \quad + \
    c(c\lambda_i)^{-1} \big(
    (k-1) [I-\eta \lambda_i]^k - c\lambda_i[I-\eta \lambda_i]^k + ck\lambda_i - (k-1) 
    \big)
    \Cov_{z \in D} (\nabla L (z) )_{i,i}
\end{split}
\end{equation}
Thus on the small eigenvalues $|c\lambda| \ll 1$ we approximately have the following regularizer
\begin{equation}
\begin{split}
    Reg_i
    \quad &= \quad
    \frac{c}{4}
    \sum_{\substack{\lambda_l\text{ small}\\ \text{eigenvalue} \\ \text{of }\nabla^2 L}} \!\!\!\!
    \Cov_{z \in D} (\nabla L (z) )_{l,l}
    \ + \
    \sum_{\substack{\lambda_l \text{ big}\\ \text{eigenvalue} \\ \text{of }\nabla^2 L}} \!\!\!\!
    (c\lambda_l)^{-1} \Cov_{z \in D} (\nabla L (z) )_{l,l}
    .
\end{split}
\end{equation}
On the big ones $c\lambda \gg 1$ we have the following regularizer
\begin{equation}
\begin{split}
    Reg_i
    \quad &= \quad
    \lambda^{-1}
    \sum_{\substack{\lambda_l\text{ small}\\ \text{eigenvalue} \\ \text{of }\nabla^2 L \cup \lambda_i}} \!\!\!\!
    \Cov_{z \in D} (\nabla L (z) )_{l,l}
    \\[0.2cm] & \quad + \
    \lambda^{-1} (1 + (-\eta\lambda)^k)
    \sum_{\substack{\lambda_l \neq \lambda \text{ big}\\ \text{eigenvalue} \\ \text{of }\nabla^2 L}} \!\!\!\!
    (c\lambda_l)^{-1} \Cov_{z \in D} (\nabla L (z) )_{l,l}.
\end{split}
\end{equation}




\section{Size of the Error}
\label{section:error}

We show here some heuristics on how smaller the size of the error is compared to the size of the regularizer. Also, we discuss the ingredients of it.

\label{section:error_math}

\paragraph{The size of the error.}
Assume that every eigenvalue $\lambda$ of $\nabla^2 L$ satisfies $-\frac{\alpha}{c}, 2 + \frac{\alpha}{c}$ for a constant $\alpha \geq 0$.
From the proof \cref{theo:SGD_bias_eta}, precisely \S \ref{section:proof_error} we see that the main part of the regularizer is smaller than 
\[
\frac{c^3}{12b}\nabla^3 L
    \Cov_{z \in D}(\nabla L(z)) \ \ O\left( 1 - \frac{\eta}{b-1} \right).
\]
along the near 0-eigenvalues of the Hessian $\nabla^2 L$ and it is upperbounded approximately by
\[
\frac{\eta c^2}{4b} [\nabla^2 L]\nabla^3 L
    \Cov_{z \in D}(\nabla L(z))
    \ \ O\left( 1 - \frac{\eta}{b-1} \right)
    \quad + \quad O\left(\frac{\eta}{b-1}\right) (\text{GD error})
\]
on the other eigendirections. Here the constant in the big $O$ is upperbounded by $\exp(\alpha)$. Note that this is very small when 
$c \var{\nabla L} \ll 1$ or $\frac{c}{12} \nabla^3 L \ll 1$ or just $\nabla L$ does not align well with the third derivative.
In all the other cases, we have to add these parts to the regularizer and look deeper into the Taylor expansion for a small error part. Also, this is a coarse analysis as in reality some of these terms $O(\eta \nabla L)^2$ are part of the derivative of the regularizer and they can be removed.

\paragraph{The ingredients}
The ingredients of these quantities are the following, all multiplied between each other
\[ 
\eta \nabla^3 L(\cdot),
\quad \qquad I - \eta \nabla^2 L(\cdot),
\quad\qandq\quad 
\eta \nabla L(\cdot)
\]
and in particular, appearing in the forms
\[ 
V_3:=\left( \E_{z \in D} [\mednorm{\nabla^3L(z)}^2] - \mednorm{\nabla^3L}^2 \right)
\qandq H:= \left( \E_{z \in D} [\mednorm{\nabla^2L(z)}^2] - \mednorm{\nabla^2L}^2 \right).
\]

\paragraph{Third derivative: $\nabla^3 L = o(1/\eta)$.}
We start our analysis by analyzing the size of the quantity 
\[ \eta \cdot \nabla^3 L (\cdot)\]
that appears in the error.
Generally, $\nabla^3 L$ is $O(1)$, or anyways much smaller than the Hessian. Even in the case of the edge of stability, the third derivative gets no bigger than $\eta^{-c}$, $c<1$. It has been observed empirically by \cite{damian_self-stabilization_2023} that that $c$ may be equal to $\tfrac{3}{4}$. We however conjecture that that exponent was due to the fact that they used MSE on a neural network with 3 layers and we conjecture that generally for MSE the size of the third derivative never crosses 
\[
O\left(\eta^{-\frac{2\text{depth}-3}{2\text{depth}-2}}\right),
\]
and attains it at the edge of stability. This allows us to conclude that in the case of bigger eigenvalues of the Hessian, the error is almost always smaller than the regularizer (as we have $\frac{\eta}{b}\nabla^3 L < 1$, and in the case of $0$ ones we have $\nabla^3 L = O(1)$ so it depends only on the size of $\frac{c}{12}$ and the size of the gradients if the result is applicable or not.


\paragraph{The usual size of the Hessian.}
Both in the regularizer and in the error figure powers of $[I - \eta \nabla^2 L]$. In the regularizer the biggest exponent is $k-2$, in the error is $2k-2$.
Most of the eigenvalues of the Hessians are usually $0$ or very small positive. In many situations, a few eigenvalues are in the range $1/\eta$ or a bit more, this induces some form of oscillation around the manifold of minima. All these eigenvalues $\lambda$s of the mini-batch Hessians in the range between $0$ and $2/\eta$ cause no issues. Indeed, in these cases, the big products above are just shrinking, as $1-\eta \lambda \in [-1,1]$ in that direction.
In the case in which $|1-\eta \lambda| = 1 + \frac{\alpha}{c}$, with $\alpha > 1$, then anyways the bounds hold but multiplied by a constant smaller than $(1+\frac{\alpha}{k})^{2k} < \exp(2\alpha)$ for the error and smaller than $(1+\frac{\alpha}{k})^{k} < \exp(\alpha)$ for the regularizer.
Thus, e.g., if every eigenvalue $\lambda$ of $\nabla^2 L$ is in $\lambda \in \left(-\frac{\alpha}{c}, \frac{2}{\eta}+\frac{\alpha}{c}\right)$ for reasonably small but O(1) constant $\alpha$ we have that we do not change regime: if it was, the error remains smaller than the regularizer.

\paragraph{Where $1/n$ comes from.}
In ll our errors figure quantities that are $1/n$ or $\eta/b$ smaller, those re due to Corollary \ref{cor:chain_cond_exp_var}. Precisely, it says that the expectation of the product is essentially the product of expectation plus a $O(1/n)$ part. When we are taking together things like $[I-\eta \nabla^2 L(\cdot)]$ those are usually $i < k$ and comes out an $\eta$. So in those cases we have $O(\eta k/n) = O(\eta/b)$.

\section{Size of the Error, the Proof.}
\label{section:proof_error}
Let us call $E_k^{SGD}$ the error $E_k^{SGD} = \Delta_k^{SGD} - \alpha_k^{SGD}$. Then we have that up to lower order in $\eta \nabla L$ (so $+O([\eta \nabla L]^3)$ the error satisfies
\begin{equation}
\begin{split}
    E_{k}
    \quad &= \quad 
    \Delta_{k}^{SGD} - \alpha_{k}^{SGD} 
    \\&= \quad 
    \Delta_{k-1}^{SGD} 
    \quad-\quad 
    \eta \nabla^2 L\big( B_{k}\big)
    \quad - \quad
    [I - \eta \nabla^2 L\big( B_{k}\big) ]
    \alpha_{k-1}^{SGD} 
    \quad+\quad
    \substack{\text{higher order}\\\text{Taylor rest}}
    \\&=\quad 
    [I - \eta \nabla^2 L(B_k)] E_{k-1}^{SGD}
    \quad + \quad
    \substack{\text{higher order}\\\text{Taylor rest}}.
\end{split}
\end{equation}
So we find the formula for it 
Keeping only the term with order 2 in $\nabla L(\cdot)$s we obtain
\begin{equation}
\begin{split}
    E_k^{SGD} 
    \quad&=\quad 
    [I - \eta \nabla^2 L(B_k)] E_{k-1}^{SGD}
    \quad + \quad
    \frac{\eta}{2} \nabla^3 L (B_k) \otimes [\theta_{k-1}^{SGD}-\theta]^2
    \\&=\quad 
    [I - \eta \nabla^2 L(B_k)] E_{k-1}^{SGD}
    \quad + \quad
    \frac{\eta}{2} \nabla^3 L (B_k) \otimes [\theta_1^{GD, (k-1)\eta}-\theta+\Delta_{k-1}^{SGD}]^2
    \\&=\quad 
    [I - \eta \nabla^2 L(B_k)] E_{k-1}^{SGD}
    \quad + \quad
    \frac{\eta}{2} \nabla^3 L (B_k) \otimes [\theta_1^{GD, (k-1)\eta} - \theta + \alpha_{k-1}^{SGD} + E_{k-1}^{SGD} ]^2
\end{split}
\end{equation}
This, again, up to second order in $\eta \nabla L$ is
\begin{equation}
\begin{split}
    E_k^{SGD}
    \quad&=\quad 
    [I - \eta \nabla^2 L(B_k)] E_{k-1}^{SGD}
    \quad + \quad
    \frac{\eta}{2} \nabla^3 L (B_k) \otimes [\theta_1^{GD, (k-1)\eta} - \theta + \alpha_{k-1}^{SGD}]^2
%
\end{split}
\end{equation}
So by inductive hypothesis we have that the size of $E_k^{SGD}$ is the size $E_{k-1}^{SGD}$ plus terms that are of the size of $O (\eta \nabla^3 L \otimes [\eta k \nabla L]^2)$ that is $O(\eta^3 k^3 \mednorm{\nabla L}^2)$.
Moreover, we obtain that up to higher order terms 
\begin{equation}
\begin{split}
    E_{k}^{SGD} \ 
    &= \ 
    \sum_{i=2}^k 
    \left[ \prod_{j=i+1}^k
    [I - \eta \nabla^2 L(B_j)] \right]
    \frac{\eta}{2} \nabla^3 L (B_i) \otimes [ \theta_1^{GD, (i-1)\eta} - \theta + \alpha_{i-1}^{SGD}]^2
    \end{split}
\end{equation}
and more precisely this is $\tfrac{1}{2}$ multiplied by
\begin{equation}
\label{eq:error}
\begin{split}
    &-\eta^3 \sum_{i=1}^k 
    \left[ \prod_{j=i+1}^k
    [I - \eta \nabla^2 L(B_j)] \right]
    \nabla^3 L (B_i) \otimes 
    \left[ 
    \sum_{j=1}^{i-1}  \left[ \prod_{h=j+1}^{i-1} [I - \eta \nabla^2 L(B_h) ] \right] \nabla L(B_j)
    \right]^2
    \\&
    \quad  - \eta \sum_{i=1}^k 
    \left[ \prod_{j=i+1}^k
    [I - \eta \nabla^2 L(B_j)] \right]
    \nabla^3 L (B_i) \otimes 
    \left( \theta^{GD,(i-1)\eta}_1 \right)^2.
    \\&
    \quad  + 2\eta \sum_{i=1}^k 
    \left[ \prod_{j=i+1}^k
    [I - \eta \nabla^2 L(B_j)] \right]
    \nabla^3 L (B_i) \otimes 
    \left[ \theta^{GD,(i-1)\eta}_1 \right] \otimes \alpha_{i-1}^{SGD}.
\end{split}
\end{equation}
Now, defining analogously to $\Delta$ and $\alpha$ the error $E_{k}^{GD}$. Concluding, we have that
\begin{equation}
\begin{split}
    \theta_k^{SGD}-\theta_k^{GD}
    \quad &= \quad 
    \Delta_k^{SGD}-\Delta_k^{GD}
    \quad = \quad
    \alpha_k^{SGD}-\alpha_k^{GD}
    +
    E_k^{SGD}-E_k^{GD}
\end{split}
\end{equation}
and $\mathcal{R} := \alpha_k^{SGD}-\alpha_k^{GD}$ is exactly the regularizer appearing in the statement
and $\mathcal{E} := E_k^{SGD}-E_k^{GD}$ is the error. We will see that once computing the expectations more terms cancel out in the error and the regularizer, as for \S \ref{section:exp_app}.

\begin{remark}[Remarks on Proposition \ref{theo:SGD_effect_fixed_batches}]
Note again that this holds for every version of SGD. In particular, there is an analogous expansion for every first-order method. 
\end{remark}

\paragraph{The second and third terms of \cref{eq:error}.}
When computing $\mathcal{E} := E_k^{SGD}-E_k^{GD}$ the parts relative to the second and third terms of \cref{eq:error} are easier to handle.
Precisely, the subtraction of the term in the second line after applying Corollary \ref{cor:chain_cond_exp_var} becomes a sum of order 
\[
\frac{\eta k}{n - k} O( \text{GD error} )
\quad = \quad
\frac{\eta}{b-1} O( \text{GD error} )
\]
so even smaller than what it would be. Similarly, also the third term is of the same size, by removing and adding 
\[
2\eta \sum_{i=1}^k 
    \sum_{i=1}^k 
    [I - \eta \nabla^2 L]^{k-i-1}
    \nabla^3 L \otimes 
    \left[ \theta^{GD,(i-1)\eta}_1 \right] \otimes \alpha_{i-1}^{SGD}
\]
Then we have that the difference between the SGD part and this is $O(1/(n-k))$ times the GD error and same for this minus the GD error.
So, this part was about
\[
\frac{5 \eta k}{n-k} \quad = \quad \frac{5\eta}{b-1}
\]
smaller than the GD error.
The main part of the error comes from the difference of the terms in the first line.

\paragraph{The first line of \cref{eq:error}.}

\begin{remark}
Next note that part of this falls into the regularizer, precisely all the summands of the form
\[
\eta^3 \E\left[ \nabla L(B_l)^\top \nabla^3 L(B_i) \prod_{h=j}^i[I- \eta \nabla^2 L (B_h)] \nabla L(B_j) \right].
\]
\end{remark}

The remaining terms can be treated as the regularizer. In particular, we rewrite it as 
\begin{equation}
\begin{split}
    & \underbrace{- \eta^3 \sum_{i=1}^k 
    \left[ \prod_{j=i+1}^k
    [I - \eta \nabla^2 L(B_j)] \right]
    \nabla^3 L (B_i) \otimes 
    \left[ 
    \sum_{j=1}^{i-1}  \left[ \prod_{h=j+1}^{i-1} [I - \eta \nabla^2 L(B_h) ] \right] \nabla L(B_j)
    \right]^2}_{\text{SGD part}}
    \\&\quad - 
    \eta^3 \sum_{i=1}^k 
    [I - \eta \nabla^2 L]^{k-i-1}
    \nabla^3 L \otimes 
    \left[ 
    \sum_{j=1}^{i-1}  \left[ \prod_{h=j+1}^{i-1} [I - \eta \nabla^2 L(B_h) ] \right] \nabla L(B_j)
    \right]^2
    \\&\quad + 
    \eta^3 \sum_{i=1}^k 
    [I - \eta \nabla^2 L]^{k-i-1}
    \nabla^3 L \otimes 
    \left[ 
    \sum_{j=1}^{i-1}  \left[ \prod_{h=j+1}^{i-1} [I - \eta \nabla^2 L(B_h) ] \right] \nabla L(B_j)
    \right]^2
    \\&\quad \underbrace{ + \
    \eta^3 \sum_{i=1}^k 
    [I - \eta \nabla^2 L]^{k-i-1}
    \nabla^3 L \otimes 
    \left[ 
    \sum_{j=1}^{i-1}  \ [I - \eta \nabla^2 L ]^{i-j-2} \nabla L
    \right]^2}_{\text{GD part}}.
\end{split}
\end{equation}
The expectation of the difference between third and fourth line, using Corollary \ref{cor:chain_cond_exp_var}, when $\|I - \eta \nabla^2 L(B_i)\| \leq 1+\epsilon$ is smaller in size or equal to
\[
    \eta^3 \sum_{i=1}^k 
    [I - \eta \nabla^2 L]^{k-i-1}
    \nabla^3 L \otimes 
    \left(
    \sum_{j=0}^{i-1}
    O \left( \frac{n(j-1) - (j-1)^2}{(n-1)b} \Cov_{z \in D}(\nabla L) \right)
    \right)
\]
where the constant of the big O is $O\epsilon i$ to the first order in $\epsilon, \ i$, and it is exponentially small in the directions in which $|I-\eta \nabla^2 L| < 1$.
This sums up to something smaller than 
\[
    \eta^3  \sum_{i=1}^k 
    [I - \eta \nabla^2 L]^{k-i-1}
    \nabla^3 L
    O
    \left(
    \frac{n(j-1) - (j-1)^2}{(n-1)b} 
    \Cov_{z \in D}(\nabla L(z)) 
    \right)
\]
\[
\quad < \quad 
    \frac{1}{2b}
    O\left(
    [\nabla^2 L]^{3\dag} \max\left\{2[I - \eta \nabla^2 L]^{k} - c^2 (\nabla^2L)^2 + 2c\nabla^2L -2 \right\}\right)
    \nabla^3 L\Cov_{z \in D}(\nabla L(z))
\]
analogously, this is upper-bounded approximately by
\[
\frac{c^3}{6b}\nabla^3 L
    \Cov_{z \in D}(\nabla L(z))
\]
along the 0-eigenvalues of the Hessian $\nabla^2 L$ and it is upperbounded approximately by
\[
\frac{\eta c^2}{2b} [\nabla^2 L]\nabla^3 L
    \Cov_{z \in D}(\nabla L(z))
\]
along the other directions.
The expectation of the difference between the terms in first and second line is instead Corollary \ref{cor:chain_cond_exp_var} of order approximately $1/n$ of it, indeed we can upperbound it with
\[
-\frac{\eta}{b-1} \big( \text{GD part} + \text{part above} \big).
\]

\paragraph{Summing all up.}
This means that summing all up we obtain that the error, when $-\epsilon < \eta \nabla^2 L< 2+\epsilon$ the part coming from the first line of \cref{eq:error} is smaller than
\[
\frac{c^3}{12b}\nabla^3 L
    \Cov_{z \in D}(\nabla L(z)) \ \ O\left( 1 - \frac{\eta}{b-1} \right)
    \quad + \quad O\left(\frac{\eta}{b-1}\right) (\text{GD error}).
\]
along the 0-eigenvalues of the Hessian $\nabla^2 L$ and it is upperbounded approximately by
\[
\frac{\eta c^2}{4b} [\nabla^2 L]\nabla^3 L
    \Cov_{z \in D}(\nabla L(z))
    \ \ O\left( 1 - \frac{\eta}{b-1} \right)
    \quad + \quad O\left(\frac{\eta}{b-1}\right) (\text{GD error}).
\]
On the other eigendirections. Note that this is even smaller when the noise covariance does not properly align with the eigenvectors of the highest eigenvalues of the third derivative.
when $c \var{\nabla L} \ll 1$ and $\frac{c}{12} \nabla^3 L \ll 1$ or just $\nabla L$ does not align well with the third derivative.
In all the other cases, we have to add these parts to the regularizer and look deeper into the Taylor expansion for a small error part.

\section{Proof of \cref{theo:reg_Hessian}}



{\bf Theorem \ref{theo:reg_Hessian}.}{
    Assume all the training data are correctly classified. Then
    \[
    Fisher \quad = \quad \E_{z \in D} [ \nabla L \nabla L^\top] \quad = \quad \nabla^2 L \quad = \quad Hessian.
    \]
    In particular, the local minima of $\mathrm{trace}(S \cdot \text{Fisher})$ coincides with the local minima of $\mathrm{trace}(S \cdot \text{Hessian})$.
}
\\
Note that 
\[
L\big(\theta, (x,y)\big) \quad = \quad - \log \big( p(y | x,\theta) \big) 
\]
Indeed, note that the approximate Fisher information matrix we work with is
\[
Fisher \quad = \quad \E_{z \in D}\left[
\nabla L \nabla L^\top
\right]
\quad = \quad 
\E_{z \in D} \left[
\nabla \log \big( p(y | x,\theta) \big) \nabla \log \big( p(y | x,\theta)\big)^\top
\right].
\]
In the final part of the training, when on the training data points we have that the labels according to our model correspond with the ones on the training dataset, we have that 
\[
\quad \E_{(x,y) \in D}\left[
\nabla L \nabla L^\top
\right]
\quad = \quad
\E_{x \in D} \left[ \E_{y \sim p(\cdot | x,\theta)} \left[ 
\nabla \log \big( p(y | x,\theta) \big) \nabla \log \big( p(y | x,\theta)\big)^\top
\right] \right].
\]
As $y_i = y\sim p(\cdot | x_i,\theta)$ for all $z = (x_i,y_i) \in D$. This is equal to the Hessian, indeed
\[
Hessian \quad = \quad \E_{z \in D}\left[
\nabla^2 L 
\right]
\quad = \quad 
\E_{x \in D} \left[ \E_{y \sim p(\cdot | x,\theta)} \left[ 
\nabla^2 \log \big( p(y | x,\theta)\big)
\right] \right].
\]
Thus noticing that $ \E_{y \sim p(\cdot | x,\theta)} [\nabla^2 p(y | x,\theta)] = \nabla^2_\theta \E[p(y | x,\theta)] = \frac{d^2}{d \theta^2}1 = 0$, we conclude that
\[
Fisher \quad = \quad Hessian.
\]
This concludes the proof of \cref{theo:reg_Hessian}.

\section{Saddles}

\label{appendix:stationary_point}

Assume $\theta$ is a stationary point for $L$. We analyze here the dynamics in the direction of a unitary eigenvector $v$ of the eigenvalue $\lambda$ of $\nabla^2 L$. Define 
\[
u
\quad := \quad
\frac{1}{n} \sum_{z \in D} \langle v_\lambda, \nabla^2 L(z) \nabla L(z) \rangle
\quad \neq \quad 0.
\]
First note that the update due to $k$ steps of GD from $\theta$, in our regime, is 
\[
\theta_{new} - \theta = [\nabla^2 L]^{-1} (\exp(-c\nabla^2 L) - I)\nabla L(\theta) =: \beta \nabla L.
\]
Indeed at every step $i < k$ we have that 
\[
\theta_{i+1}-\theta 
\quad = \quad
-\eta \nabla L + [I- \eta \nabla^2 L] (\theta_{i}-\theta)
\quad = \quad
- \eta \sum_{j=0}^i [I-\eta \nabla^2 L]^{i-j} \nabla L.
\]
Taking the limit we have the definition of $\beta$.
We see here how far we go after $m$ epochs in the direction of the eigenspace of $\lambda$ assuming $c\nabla L << 1$, then denote $\beta:= \lambda^{-1} (\exp(-c\lambda) - I)$ the update due to GD is $+\beta\nabla L$ and denote the step on the regularizer $\alpha := \frac{\eta}{b-1} S u$.

\begin{table}[ht!]
    \centering
    \begin{tabular}{c|c|c}
        epochs & 1 & 2 \\
        \hline
        $\theta_t - \theta$ & $ - \alpha + \beta\nabla L$ 
        &
        $ - 2\alpha +2\beta \nabla L + \lambda \beta (- \alpha +\beta\nabla L)$ 
        \\
        $\nabla L(\theta_t) - \nabla L$ & $\lambda (- \alpha + \beta\nabla L)$ & $\lambda \big( - (2+\lambda \beta)\alpha + (2+\lambda \beta)\beta\nabla L \big)$ 
        \\
        $\nabla^2 L(\theta_t) - \nabla^2 L$ & $\nabla^3 L  (- \alpha + \beta\nabla L)$ & $\nabla^3 L \big( - (2+\lambda \beta)\alpha + (2+\lambda \beta)\beta\nabla L \big)$ 
        \\
        $L(\theta_t) - L$ & $\substack{\nabla L(\theta_1)  (- \alpha + \beta\nabla L) + \lambda  \mednorm{- \alpha + \beta\nabla L}^2\\
        = \nabla L (- \alpha + \beta\nabla L) 
        + \lambda (- \alpha + \beta\nabla L)^2
        +\tfrac{\lambda}{2}(- \alpha + \beta\nabla L)^2
        }$ 
        & 
        $\substack{\nabla L^\top (\theta_1-\theta + \theta_2-\theta) \
        + \ \tfrac{3\lambda}{2}  \big(\mednorm{\theta_1 - \theta}^2 + \mednorm{\theta_2 - \theta}^2 \big)
        }$
    \end{tabular}
    \label{tab:sad_1}
\end{table}

And in general

\begin{table}[ht!]
    \centering
    \begin{tabular}{c|c}
        epochs & m \\
        \hline
        $\theta_m - \theta$ 
        &
        $ - (\lambda \beta)^{-1} \big((1 + \lambda \beta)^m - 1\big) \alpha
        + (\lambda \beta)^{-1} \big((1 + \lambda \beta)^m - 1\big) \beta \nabla L $ 
        \\
        $\nabla L(\theta_m) - \nabla L$ 
        & 
        $\lambda (\theta_m - \theta) + \tfrac{1}{2} \nabla^3 L (\theta_m - \theta)^2$ 
        \\
        $\nabla^2 L(\theta_m) - \nabla^2 L$ 
        &
        $\nabla^3 L (\theta_m - \theta)$ 
        \\
        $L(\theta_t) - L$ 
        & 
        $\nabla L^\top (\theta_m - \theta)
        + \tfrac{3\lambda}{2}
        \sum_{i=1}^m\mednorm{\theta_i - \theta_{i-1}}^2
        + O(\sum_{i=1}^m \mednorm{\theta_i - \theta_{i-1}}^3)
        $
    \end{tabular}
    \label{tab:sad_2}
\end{table}
We have that after $m$ steps SGD in expectation takes the parameter vector away of the following quantity
\[
    \underbrace{-(\lambda \beta)^{-1} \big((1 + \lambda \beta)^m - 1\big) \alpha}_{\text{SGD regularizer effect}}
    \quad + \quad \underbrace{\lambda^{-1} \big((1 + \lambda \beta)^m - 1\big)\nabla L}_{\text{GD effect}}.
\]
In the escaping directions, for $c\lambda = \eta k \lambda < 0$, the regularizer part is
\[
    \tfrac{\eta}{b}
    (\exp(-c\lambda)-1)^{-1} 
    \big(\exp(-c m \lambda) - 1\big) \lambda^{-2} (\exp(-c\lambda)+c\lambda-1) u
    .
\]
Then, even forgetting about the $\nabla L$ part, so even if we start exactly from a flat area
\[
\theta_m - \theta_{m-1}
\quad = \quad
-(\lambda \beta)^{-1} \lambda \beta (1 + \lambda \beta)^{m-1} \alpha
\]
That is equal to
\[
\theta_m - \theta_{m-1}
\quad = \quad
- (1 + \lambda \beta)^{m-1} \alpha
\]
Thus we have
\[
    \theta_m - \theta_{m-1}
    =
    \tfrac{\eta}{b}
    \exp(-c (m-1) \lambda) \lambda^{-2} (\exp(-c\lambda)+c\lambda-1) u.
\]
Thus, since $\lambda^{-2} (\exp(-c\lambda)+c\lambda-1) > c/2 - c^2\lambda/6$ for all $c\lambda < 0$ we have that
\[
    \mednorm{\theta_m - \theta_{m-1}}^2
    >
    \frac{\eta^2}{b^2}
    \exp(- 2 c (m-1) \lambda)
    \left( \frac{c}{2} - \frac{c^2 \lambda}{6} \right)^2 u^2.
\]
Thus
\begin{equation}
\begin{split}
    &\frac{3\lambda}{2}\sum_{i=1}^m \mednorm{\theta_i - \theta_{i-1}}^2 
    <
    \frac{3 \eta^2 c^2 \lambda u^2}{8 b^2} \left( 1 - \frac{c \lambda}{3} \right)^{\!2} \ \sum_{i=0}^{m-1} 
    \exp(- 2 c i \lambda)
\end{split}
\end{equation}
So, we have that after $m$ steps we reduced the loss by at least of
\[
\frac{3 \eta^2 c^2 \lambda u^2}{8 b^2} \left( 1 - \frac{c \lambda}{3} \right)^{\!2} 
\frac{ \exp(- 2 c m \lambda) - 1 }{ \exp(- 2 c \lambda) - 1 }.
\]
Then we have that the loss decreased by $O(1)$ after a number of steps (epochs) 
\[
\# epochs \quad > \quad
\frac{1}{2|c\lambda|} \log \left( \frac{b^2}{\eta^2 c^4 \lambda u^2} \right) + 2.
\]
This means after 
\[
\# epochs \quad > \quad
\frac{\log(\eta) + \log(u) +2 \log(c) - \log(b)}{c\lambda} + 2.
\]
So after this number of epochs, the saddle was escaped only thanks to the effect of SGD.

Even in the case in which the saddles is higher order or the negative eigenvalue is so small that $|c\lambda| \ll 1$ we have that the effect of SGD pushes you away, indeed in this case 
\[
    \theta_m - \theta \quad = \quad m \alpha \quad = \quad m\cdot \frac{\eta}{b-1}\frac{c}{2} u
\]
so after $m = \frac{b}{\eta}$ epochs $\theta_m - \theta = O(1)$ and if the first non-zero derivative is the $i$-th we have the loss to decrease like 

\begin{table}[ht!]
    \centering
    \begin{tabular}{c|c|c|}
        epochs & 1 & 2 \\
        \hline
        $\theta_t - \theta$ & $ - \alpha$ 
        &
        $ - 2\alpha - c  \frac{\nabla^i L}{(i-1)!}  (- \alpha)^{i-1}$ 
        \\
        $\nabla L(\theta_t) - \nabla L$ & $ \frac{\nabla^i L}{(i-1)!}  (- \alpha)^{i-1}$ & $ \ldots $ 
        
        \\
        $\nabla^j L(\theta_t) - \nabla^j L$ & 
        $\frac{\nabla^i L}{(i-j)!}  (- \alpha)^{i-j}$ 
        & $ (1 + 2^{i-j}) \frac{\nabla^i L}{(i-j)!}  (- \alpha)^{i-j} + \sum_{h=1}^{j-1}  \frac{\nabla^i L}{(i-h)!}  (- \alpha)^{i-h} (-2 \alpha)^{h-j}$ 
        
        \\
        $L(\theta_t) - L$ & $\frac{\nabla^i L}{i!}  (- \alpha)^{i}$ 
        & 
        $ \ldots $
    
    \end{tabular}
    \label{tab:sad_3}
\end{table}

\begin{table}[ht!]
    \centering
    \begin{tabular}{c|c|c|}
        epochs & m \\
        \hline
        $\theta_t - \theta$ 
        &
        $
        - m \alpha + c \sum_{i=1}^m \nabla L(\theta_i) 
        $
        \\
        $\nabla L(\theta_t) - \nabla L$ 
        &
        $
        (- \alpha)^{i-1}  \sum_{h=1}^{i}  \frac{1}{h!}  
        \left(1 + \sum_{l=2}^{m} (l-1)^{i-h}l^{h-1} \right)
        $
        \\
        $\nabla^j L(\theta_t) - \nabla^j L$ 
        &
        $
        (- \alpha)^{i-j} \sum_{h=j}^{i}  \frac{\nabla^i L}{h!}  
        \left(1 + \sum_{l=2}^{m} (l-1)^{i-h}l^{h-j} \right)
        $
        \\
        $L(\theta_t) - L$ 
        &
        $
        (- \alpha)^{i} \sum_{h=0}^{i}  \frac{\nabla^i L}{h!}  
        \left(1 + \sum_{l=2}^{m} (l-1)^{i-h}l^{h-j} \right)
        $
    \end{tabular}
    \label{tab:sad_4}
\end{table}

And 
\[
(- \alpha)^{i} \sum_{h=0}^{i}  \frac{\nabla^i L}{h!}  
        \left(1 + \sum_{l=2}^{m} (l-1)^{i-h}l^{h-j} \right)
\]
in particular, is about the quantity
\[
\nabla^i L(- \alpha)^{i}  
\left[\sum_{h=0}^{i}  \frac{1}{h!} \right] 
\int_{0}^{m} x^i dx
\quad = \quad 
\nabla^i L(- \alpha)^{i}  
\frac{em^{i+1}}{i+1}.
\]
Thus the loss is smaller $O(1)$ if m is such that (taking the logarithms
\[
(i+1) \log(m) + i \log(|\alpha|) \sim 0
\]
so if 
$m \quad \sim \quad \frac{2b}{\eta c u}$ no matter the order of the saddle.

\section{Practical Example: the Setting of the Plots}
\label{section:setting_plots}
The code that generated the figures can be found on \href{https://github.com/PierBeneventano/SGD_without_replacement}{this github repository}.

The setting of \cref{fig:intro,fig:1} is the easiest possible neural network setting trainable with SGD.
We have the $f(\theta, x) = \theta_1\cdot \theta_2 \cdot x$, it is a linear shallow network with width, input, and output dimensions being 1. The Dataset consists of 3 points $D = \{ (1,1), (1,2), (1,0) \}$, the chosen loss is the MSE. 
Precisely, the loss $L(\theta) = (\theta_1\cdot \theta_2-1)^2 + (\theta_1\cdot \theta_2 - 2)^2 + (\theta_1\cdot \theta_2)^2$ with initialization the white bullet $\theta = (1,6)$.
The lowest norm global minimum is $\pm (1,1)$.

This setting shares the same peculiarities of linear diagonal network scenarios. In particular, our experiments extends to bigger linear diagonal networks.
Precisely, we have that, (i) the minimum with the widest valley is also the min-norm solution, (ii) the min-norm solution coincindes also with the global minimum that zeros out our regularizer, and (iii) the global minima have loss $>0$. The datapoints have the same $x$ but different $y$.
Thus the SGD noise never fades out and keeps pushing toward the direction of the regularizer. Moreover, this makes SGD with replacement coincide with a version of GD with label noise. This is not a bad thing as label noise has generally better and understood generalization capabilities \cite{damian_label_2021}.

\end{document}